\documentclass{article}

     \PassOptionsToPackage{numbers}{natbib}
\usepackage{amsmath, amssymb, amsthm}
\usepackage{mathtools}

\DeclareMathAlphabet{\mathdutchcal}{U}{dutchcal}{m}{n}

\usepackage[T1]{fontenc}
\allowdisplaybreaks

\theoremstyle{plain}
\newtheorem{theorem}{Theorem}[section]
\newtheorem{proposition}[theorem]{Proposition}
\newtheorem{lemma}[theorem]{Lemma}

\theoremstyle{definition}
\newtheorem{definition}[theorem]{Definition}

\newtheorem{recall}[theorem]{Recall}
\theoremstyle{remark}
\newtheorem{remark}[theorem]{Remark}

\usepackage{algorithm, algorithmic}

\usepackage[textsize=tiny]{todonotes}

\usepackage{enumitem}

\usepackage{subcaption}
\usepackage{caption}

\usepackage[utf8]{inputenc} % allow utf-8 input
\usepackage[T1]{fontenc}    % use 8-bit T1 fonts
\usepackage{hyperref}       % hyperlinks
\usepackage{url}            % simple URL typesetting
\usepackage{booktabs}       % professional-quality tables
\usepackage{amsfonts}       % blackboard math symbols
\usepackage{nicefrac}       % compact symbols for 1/2, etc.
\usepackage{microtype}      % microtypography
\usepackage{xcolor}         % colors
\usepackage{multirow}
\usepackage[capitalize,noabbrev]{cleveref}

\title{Exact Tensor Completion Powered by \\
            Slim Transforms}

\author{
  $\text{Li Ge}^{1}, \text{Lin Chen}^{1}, \text{Yudong Chen}^{2}, \text{Xue Jiang}^{1}$ \\
${}^1 \text{Shanghai Jiao Tong Unversity}$ \\
  $\centering {}^2 \text{University of Wisconsin-Madison}$\\
  \texttt{geli2000an@sjtu.edu.cn}
}

\begin{document}

\maketitle

\begin{abstract}
  In this work, a tensor completion problem is studied, which aims to perfectly recover the tensor from partial observations. The existing theoretical guarantee requires the involved transform to be orthogonal, which hinders its applications. In this paper, jumping out of the constraints of isotropy and self-adjointness, the theoretical guarantee of exact tensor completion with arbitrary linear transforms is established by directly operating the tensors in the transform domain. With the enriched choices of transforms, a new analysis obtained by the proof discloses why slim transforms outperform their square counterparts from a theoretical level. Our model and proof greatly enhance the flexibility of tensor completion and extensive experiments validate the superiority of the proposed method.
\end{abstract}

\section{Introduction}
    Tensors, or multidimensional arrays, are data structures extensively applied in computer vision \cite{Tensor_in_CV_and_DL}, signal processing \cite{Survey_TSP} and machine learning \cite{rabanser2017introduction}. With the ever-growing volume of data, a substantial proportion inherently possesses a multiway nature, including RGB/multispectral/hyperspectral data, medical images \cite{T_M_E}, network traffic flow \cite{LightNestle}, etc., which leads us to organize them into tensors and enables efficient compression, transmission and recovery\cite{Tensor_Decompositions_and_Applications}.
    
    Commonly, it is not easy to satisfactorily compress or recover a tensor without any prior information. However, a specific property or prior of tensors that sheds light on tensor processing is the low-rank property. Enjoying the analogy with matrices, it is conceivable that low-rank tensors will have great advantages being represented, compressed, or recovered \cite{from_matrix_to_tensor}.

    Unfortunately, unlike the matrix case, in which the rank-related theory has been well established, low-rank tensors can be unexpectedly difficult to identify since the concept ``rank'' is not well-defined for tensors. Over the past decades, several definitions of tensor ranks have been proposed, each with its limitations. For instance, the well-known CP rank \cite{Tensor_Decompositions_and_Applications} encounters numerical problems and has no tractable convex relaxation, i.e., nuclear norm, which is a useful tool powered by convex optimization and has proven its success in matrix recovery\cite{matrix_recovery}, matrix completion \cite{the_power_of_CVP_matrix_completion} and robust principle component analysis (RPCA) \cite{Candes_RPCA}. Another category of definitions stems from the tensor multilinear rank, in which the rank of tensor $\mathcal{X}$ is a vector with each entry being the rank of a specified unfolding of $\mathcal{X}$. Such form of tensor rank includes Tucker \cite{Tucker}, tensor train \cite{oseledets2011tensortrain}, and tensor ring \cite{zhao2016tensorring} while their theoretical guarantees can be inferior to the matrix case.

    In recent years, leveraging the tensor-tensor product and tensor singular value decomposition (t-SVD), a new tensor rank, i.e., tensor tubal rank is developed for 3-way tensors \cite{original_TNN} and can be extended to high order cases \cite{high_order_t-SVD}. By viewing tensor fibers of the third dimension as tubes, a crucial concept in the methods using t-SVD is the applied transform. A good transform should strengthen the low-rankness of the tensors, i.e., making the tensor more favorable for recovering in the transform domain than in the original domain. First, in \cite{original_TNN}, the tensor nuclear norm (TNN) is induced by discrete Fourier transform (DFT). Equipped with TNN, tensor completion and tensor RPCA \cite{Lu_TRPCA} enjoy the exact recovery guarantee. Further, by extending DFT to any invertible orthogonal transform $\mathbf{L}$ and encapsulating $\mathbf{L}(\cdot)$ into the induced tensor-tensor product $\star_{\mathbf{{L}}}$, the exact recovery can still be guaranteed with enriched choices of the transforms \cite{Lu_CVPR,Lu_ICCV}.

    However, increasing experimental phenomena from recent works \cite{Framelet_TNN,S2NTNN,dctionary_learning_TNN} indicate that compared to the square orthogonal transforms, ``slim'' and non-invertible transforms that map the original tensor into bigger transform domain can achieve better recovery accuracy and flexibility (some of the transforms are self-adaptive) in tensor completion. However, the existing theoretical result requires isotropy and self-adjointness (both are properties of orthogonal transforms) and thus fails to guarantee the recovery powered by these non-invertible transforms. Moreover, the determination of the transforms is empirical and lacks theoretical justification as to why slim transforms are better.

    In this paper, jumping out of the constraints of isotropy and self-adjointness, we prove that exact tensor completion with arbitrary linear transforms can be achieved with high probability. To that end, tensors need to be operated directly in the transform domain different from existing proofs \cite{original_TNN,Lu_CVPR}. Specifically, among the enriched choices of the transforms that enjoy the exact completion guarantee, a new analysis gained through the proof is presented for the first time to show why slim transforms outperform square transforms from a theoretical level. Our proof of the main theorem is completely new and challenging as the operator is mingled with the transform $\mathbf{T}$ and the crucial norm-controlling lemmas need to be re-established. Since we require neither isotropy nor self-adjointness upon the linear transform, our model and proof are applicable to much more general transforms. Extensive experiments on random data and visual data validate the superiority of the proposed method.

    The remainder of this paper is organized as follows. Section 2 introduces the notations and tensor operators defined directly on the transform domain. In Section 3, we introduce the new tensor completion program powered by arbitrary linear transforms and provide the exact completion theoretical guarantee. Section 4 provides numerical experiments and a new analysis. The conclusion is drawn in Section 5. The detailed proof/analysis/algorithms are organized in the appendices.
    \vspace{-0.25cm}

\section{Defining Tensor Operators Directly on the Transform Domain}
    In this section, we introduce notations used throughout the paper and give definitions of tensor-tensor product, t-SVD, and norms that are directly defined on the transform domain.
    \vspace{-0.3cm}
    \paragraph*{Notations}
        Lowercase letters (e.g.\ $a$) denote scalars; their boldface versions (e.g. $\mathbf{a}$) denote vectors. Matrices and tensors are denoted by capital boldface letters (e.g.\ $\mathbf{A}$) and calligraphic letters (e.g.\ $\mathcal{A}$), respectively. $(\cdot)^{\text{T}}, (\cdot)^{\text{H}}$ and $(\cdot)^{\dag}$ denote transpose, adjoint/Hermitian transpose and pseudo-inverse, respectively. For indexing vectors, matrices, and tensors, we use $\mathcal{A}_{ijk}$ or $\mathcal{A}(i, j, k)$ to denote the $(i, j, k)$th entry of $\mathcal{A}$ and use $:$ to represent a full selection in the corresponding mode. 
        
        $\langle\cdot, \cdot\rangle$ denotes inner product. $\Vert\cdot\Vert_{\text{F}}$ and $\Vert\cdot\Vert_{\infty}$ stand for the Frobenius norm and infinity norm of vectors/matrices/tensors, respectively. $\Vert\cdot\Vert$ or $\Vert\cdot\Vert_{\text{op}}$ stands for spectral/operator norm (induced by $\Vert\cdot\Vert_{\text{F}}$ in both domains unless otherwise specified). $\Vert\cdot\Vert_*$ is the nuclear norm. $\Vert\cdot\Vert_{1\mapsto2}$ and $\Vert\cdot\Vert_{2\mapsto\infty}$ are the largest $\ell_2$ norm of the matrix/tensor columns and rows, respectively, and $\Vert\cdot\Vert_{\infty, 2} = \max(\Vert\cdot\Vert_{1\mapsto2}, \Vert\cdot\Vert_{2\mapsto\infty})$. $\text{Ber}(p)$ denotes Bernoulli sampling with rate $p$. By \textit{with high probability} (w.h.p.), we mean with probability $\ge 1 - c_1(n_1 + n_2)^{-c_2}$, where $c_1, c_2 > 0$ are universal constants. 

        For operators, $\times_3$ denotes the mode-$3$ product. $\texttt{vec}(\cdot), \texttt{bdiag}(\cdot)$ and $\texttt{fold}(\cdot)$ denote vectorization, organizing the frontal slices of a tensor into a block diagonal matrix and its inverse operation, respectively. $\sigma_{\max}(\cdot)$/$\sigma_{\min}(\cdot)$ stands for taking the maximum/minimum singular value of a matrix. $(\mathcal{A})_{+}$ denotes taking the maximum of $0$ and each entry of $\mathcal{A}$. $x \vee y$ means taking the larger of $x, y$.

    \subsection{Difficulties Encountered with Existing Definitions}
        Assume an arbitrary linear transform $\mathit{T}(\cdot)$ is given by $\mathbf{T} \in \mathbb{C}^{N_3 \times n_3}$ and $\mathit{T}(\mathcal{A}) = \mathcal{A} \times_3 \mathbf{T}$ for $\mathcal{A} \in \mathbb{C}^{n_1 \times n_2 \times n_3}$, let us first attempt to use the transform-induced tensor-tensor product $\star_{\mathbf{T}}$ to define tensor nuclear norm following the methodology in \cite{Lu_CVPR}. With transform-induced $\star_{\mathbf{T}}$ and subsequent t-SVD. One can define the tensor spectral norm as $\Vert \mathcal{A} \Vert := \Vert \texttt{bdiag}(\mathit{T}(\mathcal{A})) \Vert$ for a tensor $\mathcal{A}$.
        It is tempting to define the tensor nuclear norm as the dual of the tensor spectral norm, still following \cite{Lu_CVPR}:
        \begin{align}
            \Vert \mathcal{A} \Vert_* := \sup_{\Vert \mathcal{B} \Vert \leq 1} \langle \mathcal{A}, \mathcal{B} \rangle \nonumber = \sup_{\Vert \texttt{bdiag}(\mathit{T}(\mathcal{B})) \Vert \leq 1} \langle \mathcal{A}, \mathcal{B} \rangle. \nonumber
        \end{align}
        Now the problem of losing the norm-preserving property occurs since $\forall \mathcal{A}, \mathcal{B}, \langle \mathcal{A}, \mathcal{B} \rangle = \langle \mathit{T}(\mathcal{A}), \mathit{T}(\mathcal{B}) \rangle$ holds only when $\mathbf{T}$ is orthogonal/unitary (or with some constant $l$). Therefore, although such a manner of definition that wraps the transform into the element-element operation does bring convenience and elegance in form, the ability for generalization is hindered. In this paper, we show that the operations to push tensors from the orginal domain to the transform domain and the pull-back can be redundant, i.e., tensors can be directly operated in the transform domain. To this aim, tensor operators without the transform involved are given in the remainder of this section.

    \subsection{Tensor Operators Defined Directly on the Transform Domain}
        First, we introduce the tensor-tensor product defined directly on the transform domain.
        \begin{definition} [\textbf{Tensor-tensor product}]\footnote{Such slice-wise operation has also appeared in \cite{KERNFELD2015545}. However, the standing points are different: \cite{KERNFELD2015545} defines "face-wise product" for the convenience of tensor product in the transform domain whereas it still operates the tensors in the original domain. While in this work, tensors are directly operated in the transform domain.}
            \label{def:natural_tensor-tensor_product}
            For two tensors $\mathcal{A} \in \mathbb{C}^{n_1 \times n_2 \times n_3}, \mathcal{B} \in \mathbb{C}^{n_2 \times l \times n_3}$, tensor-tensor product between $\mathcal{A}$ and $\mathcal{B}$ is denoted as $\mathcal{C} = \mathcal{A} \star \mathcal{B} \in \mathbb{C}^{n_1 \times l \times n_3}$, such that $\mathcal{C}(:, :, k) = \mathcal{A}(:, :, k) \mathcal{B}(:, :, k), k = 1, \hdots, n_3$. This process is given in Algorithm \ref{alg:natural_tensor-tensor_product}.
        \end{definition}
        It is evident that such tensor-tensor product is simply the matrix-matrix product on each frontal slice thus a few relative definitions can be given as expected.
        
        \begin{minipage}{0.36\textwidth}
            \vskip -0.15 in
            \begin{algorithm}[H]
                \caption{\scriptsize Tensor-tensor product}
                \label{alg:natural_tensor-tensor_product}
                \begin{algorithmic}
                \STATE {\bfseries Input:} $\mathcal{A} \in \mathbb{C}^{n_1 \times n_2 \times n_3}, \mathcal{B} \in \mathbb{C}^{n_2 \times l \times n_3}$
                \STATE {\bfseries Output:} $\mathcal{C} \in \mathbb{C}^{n_1 \times l \times n_3}$
                \FOR{$k = 1$ {\bfseries to} $n_3$}
                    \STATE $\mathcal{C}(:, :, k) = \mathcal{A}(:, :, k) \mathcal{B}(:, :, k)$;
                \ENDFOR
                \end{algorithmic}
            \end{algorithm}
        \end{minipage}
        \hfill
        \begin{minipage}{0.64\textwidth}
        \vskip -0.15 in
            \begin{algorithm}[H]
                \caption{\scriptsize T-SVD}
                \label{alg:natural_t-SVD}
                \begin{algorithmic}
                \STATE {\bfseries Input:} $\mathcal{A} \in \mathbb{C}^{n_1 \times n_2 \times n_3}$
                \STATE {\bfseries Output:} $\mathcal{U} \in \mathbb{C}^{n_1 \times n_1 \times n_3}, \mathcal{S} \in \mathbb{C}^{n_1 \times n_2 \times n_3}, \mathcal{V} \in \mathbb{C}^{n_2 \times n_2 \times n_3}$
                \FOR{$k = 1$ {\bfseries to} $n_3$}
                    \STATE $[\mathbf{U}, \mathbf{S}, \mathbf{V}] = \text{SVD}(\mathcal{A}(:, :, k))$;
                    \STATE $\mathcal{U}(:, :, k) = \mathbf{U}, \mathcal{S}(:, :, k) = \mathbf{S}, \mathcal{V}(:, :, k) = \mathbf{V}$;
                \ENDFOR
                \end{algorithmic}
            \end{algorithm}
        \end{minipage}

        \begin{definition}[\textbf{Tensor (conjugate) transpose}]
            \label{def:tensor_conjugate_transpose}
            Given a tensor $\mathcal{A} \in \mathbb{C}^{n_1 \times n_2 \times n_3}$, the tensor transpose $\mathcal{A}^{\text{T}} \in \mathbb{C}^{n_2 \times n_1 \times n_3}$ can be defined following $\mathcal{A}^{\text{T}}(:, :, k) = [\mathcal{A}(:, :, k)]^{\text{T}}$. The tensor conjugate transpose $\mathcal{A}^{\text{H}}$ can be defined as $\mathcal{A}^{\text{H}}(:, :, k) = [\mathcal{A}(:, :, k)]^{\text{H}}$.
        \end{definition}
        
        \begin{definition}[\textbf{Identity tensor}]
            \label{def:identity_tensor}
            The identity tensor $\mathcal{I} \in \mathbb{C}^{n_1 \times n_1 \times n_3}$ is defined such that each frontal slice of $\mathcal{I}$ is an identity matrix $\mathbf{I} \in \mathbb{C}^{n_1 \times n_1}$.
        \end{definition}

        \begin{definition}[\textbf{Unitary tensor}]
            \label{def:unitray_tensor}
            A tensor $\mathcal{U} \in \mathbb{C}^{n_1 \times n_1 \times n_3}$ is called unitary if $\mathcal{U} \star \mathcal{U}^{\text{H}} = \mathcal{I}$.
        \end{definition}

        \begin{definition}[\textbf{T-SVD}]
            \label{def:natural_t-SVD}
            Tensor SVD (t-SVD) of $\mathcal{A} \in \mathbb{C}^{n_1 \times n_2 \times n_3}$is performing matrix SVD for each frontal slice of $\mathcal{A}$. The process of computing t-SVD is given in Algorithm \ref{alg:natural_t-SVD}.
        \end{definition}
        After performing t-SVD defined as above, one can verify that $\mathcal{U}, \mathcal{V}$ decomposed are both unitary according to Definition \ref{def:unitray_tensor}.

        With t-SVD, we are ready to further define the tensor rank and the tensor spectral norm.
        \begin{definition}[\textbf{Tensor tubal rank}]
            \label{def:tensor_tubal_rank}
            For $\mathcal{A} \in \mathbb{C}^{n_1 \times n_2 \times n_3}$ with t-SVD $\mathcal{A} = \mathcal{U} \star \mathcal{S} \star \mathcal{V}^{\text{H}}$, the tensor tubal rank of $\mathcal{A}$ is defined as the number of nonzero tubes of $\mathcal{S}$, i.e., $\text{rank}(\mathcal{A}) = \#\{i, \mathcal{S}(i, i, :) \neq \mathbf{0}\}$.
        \end{definition}
        Sharing the analogy with matrices, we also have skinny t-SVD as $\mathcal{A} = \mathcal{U}_r \star \mathcal{S}_r \star \mathcal{V}_r^{\text{H}}$ with $\mathcal{U}_r \in \mathbb{C}^{n_1 \times r \times n_3}, \mathcal{S}_r \in \mathbb{C}^{r \times r \times n_3}, \mathcal{V}_r \in \mathbb{C}^{n_2 \times r \times n_3}$, where $r = \text{rank}(\mathcal{A})$. Since the skinny t-SVD explicitly unveils the subspace of $\mathcal{A}$, we use it when referring to ``t-SVD'' throughout the paper. 

        \begin{definition}[\textbf{Tensor spectral norm}]
            \label{def:tensor_spectral_norm}
            For a tensor $\mathcal{A} \in \mathbb{C}^{n_1 \times n_2 \times n_3}$ with t-SVD $\mathcal{A} = \mathcal{U} \star \mathcal{S} \star \mathcal{V}^{\text{H}}$, the tensor spectral norm is the maximum value of $\mathcal{S}$, i.e., $\Vert \mathcal{A} \Vert = \max_{i, k} \mathcal{S}(i, i, k)$.
            \label{def:tensor_spectral_norm}
        \end{definition}
        It can be verified that the tensor spectral norm of $\mathcal{A}$ in Definition \ref{def:tensor_spectral_norm} is the operator norm induced by Frobenius norm in both domains as we treat ``$\mathcal{A} \star (\cdot)$'' as an operator and use the fact that the matrix operator norm is the matrix spectral norm. 
        
        Now we can define the tensor nuclear norm as the dual norm of the tensor spectral norm:\footnote{One can verify that the tensor nuclear norm that follows Definition \ref{def:tensor_nuclear_norm} is indeed the dual norm of the tensor spectral norm.}
        % i.e., $\forall \mathcal{A}, \mathcal{B}, \Vert \mathcal{A} \Vert_{*} = \text{sup}_{\Vert \mathcal{B} \Vert \leq 1} \langle \mathcal{B}, \mathcal{A} \rangle.$
        \begin{definition}[\textbf{Tensor nuclear norm}]
            \label{def:tensor_nuclear_norm}
            For a tensor $\mathcal{A} \in \mathbb{C}^{n_1 \times n_2 \times n_3}$ with t-SVD $\mathcal{A} = \mathcal{U} \star \mathcal{S} \star \mathcal{V}^{\text{H}}$, the tensor spectral norm is defined as $\Vert \mathcal{A} \Vert_* = \sum_{k=1}^{n_3}\sum_{i=1}^{r} \mathcal{S}(i, i, k)$.
        \end{definition}

\section{Exact Tensor Completion Guarantee with Arbitrary Linear Transforms}
    Equipped with these concepts directly built on the transform domain, we now consider completing a tensor with an arbitrary linear transform. In this section, under certain tensor incoherence conditions, a theoretical guarantee to achieve exact tensor completion with arbitrary linear transforms is established.

    \subsection{The Optimization Program}
        The program we consider can be expressed as
        \begin{equation}
            \begin{aligned}
                \min_{\mathcal{X}} \Vert \mathit{T}(\mathcal{X})\Vert_{*} ~~~\text{s.t.}~\mathcal{S}_{\Omega}(\mathcal{X}) = \mathcal{S}_{\Omega}(\mathcal{M}), 
            \end{aligned}
            \label{eq:program:min_transformed_TNN}
        \end{equation}
        where $\mathit{T}: \mathbb{C}^{n_3}\mapsto\mathbb{C}^{N_3}$ is a linear injection, $\mathit{T}(\mathcal{X}) = \mathcal{X} \times_3 \mathbf{T}, \mathbf{T} \in \mathbb{C}^{N_3 \times n_3}$. $\mathcal{X}, \mathcal{M} \in \mathbb{C}^{n_1 \times n_2 \times n_3}$, $\mathcal{S}_{\Omega}$ is a sampling operator with $\Omega \sim \text{Ber}(p)$. $[\mathcal{S}_{\Omega}(\mathcal{X})]_{ijk} = \mathcal{X}_{ijk}, (i, j, k) \in \Omega$ and equaling $0$ elsewhere. When $\mathbf{T}$ is the discrete Fourier transform (DFT), Program \eqref{eq:program:min_transformed_TNN} regresses to the optimization problem considered in existing TNN works such as\cite{original_TNN}. 
        
        Program \eqref{eq:program:min_transformed_TNN} can be solved by the standard ADMM framework, with details provided in Appendix \ref{section:solving_main_program}.

        Program \eqref{eq:program:min_transformed_TNN} is motivated by a natural idea: we wish to recover tensor $\mathcal{X}$ from partial observations $\Omega$, while the intrinsic low-rank property of $\mathcal{X}$ itself might not be significant enough, Thus, we seek to minimize its nuclear norm in the transform domain characterized by a transform $\mathbf{T}$. Obviously, the transform $\mathbf{T}$ plays a direct and crucial role in the success of the completion. Initially, this is only considered in the context of DFT representing the idea that convolution can well capture the structure along the third mode of the data \cite{original_TNN}. The transform is further generalized to any invertible orthogonal transforms, e.g. discrete cosine transform (DCT) and discrete sine transform (DST), while still maintaining the exact recovery guarantee \cite{Lu_CVPR}.

        However, the requirements on invertibility and orthogonality limit our choices of transforms. The low-rank property might become more pronounced when the transform maps $\mathcal{X}$ into a larger space, which is not invertible ($N_3 > n_3$) \cite{Framelet_TNN,S2NTNN,dctionary_learning_TNN}. Thus, a recovery guarantee overcoming the limitations imposed by isotropy and self-adjointness is desired, which is answered in the next two subsections.

        \subsection{Tensor Incoherence Conditions}
            \label{section:tensor_inc_cond}
            To establish such a guarantee, we need tensor incoherence conditions. First, the tensor bases in the transform domain are introduced:
            \begin{definition}[\textbf{Tensor basis}]
                \label{def:tensor_basis}
                We denote $\xi_i \in \mathbb{C}^{n_1 \times 1 \times N_3}$ as the tensor column basis with $\xi_i(i, 1, :) = 1$ and equalling $0$ elsewhere. $\xi_j^{\text{H}} \in \mathbb{C}^{1 \times n_2 \times N_3}$ denotes the tensor row basis with $\xi_j^{\text{H}}(1, j, :) = 1$ and equalling $0$ elsewhere. The tensor frontal slice basis is denoted by $\zeta_k \in \mathbb{C}^{1 \times 1 \times N_3}$ with $\zeta_k(1, 1, k) = 1$ and equalling $0$ elsewhere. Figure \ref{fig:sketch_tensor_basis} provides an illustration. 
            \end{definition}

            \begin{figure}[ht]
                \vskip -0.1 in
                \centering
                \includegraphics[width=0.3\textwidth, height=0.12\textwidth]{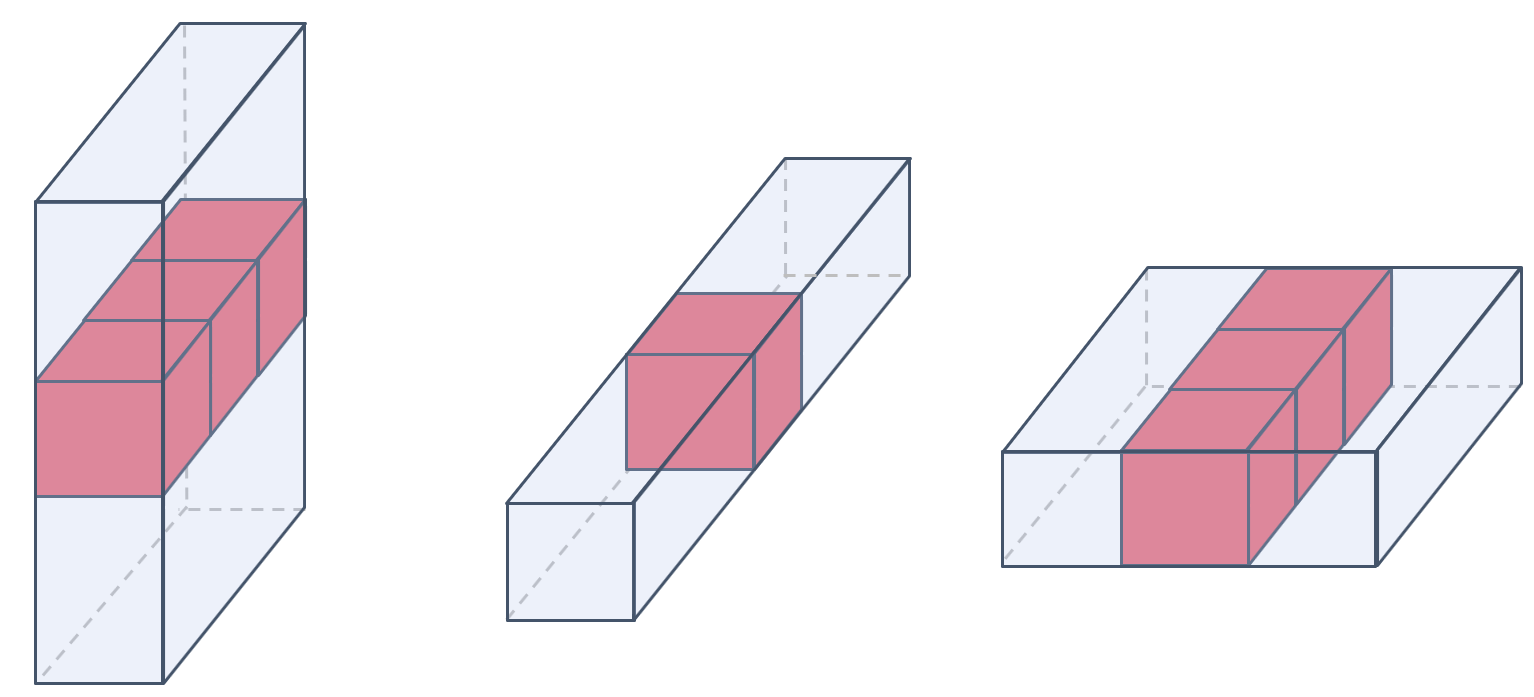}
                \caption{\scriptsize Tensor basis $\xi_i, \zeta_k$ and $\xi_j^{\text{H}}$ from left to right. The red cubes represent $1$ and rest represent $0$.}
                \label{fig:sketch_tensor_basis}
                \vskip -0.1 in
            \end{figure}
            
            To ensure the energy of the singular tensors is sufficiently spread, we need the following tensor incoherence conditions:
            \begin{definition}[\textbf{Tensor incoherence conditions}]
                Let $\mathit{T}(\mathcal{X}) = \mathcal{U} \star \mathcal{S} \star \mathcal{V}^{\text{H}}$. We say that Program \eqref{eq:program:min_transformed_TNN} satisfies tensor incoherence condition with parameters $ \mu > 0, \nu > 0$ if it holds that 
                \begin{align}
                    \max_{i}~\Vert \mathcal{U}^{\text{H}} \star \xi_{i} \Vert_{\text{F}} & \leq \sqrt{\frac{\mu r N_3}{n_1}}, \quad \label{eq:inc_cond_U_V}
                    &
                    \max_{j}~\Vert \mathcal{V}^{\text{H}} \star \xi_{j} \Vert_{\text{F}} & \leq \sqrt{\frac{\mu r N_3}{n_2}}; 
                    \\
                    \max_{ik}~\Vert \mathcal{U}^{\text{H}} \star \xi_i \star \mathit{T}(\zeta_k) \Vert_{\text{F}} & \leq \sqrt{\frac{\nu r}{n_1}} \Vert \mathbf{T} \Vert_{1\mapsto2} , \quad 
                    &
                    \max_{jk}~\Vert \mathcal{V}^{\text{H}} \star \xi_j \star \mathit{T}(\zeta_k) \Vert_{\text{F}} & \leq \sqrt{\frac{\nu r}{n_2}} \Vert \mathbf{T} \Vert_{1\mapsto2} .\label{eq:inc_cond_U_V_T}
                \end{align}
            \end{definition}
            Parameters $\mu, \nu$ reflect the energy distribution of the singular tensors of $\mathit{T}(\mathcal{X})$: When the energy of $\mathcal{U}(:, :, k), \mathcal{V}(:, :, k)$ is most uniformly distributed, $\mu, \nu$ have the minimum values of $1$; When $\mathcal{U}(:, :, k)$ or $\mathcal{V}(:, :, k))$ contains an identity matrix, $\mu, \nu$ have the maximum values of $\frac{n_1}{r}, \frac{n_2}{r}$, respectively.

            \begin{remark}
                When $\mathbf{T}$ is an orthogonal matrix, \eqref{eq:inc_cond_U_V_T} reduces to (14) and (15) in \cite{Lu_ICCV}. Also \cite{Lu_ICCV} states that if letting $\mu = \nu$,  \eqref{eq:inc_cond_U_V_T} can be derived from \eqref{eq:inc_cond_U_V}, which means that \eqref{eq:inc_cond_U_V} is stronger and only one set of incoherence conditions is needed. In fact, though if $\nu < 1$ there for sure exists some $\mu < 1$ and vice versa, $\mu = \nu$ does not necessarily hold. Thus for the precision of our model and proof, we keep both incoherence conditions.
            \end{remark}

        \subsection{Exact Completion Guarantee}
            Now we give the exact tensor completion guarantee with arbitrary linear transforms:
            \begin{theorem}
                Suppose $(\mathcal{X}, \mathit{T})$ satisfies \eqref{eq:inc_cond_U_V}-\eqref{eq:inc_cond_U_V_T} with $\mu, \nu > 0$, $\lambda = \max(\mu, \nu)$, then there exist constants $c_0, c_1, c_2> 0$ such that if 
                \begin{align}
                    p \geq &c_0 (\kappa^2(\mathbf{T}) + \rho^2(\mathbf{T}))\frac{\lambda r (n_1 + n_2)}{n_1 n_2} \log^2(\kappa(\mathbf{T})(n_1 + n_2)N_3), 
                    \label{eq:sampling_rate_p}
                \end{align}
                then $\mathcal{X}$ is the unique optimal solution to \eqref{eq:program:min_transformed_TNN} with probability at least $1 - c_1((n_1 + n_2)N_3)^{-c_2}$, where $\kappa(\mathbf{T}) = \frac{\sigma_{\max}(\mathbf{T})}{\sigma_{\min}(\mathbf{T})}$ is the condition number of $\mathit{T}$ and $\rho(\mathbf{T}) = N_3\max(\Vert \mathbf{T} \Vert_{\infty}^2, \Vert \mathbf{T}^{\dag} \Vert_{\infty}^2)$.
                \label{theorem:exact_completion}
                \vskip -0.1 in
        \end{theorem}
        \begin{remark}[\textbf{Consistency with the existing TNN result}]
            It can be readily checked that when $\mathbf{T}$ is the DFT matrix\footnote{The DFT matrix referred to here is with a normalizing parameter $\frac{1}{\sqrt{n_3}}$.}, \eqref{eq:inc_cond_U_V} is the same as \eqref{eq:inc_cond_U_V_T}. Since for DFT, $N_3 = n_3, \kappa(\mathbf{T}) = 1, \rho(\mathbf{T}) = 1$, Theorem \ref{theorem:exact_completion} regresses to the existing TNN result. 
        \end{remark}
        \begin{remark}[\textbf{How to design $\mathbf{T}$ according to Theorem \ref{theorem:exact_completion}}]
            From the perspective of minimizing the sampling rate, $\mathbf{T}$ is a good measurement matrix (condition number close to 1 and energy is uniformly distributed) if $\kappa(\mathbf{T})$ and $\rho(\mathbf{T})$ take smaller values. However, the presence of incoherence parameter $\lambda$ makes the optimal $\mathbf{T}$ data-dependent and thus coupled in a complex way. Experiment results are provided in Section \ref{section:exp}, whereas the design of the optimal transform is outside the scope of this work. 
            \label{remark:design_T}
        \end{remark}

        Theorem \ref{theorem:exact_completion} is applicable to arbitrary transforms with no requirements on isotropy or self-adjointness. It provides a theoretical guarantee for transforms such as wavelet, Framelet, or data-driven types and greatly enhances the flexibility of the transformed tensor completion model \eqref{eq:program:min_transformed_TNN}.

        \subsection{Roadmap and Main Challenges of the Proof}
            \label{subsection:roadmap_and_challenges}
            The high-level roadmap of the proof of Theorem \ref{theorem:exact_completion} is standard. First, a dual certificate is assumed to have been constructed successfully and we need to prove that with this dual certificate, Program \ref{eq:program:min_transformed_TNN} has the unique optimal solution and $\mathcal{X}$ is completed exactly. This step is given by Proposition \ref{proposition:optimality_condition}. Second, with the tensor incoherence conditions, it suffices to construct such certificate satisfying subgradient conditions, which is given by Proposition \ref{proposition:constructing_dual_certificate}. Combining these two propositions completes the proof of Theorem \ref{theorem:exact_completion}.
    
            It is worth noting that the proof of Theorem \ref{theorem:exact_completion} is challenging. Firstly, the definitions of transform-involved tensor-tensor product, t-SVD and tensor nuclear norm in existing works cannot be utilized in our case, which makes us turn to work directly in the transform domain. Secondly, to the best of our knowledge, the proof of subgradient optimality conditions (Proposition \ref{proposition:optimality_condition}) which requires a case-by-case analysis, is new in the problem of tensor completion. The proof of this proposition is inspired by \cite{proof_6}, which aims to recover 1-D data sparse in the transform domain. By noticing the analogy among 1-D sparse signal, 2-D low rank matrices and 3-D low rank tensors, we complete the proof of Proposition \ref{proposition:optimality_condition}. Thirdly, the construction of dual certificate requires the establishment of three norm-controlling lemmas (Lemma \ref{lemma:bounding_spectral_norm}, \ref{lemma:bounding_infinity_norm}, \ref{lemma:bounding_infinity_2_norm}), in which the operator is mingled with the transform $\mathbf{T}$. Bounding Norms such as $\Vert \mathcal{P}_{\mathbb{S}} \mathit{T} \mathcal{P}_{\Omega} \mathit{T}^{\dag} \mathcal{P}_{\mathbb{S}} \Vert$, $\Vert T^{\text{H}}(\mathcal{U} \star \mathcal{V}^{\text{H}}) \Vert_{\infty}$, and $\Vert \mathcal{P}_{\mathbb{S}}\mathit{T}(\mathdutchcal{e}_{ijk}) \Vert_{\text{F}}^2$ that have not shown up in existing proofs have brought new challenges. Details are in the appendices.            

    %\vspace{-1cm}
    \section{Experiments and Analysis}
        \label{section:exp}
        Numerical experiments are first conducted on random data to study the completion phenomena and then application to visual data impainting is considered.\footnote{The source codes can be found at \href{https://github.com/vecevecev/transformed_TNN}{\color{blue}{https://github.com/vecevecev/transformed\_TNN}}. The experiments can be conducted on a PC with a 8-core CPU and 16GB RAM.} Also, in Section \ref{section:brief_analysis} we present a new analysis explaining why slim transforms can outperform square transforms from a theoretical level. 
        \subsection{Exact Completion of Random Tensors}
            \label{section:exp_random_tensor}
            First, we examine the completion behaviors of Program \eqref{eq:program:min_transformed_TNN} with $\text{rank}(\mathcal{M})$ and sampling rate $p$ both varying. To that goal, we need to generate random tensors $\mathit{T}(\mathcal{M})$ with specified tubal rank $r$. Such generation of $\mathcal{M}$ is not that trivial since the randomly generated $\mathit{T}(\mathcal{M}) \in \mathbb{C}^{n_1 \times n_2 \times N_3}$ with a tubal rank might not have preimage $\mathcal{M} \in \mathbb{C}^{n_1 \times n_2 \times n_3}$ when $\mathit{T}(\cdot)$ is not bijection ($N_3 > n_3$). Thus, to generate random tensors $\mathcal{M}$ satisfying $\mathit{T}(\mathcal{M}) = r$, we design an alternating projection algorithm to iteratively project a randomly initialized tensor onto the target manifold. Due to space constraints, elaborate explanation and corresponding algorithm (Algorithm \ref{alg:gen_2_TX_with_rank_r1_r2}) are provided in Appendix~\ref{section:generate_M}. 
            \begin{figure}[ht]
                \vskip -0.05 in
                \centering
                \begin{subfigure}{0.3\textwidth}
                    \includegraphics[width=0.96\textwidth, height=0.72\textwidth]{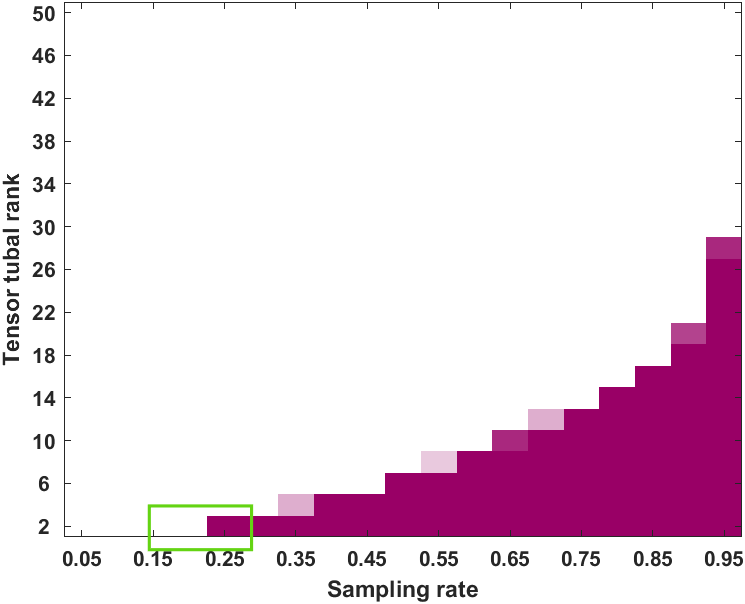}
                    \caption{\scriptsize DFT, $N_3 = n_3$}
                    \label{fig:phase_N3=n3_DFT}
                    \includegraphics[width=0.96\textwidth, height=0.72\textwidth]{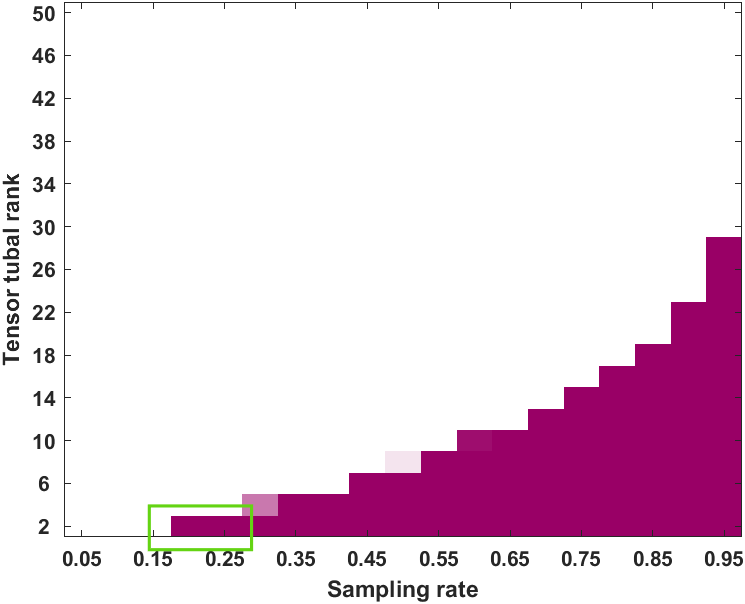}
                    \caption{\scriptsize DFT, $N_3 = 2n_3$}
                    \label{fig:phase_N3=2n3_DFT}
                \end{subfigure}
                \hfill
                \begin{subfigure}{0.3\textwidth}
                    \includegraphics[width=0.96\textwidth, height=0.72\textwidth]{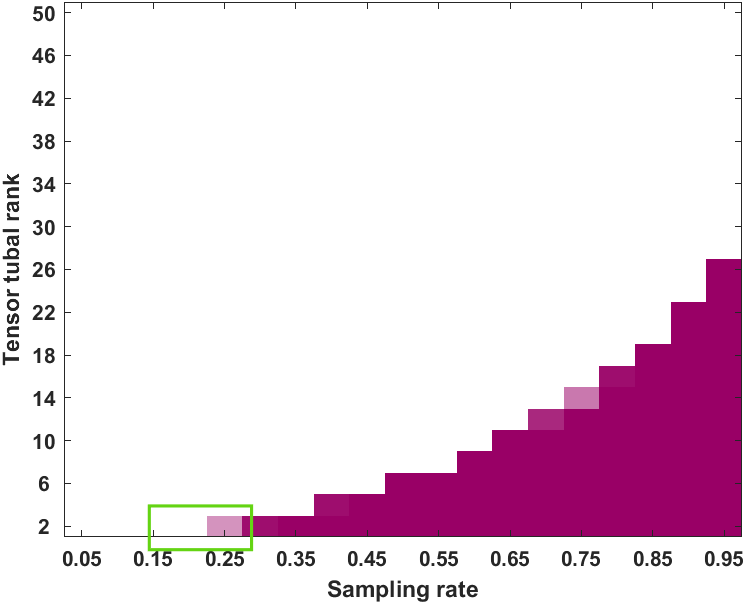}
                    \caption{\scriptsize RUT, $N_3 = n_3$}
                    \label{fig:phase_N3=n3_RUT}
                    \includegraphics[width=0.96\textwidth, height=0.72\textwidth]{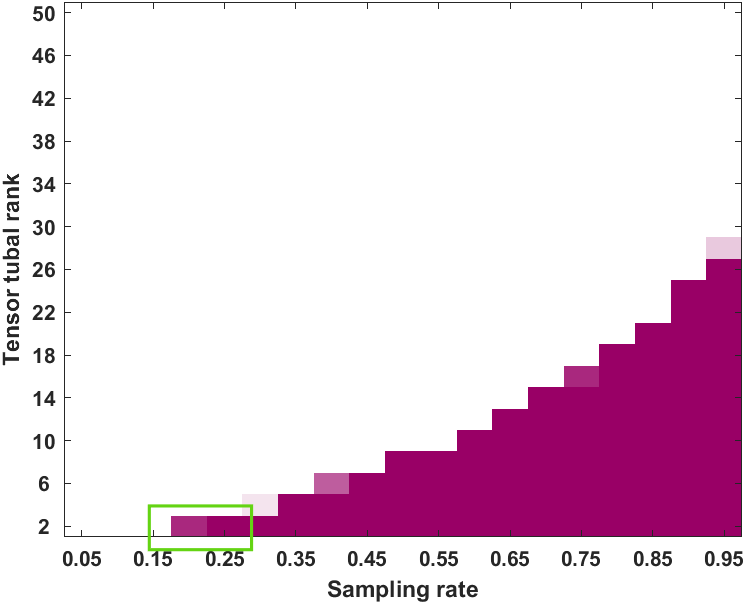}
                    \caption{\scriptsize RUT, $N_3 = 2n_3$}
                    \label{fig:phase_N3=2n3_RUT}
                \end{subfigure}
                \hfill
                \begin{subfigure}{0.3\textwidth}
                    \includegraphics[width=0.96\textwidth, height=0.72\textwidth]{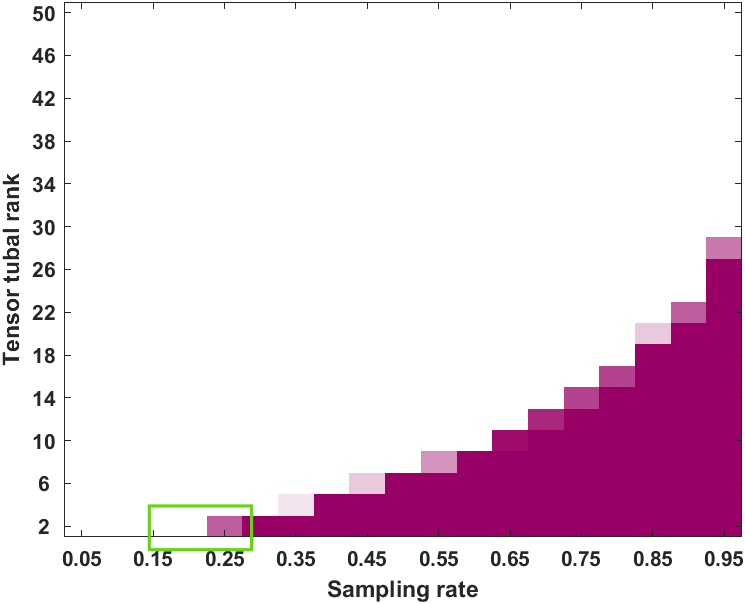}
                    \caption{\scriptsize $\kappa(\mathbf{T}) = 1, N_3 = n_3$}
                    \label{fig:phase_N3=n3_RUT_cond_1}
                    \includegraphics[width=0.96\textwidth, height=0.72\textwidth]{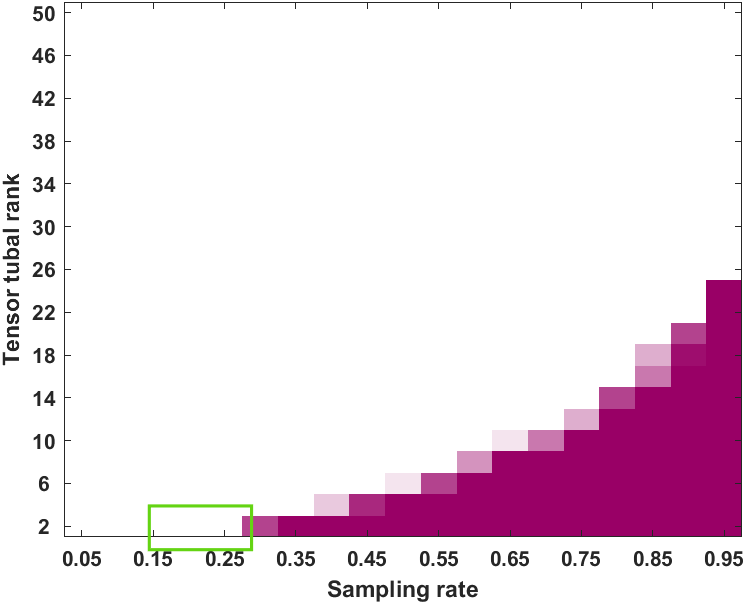}
                    \caption{\scriptsize $\kappa(\mathbf{T}) = 4, N_3 = n_3$}
                    \label{fig:phase_N3=2n3_RUT_cond_4}
                \end{subfigure}
                \caption{\scriptsize Exact completion with varying rank (x-axis) and sampling rate (y-axis). The white cube denotes all failure and the magenta cube denotes all success.}
                \label{fig:phase}
            \end{figure}

            \begin{figure}[ht]
                \vskip -0.3 in
                \centering
                \begin{subfigure}{0.32\textwidth}
                    \centering
                    \includegraphics[width=0.96\textwidth,height=0.84\textwidth]{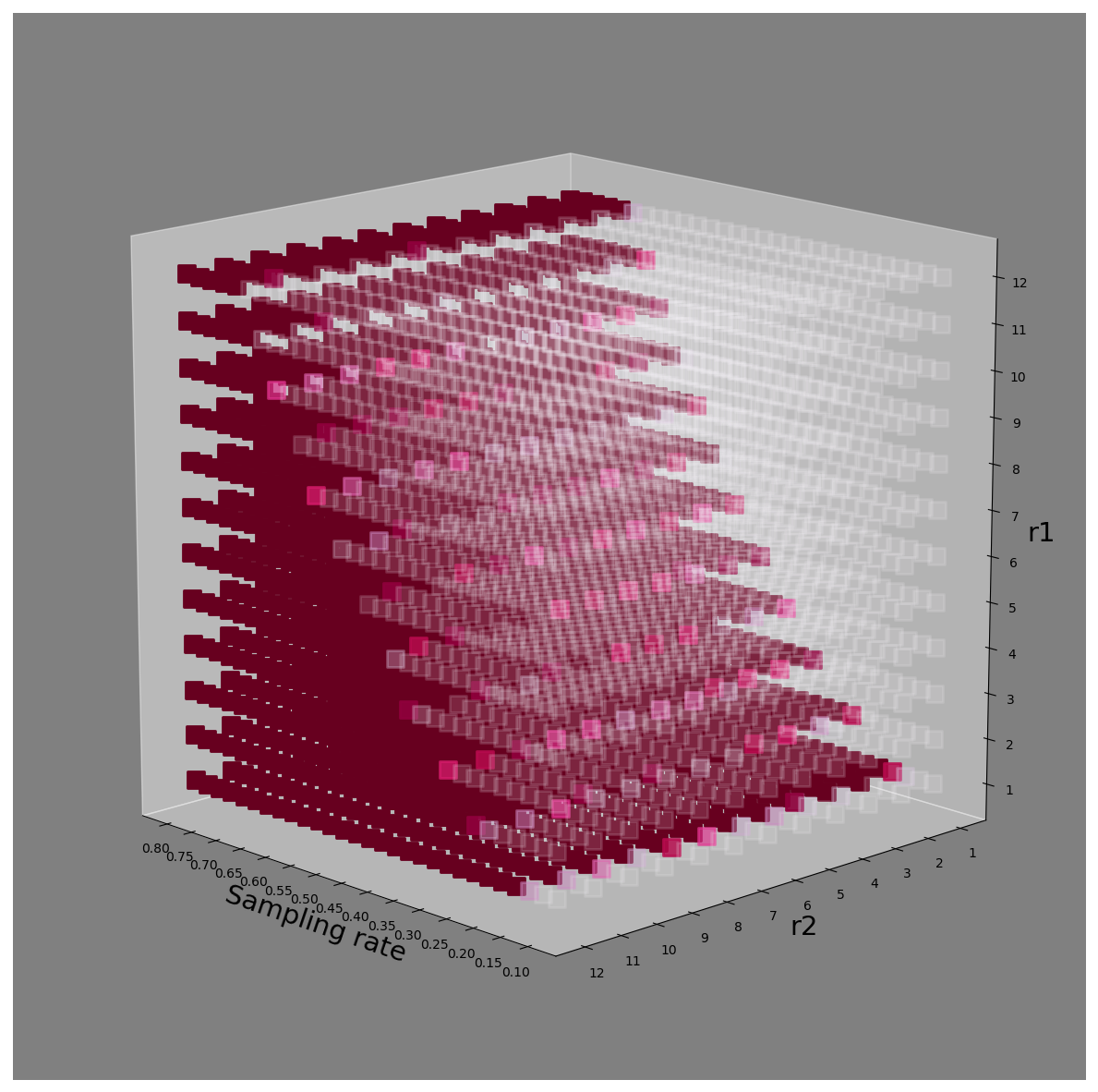}
                    \subcaption{\scriptsize Using RUT1 only}
                    \label{fig:phase_3d-RUT1}
                \end{subfigure}
                \hfill
                \begin{subfigure}{0.32\textwidth}
                    \centering
                    \includegraphics[width=0.96\textwidth,height=0.84\textwidth]{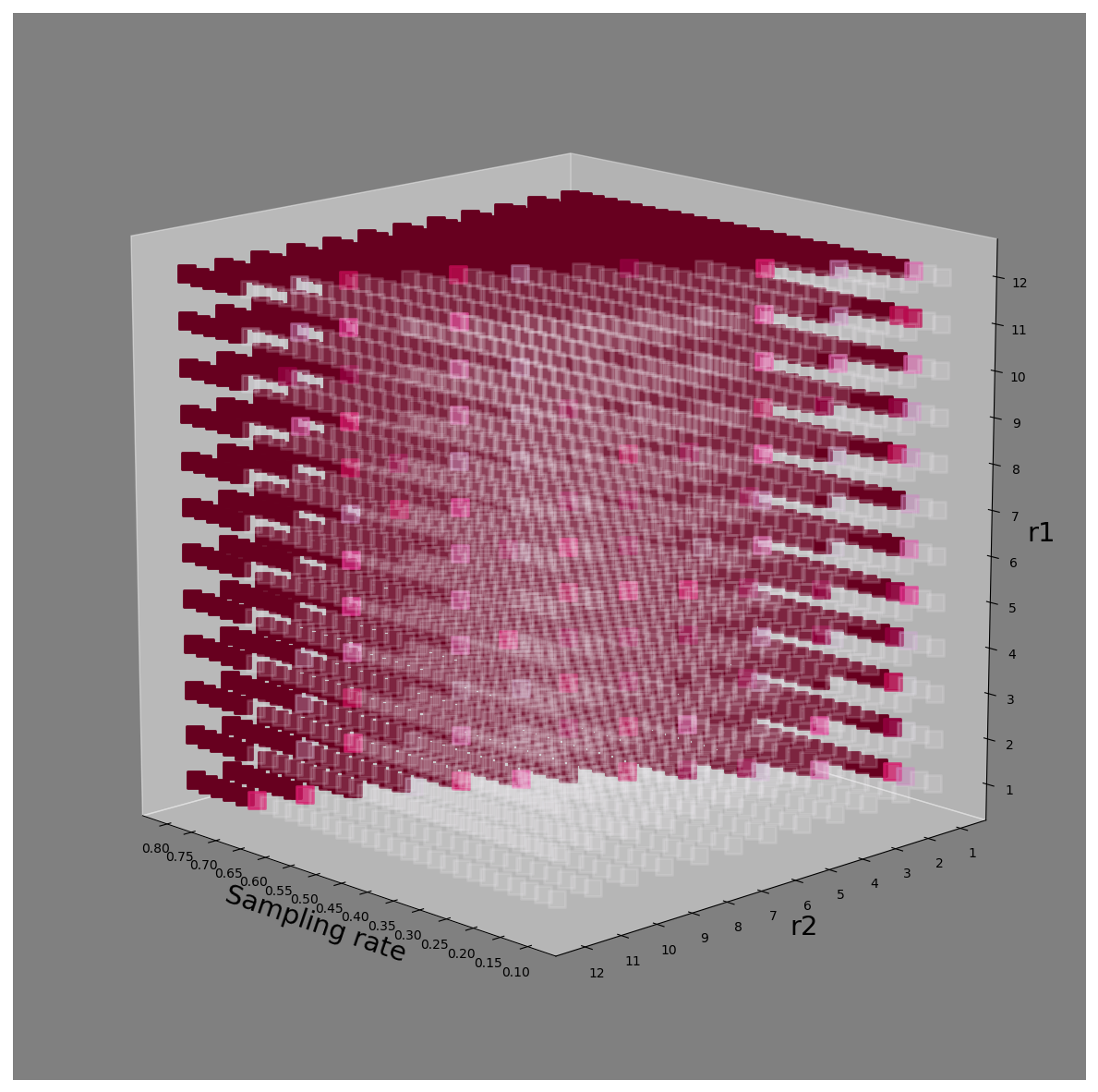}
                    \subcaption{\scriptsize Using RUT2 only}
                    \label{fig:phase_3d-RUT2}
                \end{subfigure}
                \hfill
                \begin{subfigure}{0.32\textwidth}
                    \centering
                    \includegraphics[width=0.96\textwidth,height=0.84\textwidth]{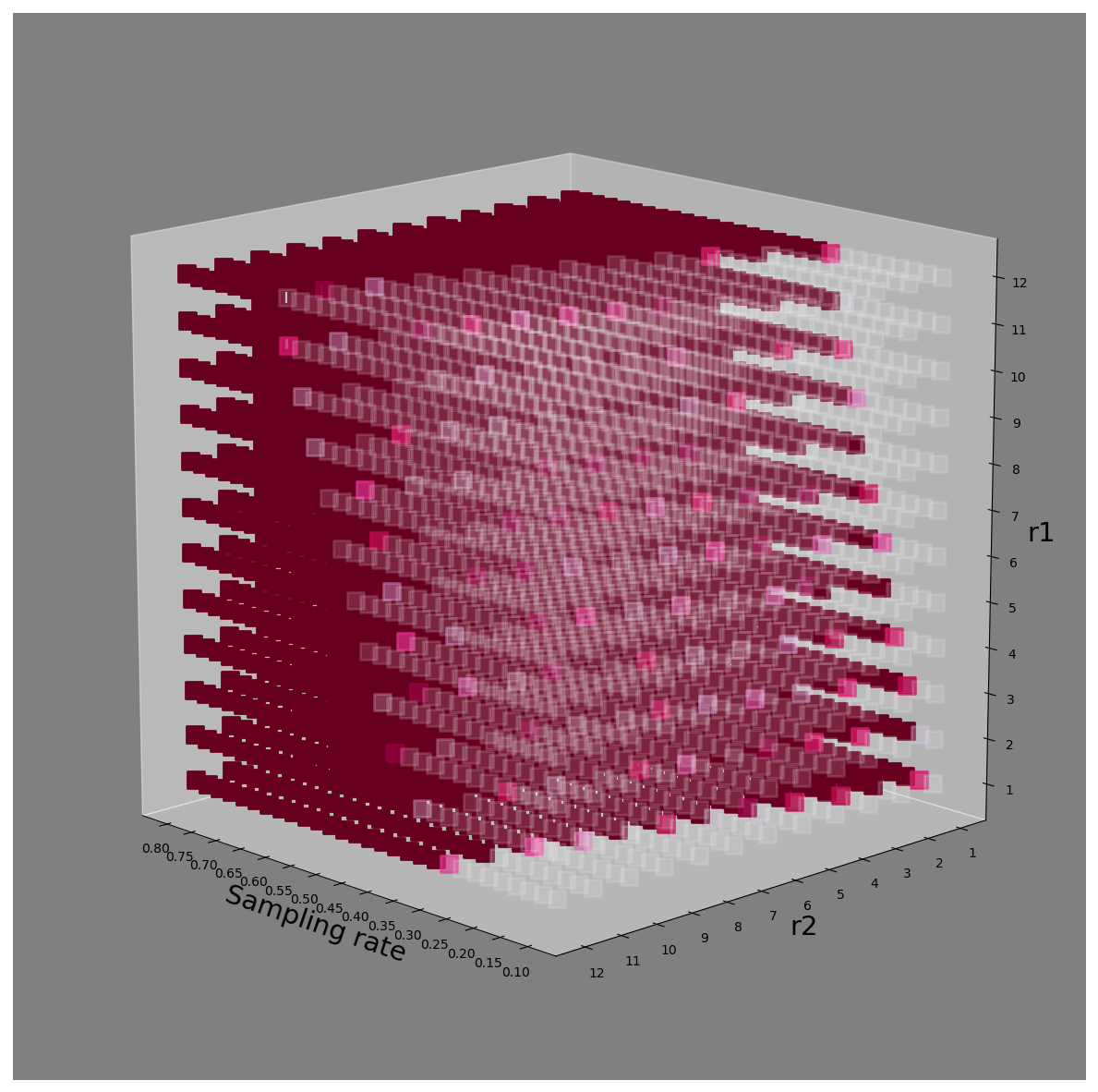}
                    \subcaption{\scriptsize RUT1 and RUT2 combined}
                    \label{fig:phase_3d-combined}
                \end{subfigure}
                \caption{\scriptsize Exact completion with varying $\text{rank}(\mathit{T}_1(\mathcal{X}))$ (z-axis), $\text{rank}(\mathit{T}_2(\mathcal{X}))$ (y-axis) and sampling rate (x-axis). The white cube denotes all failure and the magenta cube denotes all success.}
                \label{fig:phase_3D}
                \vskip -0.1 in
            \end{figure}
                
            Using Algorithm \ref{alg:gen_2_TX_with_rank_r1_r2}, we generate $\mathcal{M} \in \mathbb{C}^{n_1 \times n_2 \times n_3}$ with $n_1 = n_2 = 50, n_3 = 20$ and $N_3 = \{20, 40\}$. We have tested two types of $\mathbf{T}$: DFT and random unitary transform (RUT), using the first $n_3$ columns when $N_3 > n_3$. We have also conducted experiments to see the influence of $\kappa(\mathbf{T})$, in which the singular values of the generated RUT are randomly set between $[0.5, 2]$, resulting $\kappa(\mathbf{T}) = 4$\footnote{Since the singular values are random from $[0.5, 2]$, in some test cases $\kappa(\mathbf{T})$ is a bit smaller than 4.}. The sampling rate (SR) is set as $p = [0.05 : 0.05: 0.95]$
            %\footnote{\label{ft:matlab_command} The MATLAB type command is used here.} 
            and $\text{rank}(\mathcal{M}) = [2 : 2: 50]$ (we use MATLAB notation here). We conducted $10$ test cases for each $(r, p)$ and in each case if $\Vert \mathcal{X} - \mathcal{M} \Vert_{\text{F}} / \Vert \mathcal{X} \Vert_{\text{F}} \leq {10}^{-3}$ we deem $\mathcal{X}$ completed successfully. The parameters are set as $\alpha_0 = 10^{-2}, \alpha_{\max} = 10^{6}, \rho = 1.02$ and the ADMM algorithm terminates when $\max(\Vert \mathcal{X}^t - \mathcal{X}^{t-1}\Vert_{\infty}, \Vert \mathcal{Y}^t - \mathcal{Y}^{t-1}\Vert_{\infty}, \Vert \mathit{T}(\mathcal{X}^t) - \mathit{T}(\mathcal{X}^{t-1})\Vert_{\infty}) \leq 10^{-10}$.

            From Figure \ref{fig:phase}, it can be observed that the completion behaviors are alike among different types of $\mathbf{T}$ and different $N_3$, all showing a region of exact completion in all cases which supports Theorem \ref{theorem:exact_completion}. It can also be observed by comparing Figure \ref{fig:phase_N3=n3_DFT} and Figure \ref{fig:phase_N3=2n3_DFT}, Figure \ref{fig:phase_N3=n3_RUT} and Figure \ref{fig:phase_N3=2n3_RUT} that, with the same tubal ranks, the transforms that map $\mathcal{X}$ into higher dimensional space outperforms their counterparts, which validates the advantage of such non-invertible transforms. Additionally, from Figure \ref{fig:phase_N3=n3_RUT_cond_1} and Figure \ref{fig:phase_N3=2n3_RUT_cond_4}, the transform with $\kappa(\mathbf{T}) = 1$ exhibits better recovery efficacy than the transform with $\kappa(\mathbf{T}) = 4$, also supporting Theorem \ref{theorem:exact_completion}.

           Such non-invertible slim transforms can provide more information since more prior knowledge of $\mathcal{X}$ is available. To illustrate this, consider two transforms $\mathbf{T}_1$ and $\mathbf{T}_2$. We generate $\mathcal{X}$ by Algorithm \ref{alg:gen_2_TX_with_rank_r1_r2} to satisfy $\text{rank}(\mathit{T}_1(\mathcal{X})) = r_1$ and $\text{rank}(\mathit{T}_2(\mathcal{X})) = r_2$. Since $\mathbf{T}_1$ and $\mathbf{T}_2$ provide information of $\mathcal{X}$ in different domains, the concatenated slim transform $\mathbf{T}_3 = [\mathbf{T}_1^{\text{T}}, \mathbf{T}_2^{\text{T}}]^{\text{T}}$ ensembles richer information than its individual components. In Figure \ref{fig:phase_3D}, we set $\mathbf{T}_1$ and $\mathbf{T}_2$ as two RUTs, the slim $\mathbf{T}_3$ is cascaded as $\mathbf{T}_3 = \frac{1}{\sqrt{2}}[\mathbf{T}_1^{\text{T}}, \mathbf{T}_2^{\text{T}}]^{\text{T}}$ with $r_1, r_2 = [1 : 12]$, $p = [0.1 : 0.025 : 0.8]$ and plot 3D phase graphs.

            It can be observed from Figure \ref{fig:phase_3D} that: 1) Since in Figure \ref{fig:phase_3d-RUT1} (Figure \ref{fig:phase_3d-RUT2}), $\mathbf{T}_1$ ($\mathbf{T}_2$) is the only transform, its performance basically remains fixed with $r_2$ ($r_1$) varying; 2) As shown in Figure \ref{fig:phase_3d-combined}, though the tubal rank of $\mathit{T}_3(\mathcal{X})$ has not decreased ($\text{rank}(\mathit{T}_3(\mathcal{X})) = \text{rank}(\mathit{T}_1(\mathcal{X})) \vee \text{rank}(\mathit{T}_2(\mathcal{X}))$), compared to using $\mathbf{T}_1$/$\mathbf{T}_2$ only, the exact completion region is enlarged giving credit to the utilization of the additional information provided by another transform. Such a simple and natural combinational design of transform $\mathit{T}(\cdot)$ provides a perspective of understanding why slim transforms may have better performances.

             \subsection{Why Slim Transforms Are Better? An Analysis Obtained by Proof}
             \label{section:brief_analysis}
             Previous works \cite{Framelet_TNN,dctionary_learning_TNN,S2NTNN} give an intuitive explanation that slim transforms can induce redundancy and achieve a better low-rank representation in the transform domain, but no theoretical justification is provided. Here, we use our proof to present a justification. From our proof, it can be observed that the sampling rate $p$ to achieve perfect completion is proportional to two energy terms: $E_1 = \max_{ijk} \{\Vert \mathcal{P}_{\mathbb{S}}\mathbf{T}(\mathdutchcal{e}_{ijk}) \Vert_{\text{F}}^2 \vee \Vert \mathcal{P}_{\mathbb{S}}\tilde{\mathbf{T}}(\mathdutchcal{e}_{ijk}) \Vert_{\text{F}}^2\}, E_2 = \max_{ijk} \Vert \mathbf{T}^{\dag}\mathcal{P}_{\mathbb{S}}\mathbf{T}(\mathdutchcal{e}_{ijk}) \Vert_{\text{F}}^2$. Let us use a special class of slim transforms, i.e., combinational transforms to see how the value of $N_3$ impacts $E_1$ and $E_2$.\footnote{There are other terms that affect $p$. The formulations and analysis of them are similar to $E_1$ and $E_2$. Please find a complete analysis in Appendix \ref{section:ananlyzing_sampling_rate}.}
             \begin{figure}[ht]
                \vskip -0.1 in
                \centering
                \begin{minipage}{0.45\textwidth}
                    \centering
                    \includegraphics[width=0.9\textwidth, height=0.55\textwidth]{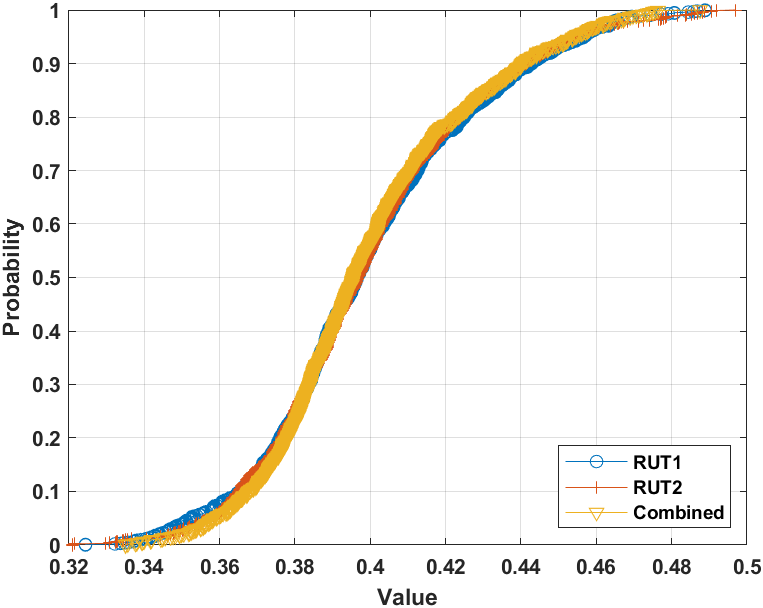}
                    \subcaption{\scriptsize $\Vert \mathcal{P}_{\mathbb{S}}\mathbf{T}(\mathdutchcal{e}_{ijk}) \Vert_{\text{F}}^2 \vee \Vert \mathcal{P}_{\mathbb{S}}\tilde{\mathbf{T}}(\mathdutchcal{e}_{ijk}) \Vert_{\text{F}}^2$}
                    \label{fig:E1}
                \end{minipage}
                \begin{minipage}{0.45\textwidth}
                    \centering
                    \includegraphics[width=0.9\textwidth, height=0.55\textwidth]{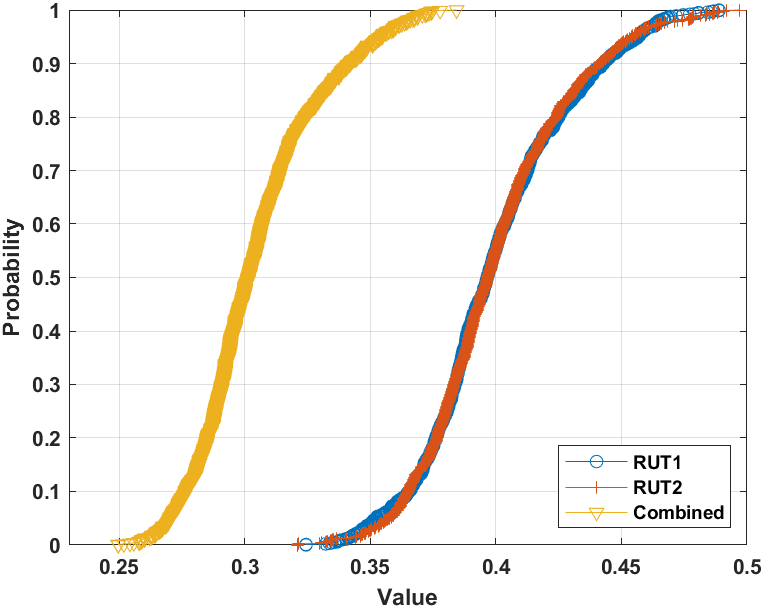}
                    \subcaption{\scriptsize $\Vert \mathbf{T}^{\dag}\mathcal{P}_{\mathbb{S}}\mathbf{T}(\mathdutchcal{e}_{ijk}) \Vert_{\text{F}}^2$}
                    \label{fig:E2}
                \end{minipage}
                \caption{\scriptsize Empirical cumulative distribution of two energy terms. Note that although the bound is given by $\max\{\cdot\}$, actually the distribution plays an important role.}
                \label{fig:energy}
                \vskip -0.2 in
            \end{figure}

             First, we analyze $E_1$. Consider that our slim transform $\mathbf{T}_{\text{c}}$ is formed by concatenating a list of square balanced transforms as $\mathbf{T}_{\text{c}} = \frac{1}{\sqrt{n}}[\mathbf{T}_1^{\text{T}}, \mathbf{T}_2^{\text{T}}, \hdots, \mathbf{T}_M^{\text{T}}]^{\text{T}}$, where the normalizier $\frac{1}{\sqrt{n}}$ is to ensure that the combined $\mathbf{T}_{\text{c}}$ is also balanced. Since the tensor-tensor product is operated slice-wise, the following relations can be acquired: $\forall~i,j,k,~ \Vert \mathcal{P}_{\mathbb{S}}\mathbf{T}_{c}(\mathdutchcal{e}_{ijk}) \Vert_{\text{F}}^2 = \text{mean}\{ \Vert \mathcal{P}_{\mathbb{S}}\mathbf{T}_m(\mathdutchcal{e}_{ijk}) \Vert_{\text{F}}^2 \},~\Vert \mathcal{P}_{\mathbb{S}}\tilde{\mathbf{T}}_{\text{c}}(\mathdutchcal{e}_{ijk}) \Vert_{\text{F}}^2 = \text{mean}\{\Vert \mathcal{P}_{\mathbb{S}}\tilde{\mathbf{T}_m}(\mathdutchcal{e}_{ijk}) \Vert_{\text{F}}^2\}, m = 1, 2, \hdots, M$, which means that the two norms in $E_1$ of the slim transform $\mathbf{T}_{c}$ are averaged by those of its components. After taking the maximum w.r.t. $i,j,k$, it becomes an inequality: $E_1(\mathbf{T}_c) \leq \text{mean}\{E_1(\mathbf{T}_m)\}$.\footnote{This inequality will be observed in Figure \ref{fig:arr_v_1_1} in Appendix \ref{section:ananlyzing_sampling_rate}.}
        
             However, the relations become different when analyzing $E_2$ due to ${\mathbf{T}_{\text{c}}}^{\dag}\mathcal{P}_{\mathbb{S}}\mathbf{T}_{\text{c}}(\mathdutchcal{e}_{ijk})$. One observes that: $\forall~i,j,k,~ \mathbf{T}_{\text{c}}^{\dag}\mathcal{P}_{\mathbb{S}}\mathbf{T}_{\text{c}}(\mathdutchcal{e}_{ijk}) = \text{mean}(\mathbf{T}_{m}^{\dag}\mathcal{P}_{\mathbb{S}}\mathbf{T}_m(\mathdutchcal{e}_{ijk})), m = 1, \hdots, M$, that is, instead of the energy, the (entries of) tensors themselves are the mean of the components. If we assume the entries in $\mathbf{T}_{m}^{\dag}\mathcal{P}_{\mathbb{S}}\mathbf{T}_m(\mathdutchcal{e}_{ijk})$ are i.i.d., Central Limit Theorem implies that the energy/variance of $\mathbf{T}_{m}^{\dag}\mathcal{P}_{\mathbb{S}}\mathbf{T}_m(\mathdutchcal{e}_{ijk})$ is less than the energy of each of its components, thus allowing for a smaller sampling rate $p$.
             
             To validate this analysis, we conduct an experiment to compare the energies. Figure \ref{fig:energy} shows the empirical cumulative distribution of $\Vert \mathcal{P}_{\mathbb{S}}\mathbf{T}(\mathdutchcal{e}_{ijk}) \Vert_{\text{F}}^2 \vee \Vert \mathcal{P}_{\mathbb{S}}\tilde{\mathbf{T}}(\mathdutchcal{e}_{ijk}) \Vert_{\text{F}}^2$ and $\Vert \mathbf{T}^{\dag}\mathcal{P}_{\mathbb{S}}\mathbf{T}(\mathdutchcal{e}_{ijk}) \Vert_{\text{F}}^2$ under $\mathbf{T}_1$, $\mathbf{T}_2$ and $\mathbf{T}_{\text{c}}$. One can observe that the former is averaged and the latter has been shrunken.

            \subsection{Application to Visual Data Inpainting}
             
             We conduct experiments on visual data.\footnote{The emphasis of this section is to show the impacts of different transforms. Thus, we have not included methods with different tensor structures and assumptions, such as CP, Tucker, tensor ring, etc.} The transforms used for comparison include $n_3 \times n_3$ DFT, i.e., TNN, $n_3 \times n_3$ DCT, the first $n_3$ columns of $N_3 \times N_3$ DFT matrix and DCT matrix with $N_3 = 2 n_3$ (DFT2, DCT2), $n_3 \times n_3$ DFT/DCT cantencated transform (DFT-DCT), discrete Walsh-Hadamard transform (DWHT) and Framelet \cite{Framelet_TNN} with piecewise cubic B-spline and decomposition level $l = 4$. The parameters are $\alpha_0 = 10^{-2}, \alpha_{\max} = 10^{6}, \rho = 1.2$ and the ADMM algorithm terminates when $\max(\Vert \mathcal{X}^t - \mathcal{X}^{t-1}\Vert_{\infty}, \Vert \mathcal{Y}^t - \mathcal{Y}^{t-1}\Vert_{\infty}, \Vert \mathit{T}(\mathcal{X}^t) - \mathit{T}(\mathcal{X}^{t-1})\Vert_{\infty}) \leq 10^{-3}$ ($10^{-4}$ for Framelet). 
            
             \subsubsection{MSI Data}                
                We apply our model on MSI data, of which the third dimension contains rich spectrum information. Each MSI is of size $512 \times 521 \times 31$ (height $\times$ width $\times$ band). The methods are tested on MSI ``super balls'', ``cloth'', ``stuffed toys'', ``photo and face'', ``oil painting'' and ``clay''. The evaluation metrics adopted are the peak signal-to-noise ratio (PSNR) and the structural similarity index measure (SSIM).
                \begin{table}[H]
                    \vskip -0.2 in
                    \setlength\tabcolsep{5pt}
                    \centering
                    \caption{\scriptsize The average performances on MSI and video data. The best and the second best results are highlighted by bolder and italic fonts, respectively.} 
                    \label{tb:MSI_Video}
                    \begin{scriptsize}
                    \begin{tabular}{cccccccccc}
                        \toprule Data & SR & Metric & DWHT & DFT & DFT2 & DCT & DCT2 & DFT-DCT & Framelet\\
                        \midrule
                        \multirow{6}*{MSI} & \multirow{2}*{0.05} & PSNR & 31.1468 & 31.9586 & 32.5784 & 33.0083 & 33.3522 & \textit{33.7772} & \textbf{34.2920}\\
                        & & SSIM & 0.8256 & 0.8465 & 0.8746  & 0.8724 & 0.8910 & \textit{0.8942} & \textbf{0.9151} \\
                        \cmidrule(lr){2-10}
                        & \multirow{2}*{0.1} & PSNR & 35.6833 & 37.0205 & 37.4532 & 38.1529 & 38.2279 & \textit{38.7927} & \textbf{39.1917} \\
                        & & SSIM & 0.9186 & 0.9373 & 0.9489 & 0.9512 & 0.9560 & \textit{0.9594} & \textbf{0.9676} \\
                        \cmidrule(lr){2-10}
                        & \multirow{2}*{0.15} & PSNR & 38.1754 & 39.7736 & 40.2428 & 40.9420 & 40.9023 & \textit{41.4876} & \textbf{41.8892} \\
                        & & SSIM & 0.9478 & 0.9627 & 0.9705 & 0.9722 & 0.9741 & \textit{0.9764} & \textbf{0.9808} \\
                        \cmidrule(lr){1-10} %%%%%%%%%%
                        \multirow{6}*{Video} & \multirow{2}*{0.05} & mPSNR & 23.1675 & 23.1469 & 23.8319 & 23.3915 & \textit{24.1659} & 24.0859 & \textbf{24.5394}\\
                        & & mSSIM & 0.6046 & 0.5849 & 0.6499 & 0.5948 & \textit{0.6611} & 0.6525 & \textbf{0.7008} \\
                        \cmidrule(lr){2-10}
                        & \multirow{2}*{0.1} & mPSNR & 25.7329 & 25.7384 & 26.6074 & 25.9600 & \textit{26.8386} & 26.7433 & \textbf{27.2258} \\
                        & & mSSIM & 0.7099 & 0.6858 & 0.7526 & 0.6937 & \textit{0.7570} & 0.7488 & \textbf{0.7870} \\
                        \cmidrule(lr){2-10}
                        & \multirow{2}*{0.15} & mPSNR & 27.0870 & 27.1902 & 28.1257 & 27.4041 & \textit{28.2968} & 28.1984 & \textbf{28.7145} \\
                        & & mSSIM & 0.7599 & 0.7408 & 0.8017 & 0.7470 & \textit{0.8032} & 0.7971 & \textbf{0.8308}\\
                        \bottomrule
                    \end{tabular}
                    \end{scriptsize}
                    \vskip -0.2 in
                \end{table}

                \begin{figure}[ht]
                    \begin{subfigure}{0.45\textwidth}
                        \centering
                        \includegraphics[width=0.96\textwidth,height=0.54\textwidth]{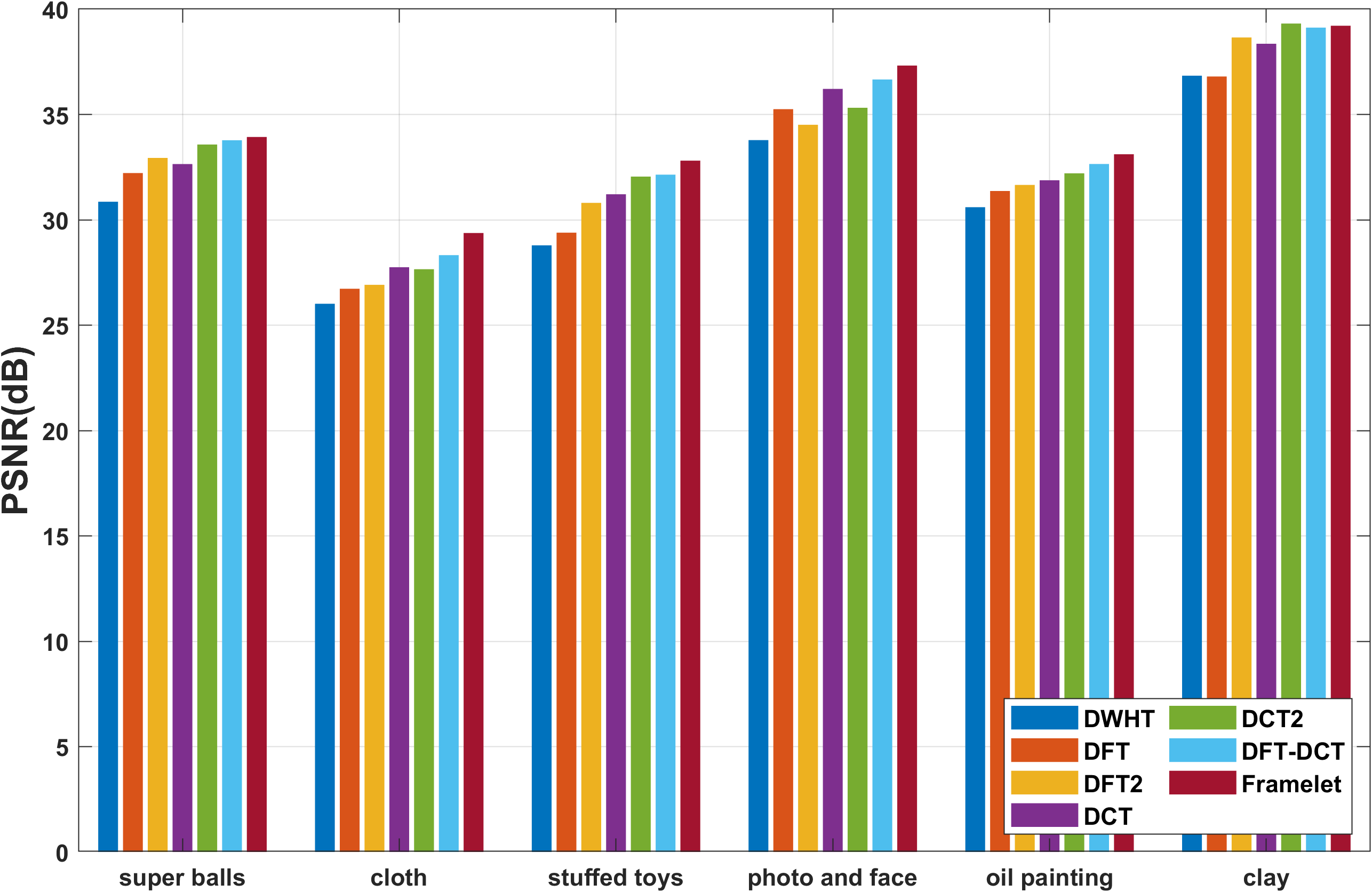}
                        %\caption{\scriptsize PSNR}
                        \label{fig:msi_PSNR}
                    \end{subfigure}
                    \hfill
                    \begin{subfigure}{0.45\textwidth}
                        \centering
                        \includegraphics[width=0.96\textwidth,height=0.54\textwidth]{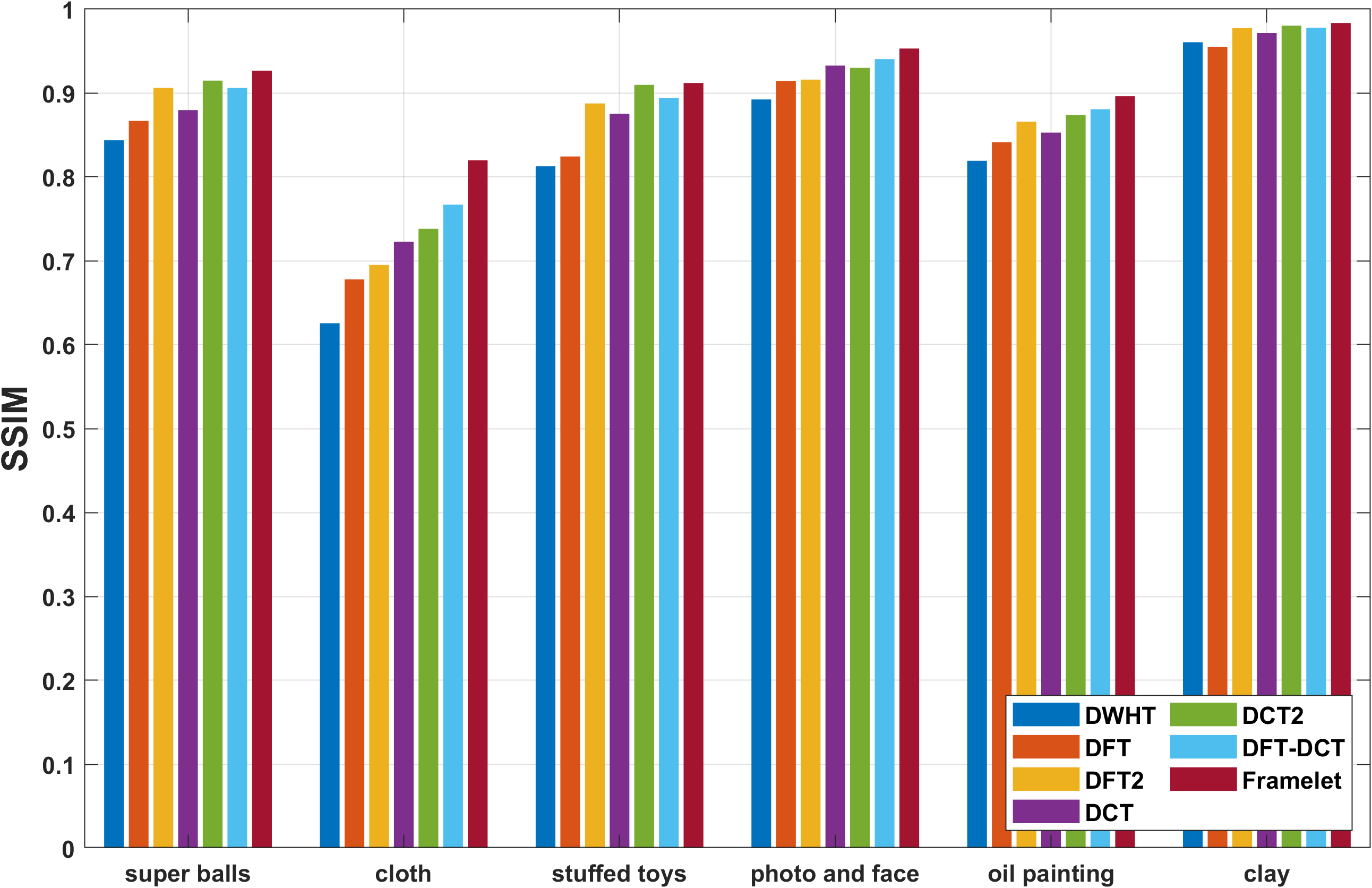}
                        %\caption{\scriptsize SSIM}
                        \label{fig:msi_SSIM}
                    \end{subfigure}
                    \caption{\scriptsize PSNR (left) and SSIM (right) results on MSI data with SR =0.05.}
                    \label{fig:msi_ind_p=0.05}
                    \vskip -0.12 in
                \end{figure}
                
                The top half of Table \ref{tb:MSI_Video} shows the average PSNR and SSIM value on MSI data for SR$=0.05, 0.1, 0.15$. Figure \ref{fig:msi_ind_p=0.05} lists the indicators for each MSI. It is observable that the non-invertible slim transforms, i.e., DFT2 and DCT2, outperform their square counterparts from which they originate. The Framelet transform and DFT-DCT achieve the best and the second best results, both being slim (the Framelet transform is slimmer as we set $l=4$). It is worth noting that the recovery performance of DFT-DCT is quite satisfactory by further considering that such design is simple and relatively efficient (the dimension is only doubled after the transform). Naturally more transforms could be included and each of them provides information about the original data from a different aspect/domain, which reflects the idea of ensemble learning or information fusion.

                \begin{figure}[ht]
                \vskip -0.1 in
                \centering
                \setlength{\belowcaptionskip}{-5pt}
                \begin{minipage}{0.102\textwidth}
                		\centering
                		\includegraphics[width=1\textwidth, height=0.75\textwidth]{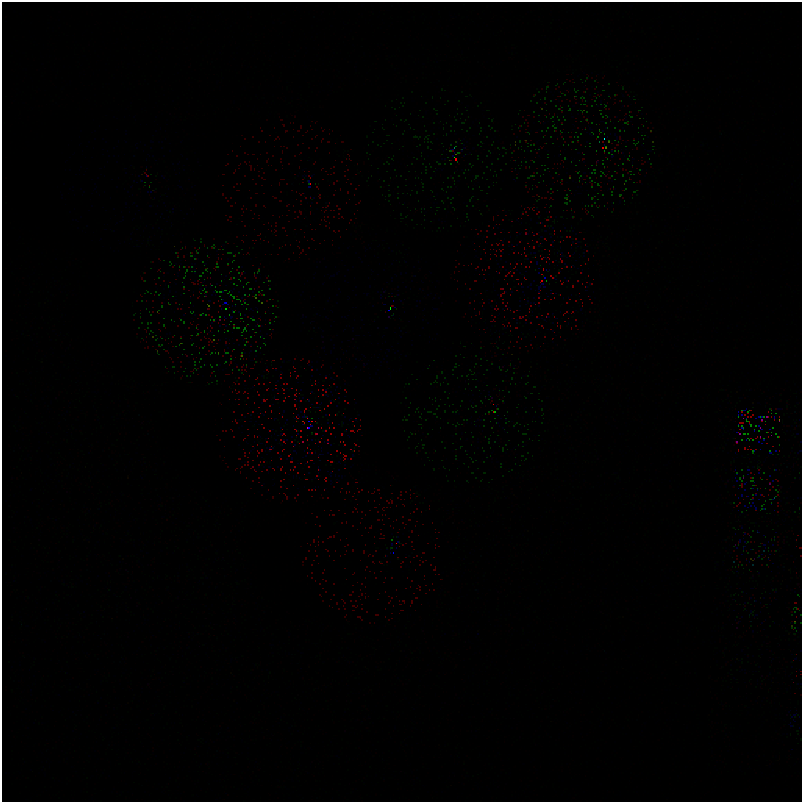}
                	\end{minipage}
                    \begin{minipage}{0.102\textwidth}
                		\centering
                		\includegraphics[width=1\textwidth, height=0.75\textwidth]{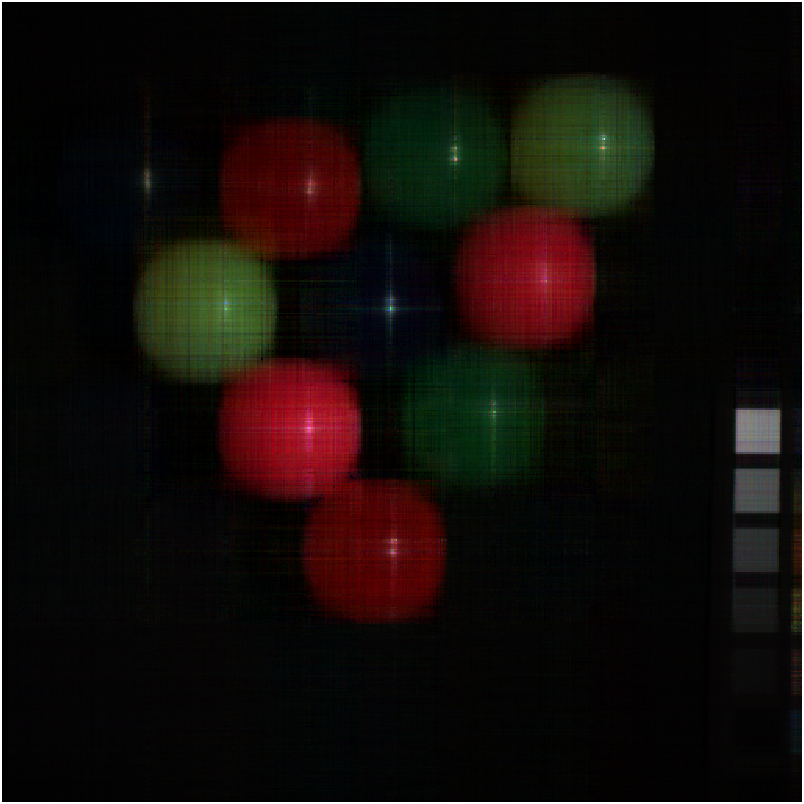}
                	\end{minipage}
                    \begin{minipage}{0.102\textwidth}
                		\centering
                		\includegraphics[width=1\textwidth, height=0.75\textwidth]{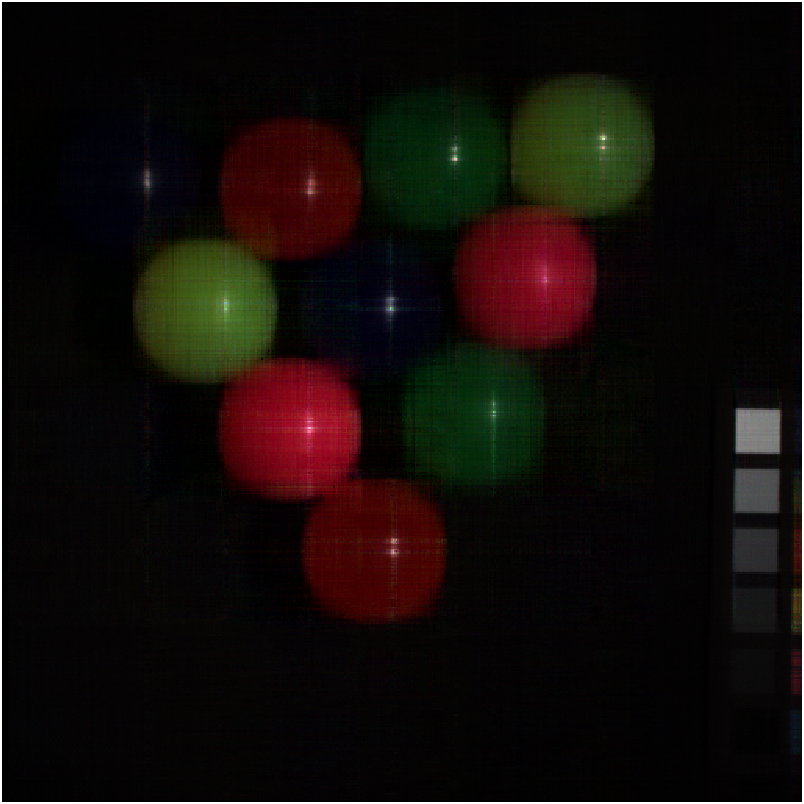}
                	\end{minipage}
                    \begin{minipage}{0.102\textwidth}
                		\centering
                		\includegraphics[width=1\textwidth, height=0.75\textwidth]{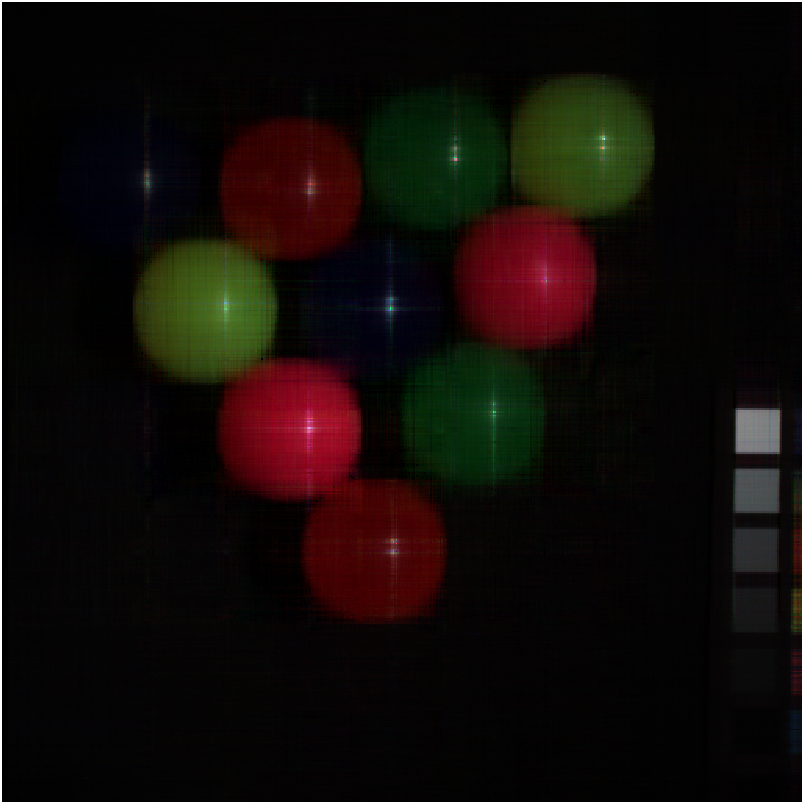}
                	\end{minipage}
                    \begin{minipage}{0.102\textwidth}
                		\centering
                		\includegraphics[width=1\textwidth, height=0.75\textwidth]{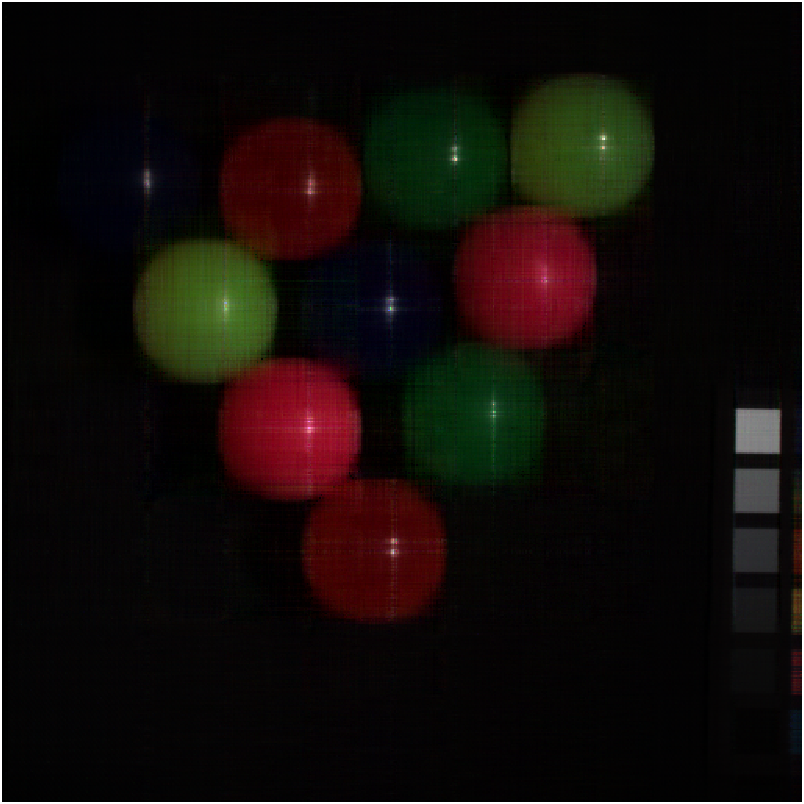}
                	\end{minipage}
                    \begin{minipage}{0.102\textwidth}
                		\centering
                		\includegraphics[width=1\textwidth, height=0.75\textwidth]{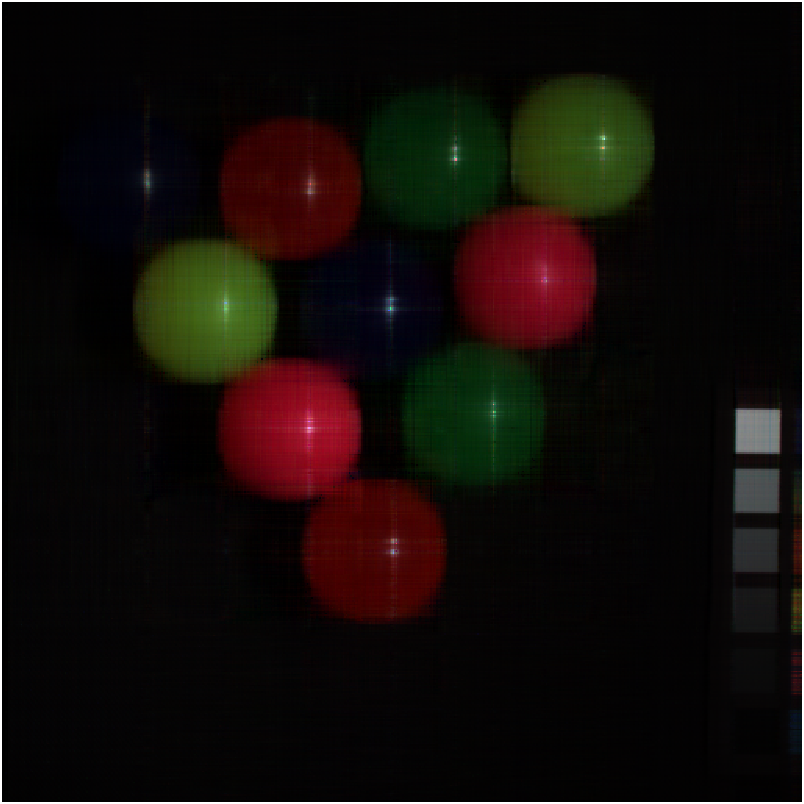}
                	\end{minipage}
                    \begin{minipage}{0.102\textwidth}
                		\centering
                		\includegraphics[width=1\textwidth, height=0.75\textwidth]{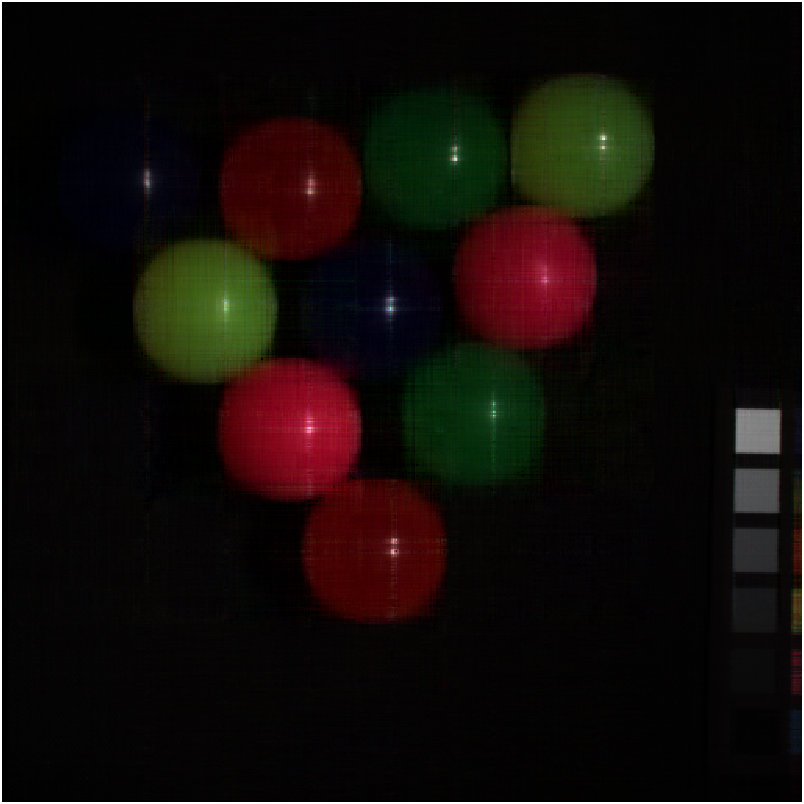}
                	\end{minipage}
                    \begin{minipage}{0.102\textwidth}
                		\centering
                		\includegraphics[width=1\textwidth, height=0.75\textwidth]{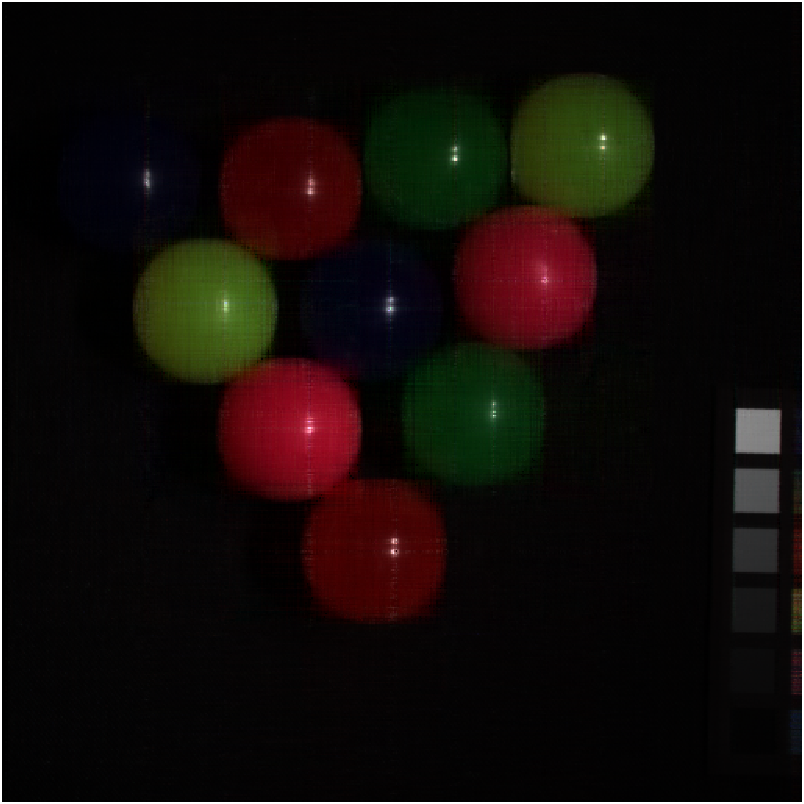}
                	\end{minipage}
                    \begin{minipage}{0.102\textwidth}
                		\centering
                		\includegraphics[width=1\textwidth, height=0.75\textwidth]{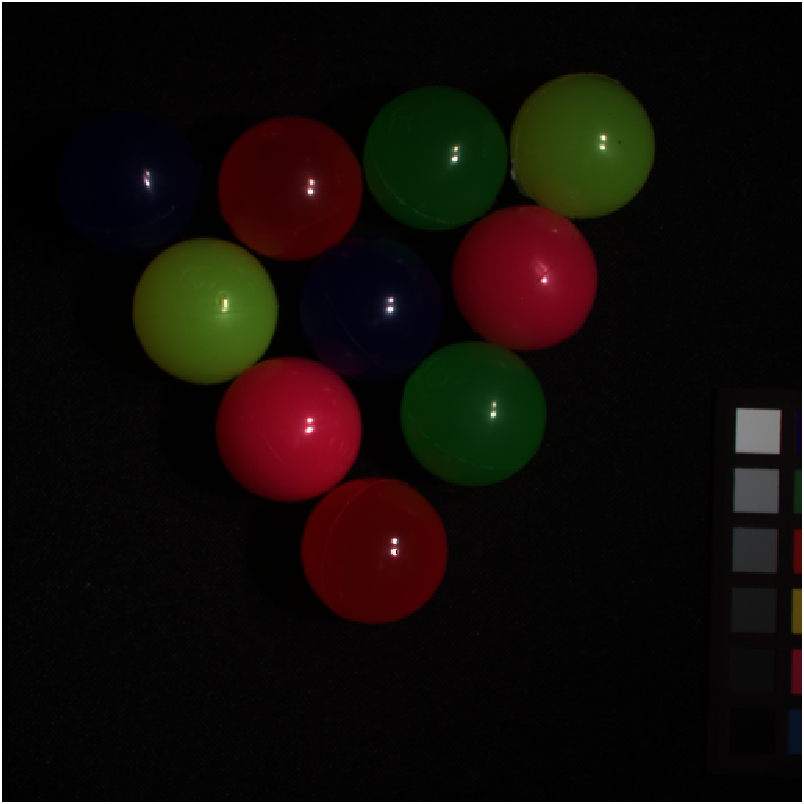}
                	\end{minipage}
                    \begin{minipage}{0.102\textwidth}
                		\centering
                		\includegraphics[width=1\textwidth, height=0.75\textwidth]{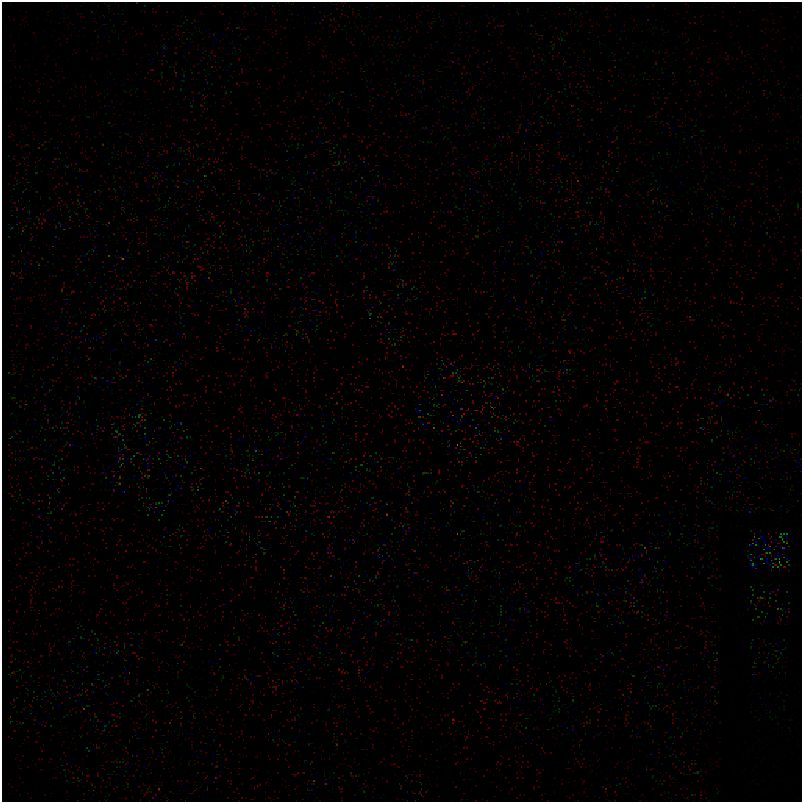}
                	\end{minipage}
                    \begin{minipage}{0.102\textwidth}
                		\centering
                		\includegraphics[width=1\textwidth, height=0.75\textwidth]{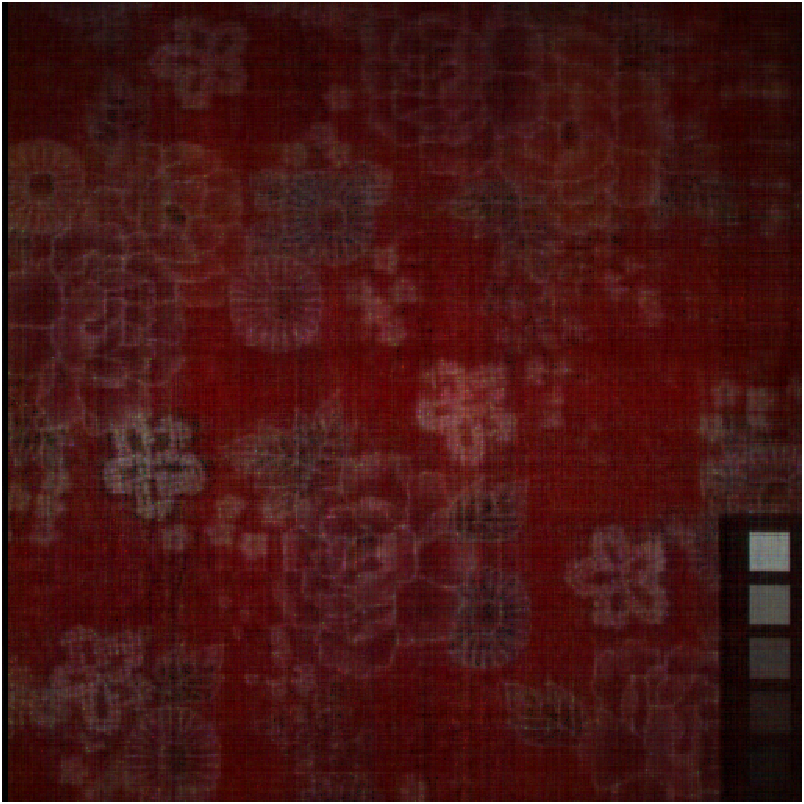}
                	\end{minipage}
                    \begin{minipage}{0.102\textwidth}
                		\centering
                		\includegraphics[width=1\textwidth, height=0.75\textwidth]{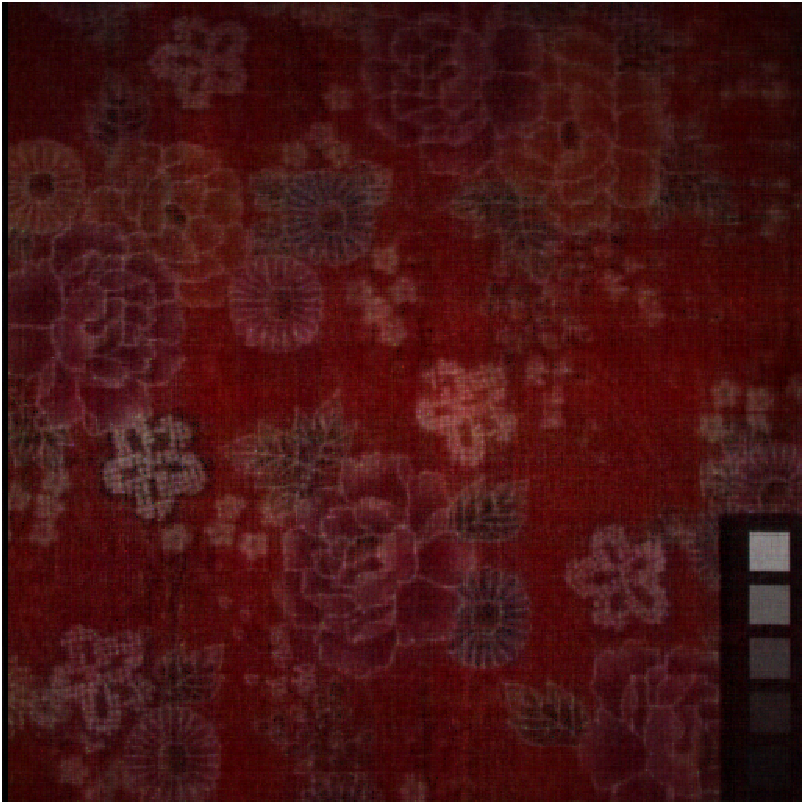}
                	\end{minipage}
                    \begin{minipage}{0.102\textwidth}
                		\centering
                		\includegraphics[width=1\textwidth, height=0.75\textwidth]{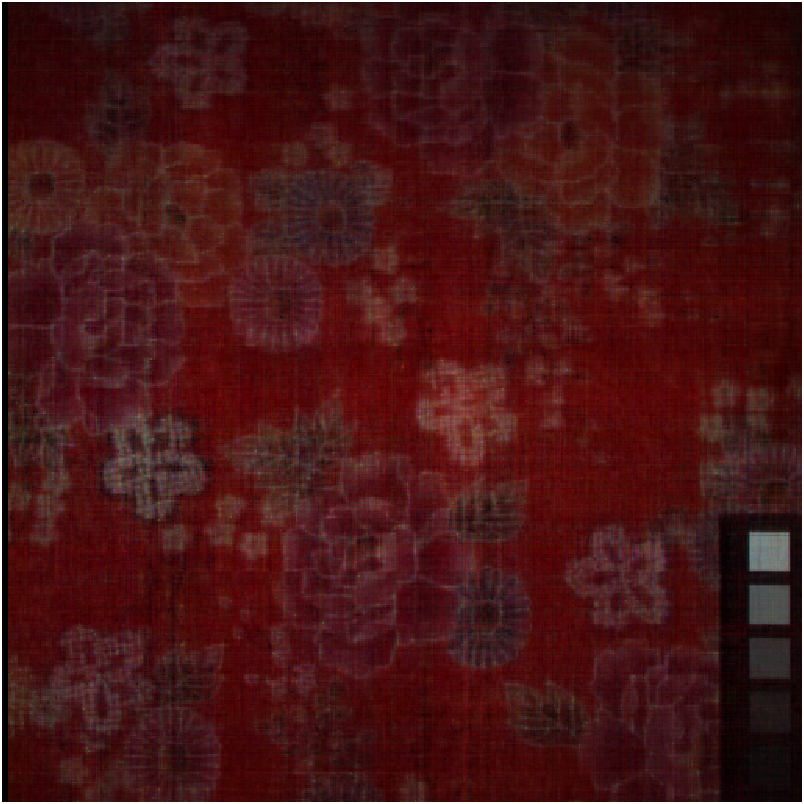}
                	\end{minipage}
                    \begin{minipage}{0.102\textwidth}
                		\centering
                		\includegraphics[width=1\textwidth, height=0.75\textwidth]{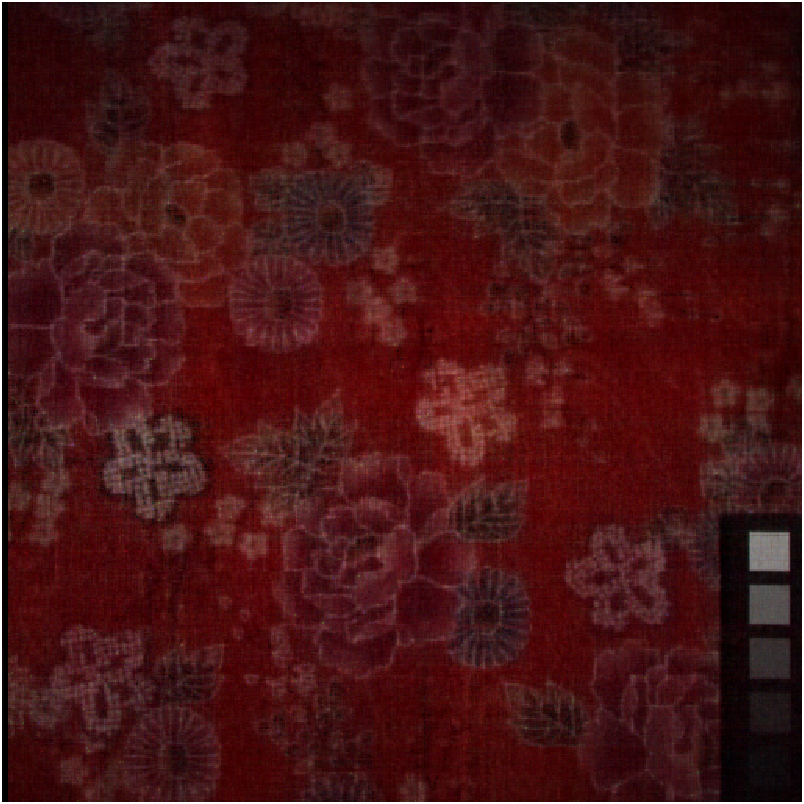}
                	\end{minipage}
                    \begin{minipage}{0.102\textwidth}
                		\centering
                		\includegraphics[width=1\textwidth, height=0.75\textwidth]{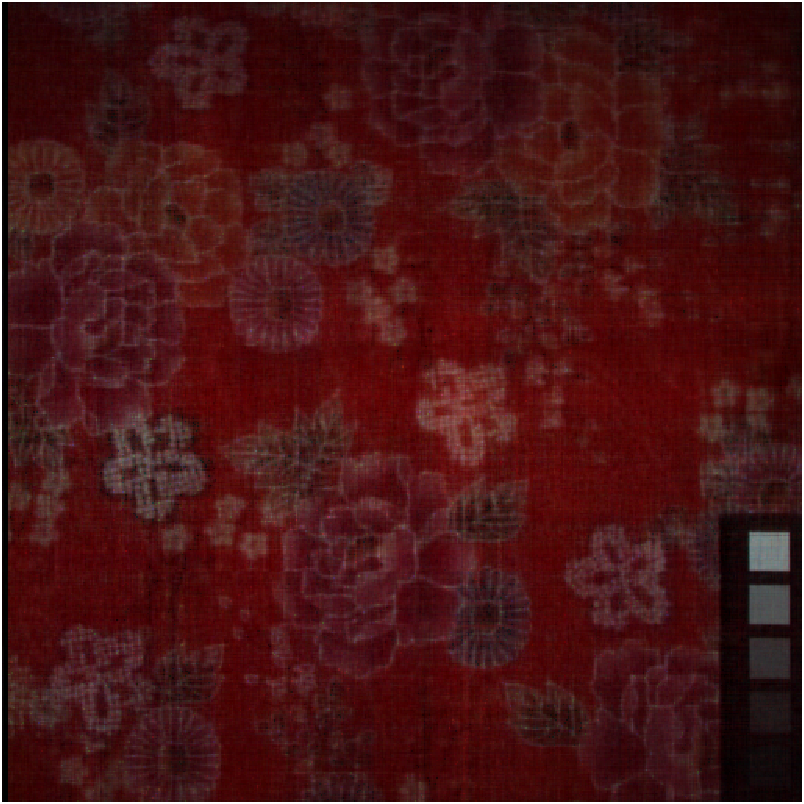}
                	\end{minipage}
                    \begin{minipage}{0.102\textwidth}
                		\centering
                		\includegraphics[width=1\textwidth, height=0.75\textwidth]{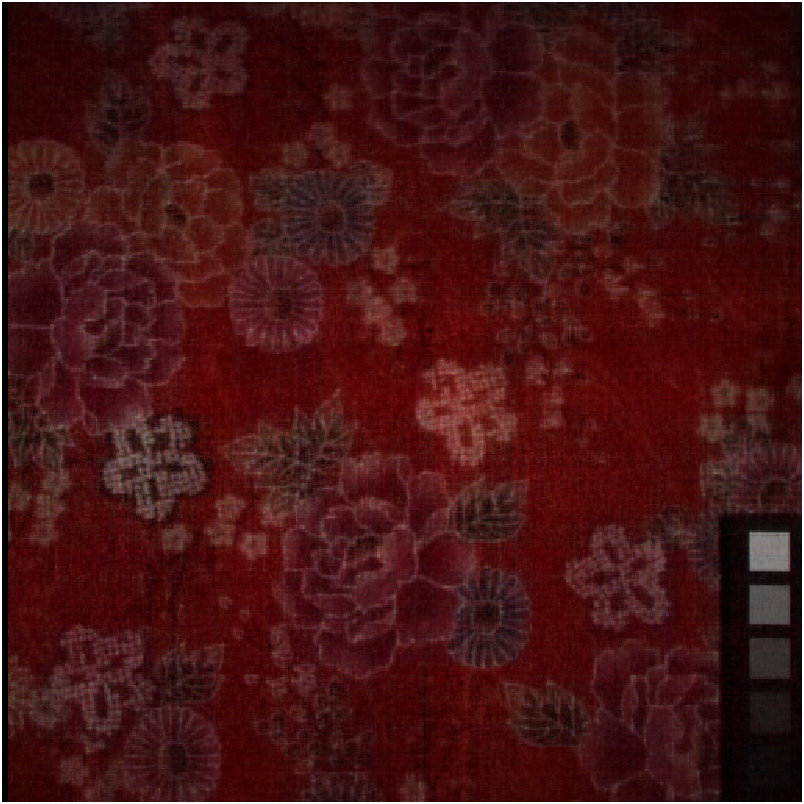}
                	\end{minipage}
                    \begin{minipage}{0.102\textwidth}
                		\centering
                		\includegraphics[width=1\textwidth, height=0.75\textwidth]{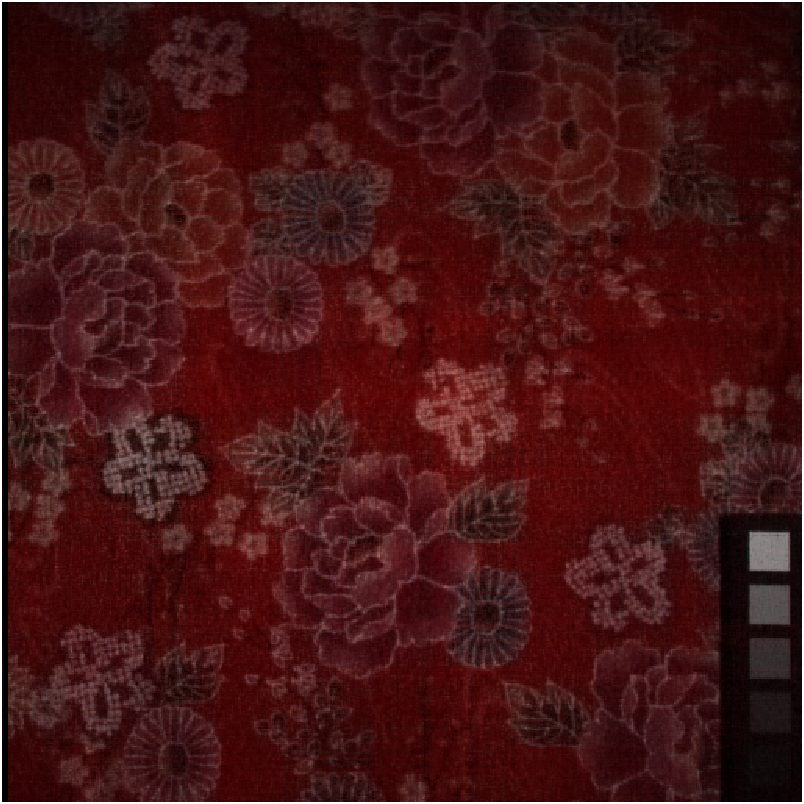}
                	\end{minipage}
                    \begin{minipage}{0.102\textwidth}
                		\centering
                		\includegraphics[width=1\textwidth, height=0.75\textwidth]{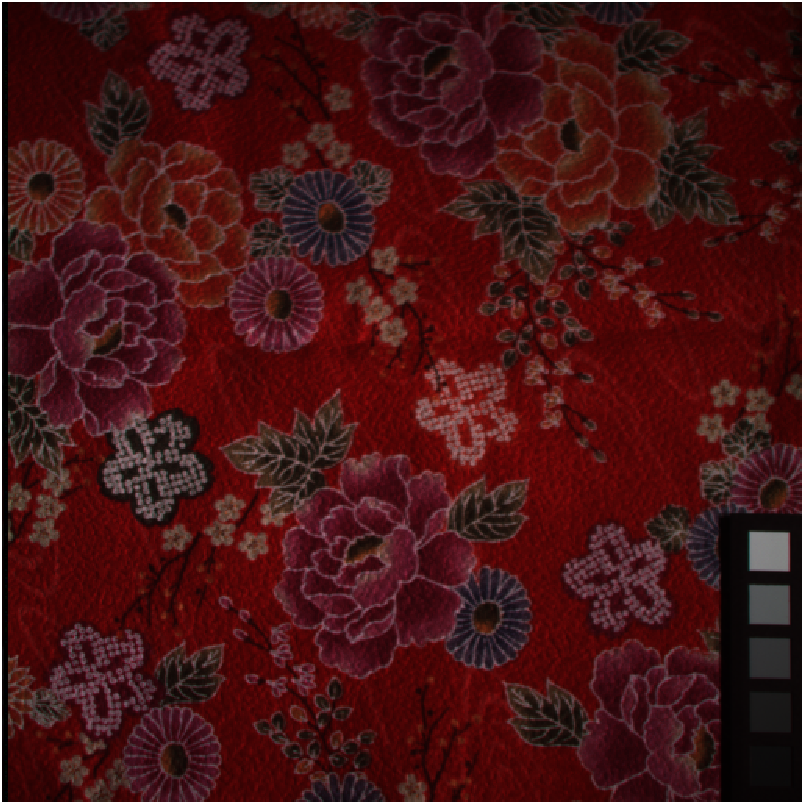}
                	\end{minipage}
                	\begin{minipage}{0.102\textwidth}
                		\centering
                		\includegraphics[width=1\textwidth, height=0.75\textwidth]{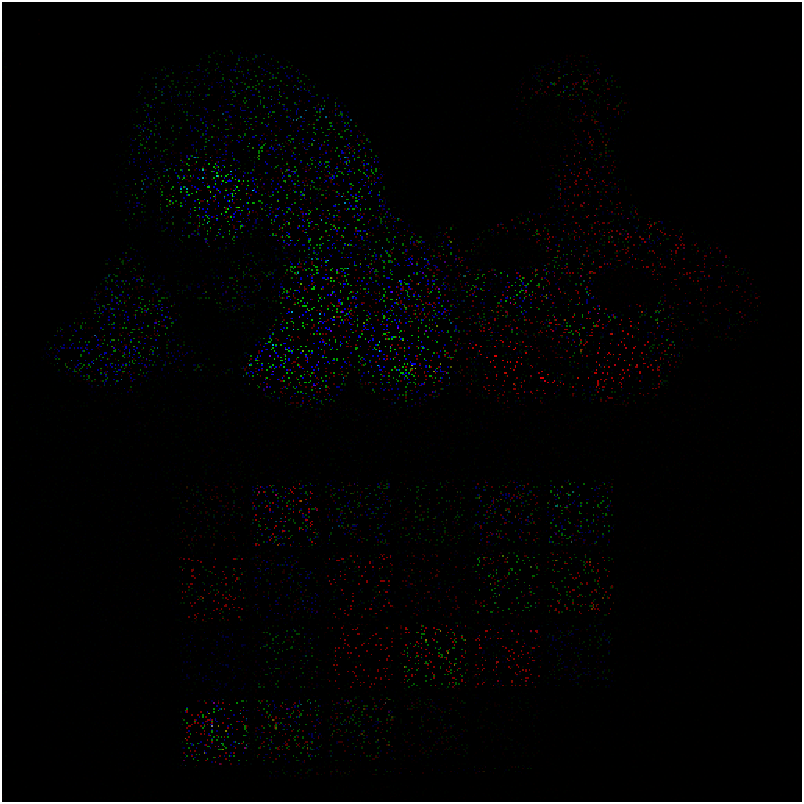}
                	\end{minipage}
                    \begin{minipage}{0.102\textwidth}
                		\centering
                		\includegraphics[width=1\textwidth, height=0.75\textwidth]{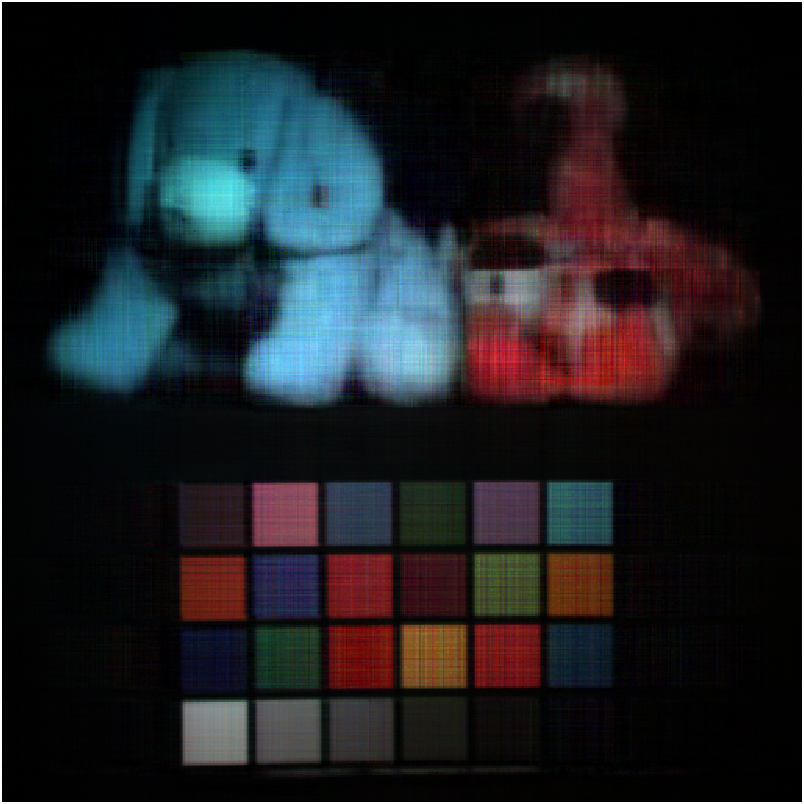}
                	\end{minipage}
                    \begin{minipage}{0.102\textwidth}
                		\centering
                		\includegraphics[width=1\textwidth, height=0.75\textwidth]{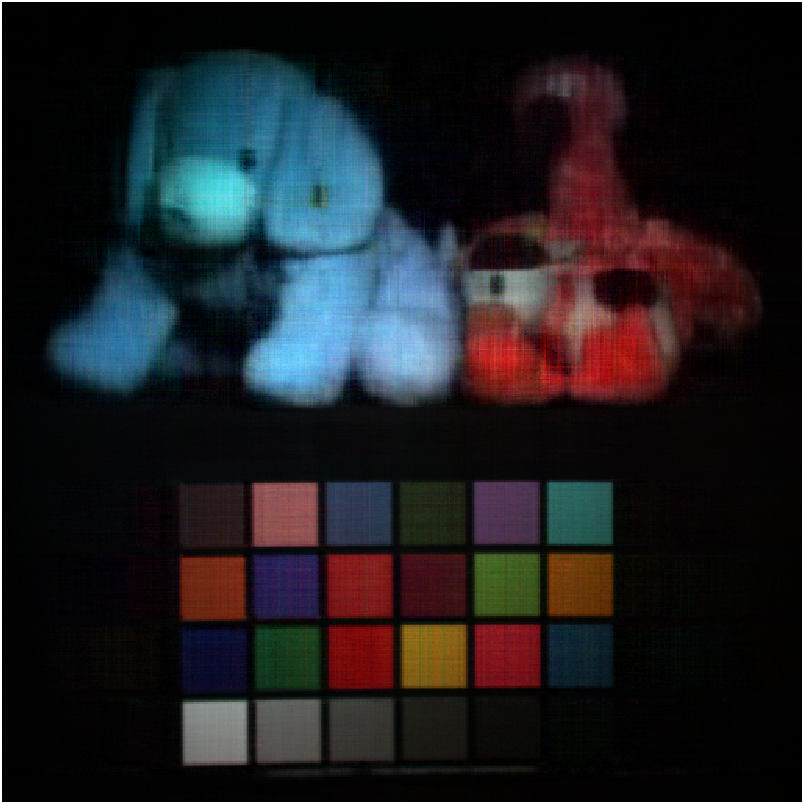}
                	\end{minipage}
                    \begin{minipage}{0.102\textwidth}
                		\centering
                		\includegraphics[width=1\textwidth, height=0.75\textwidth]{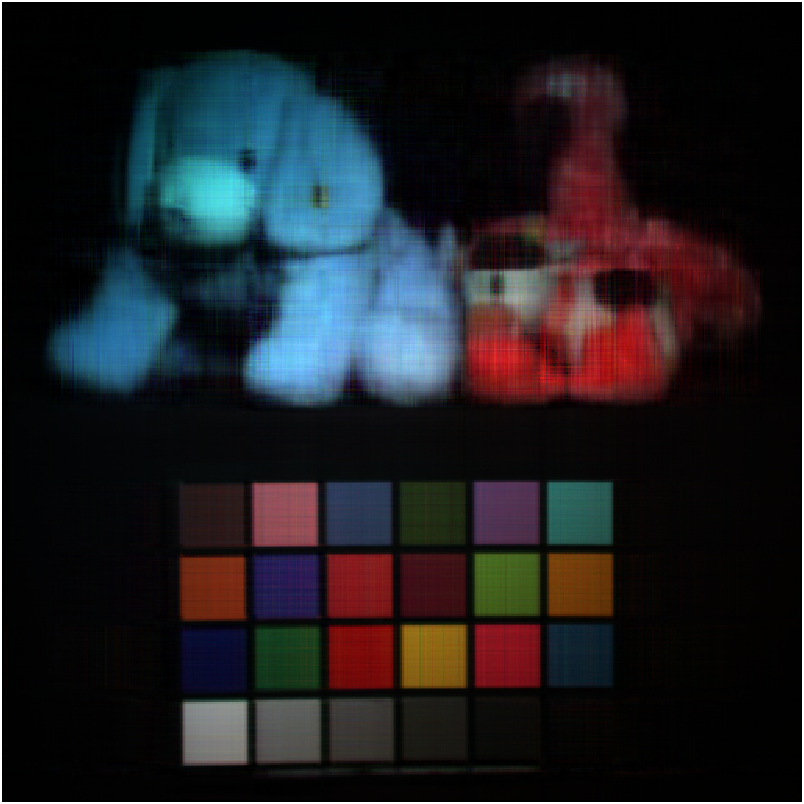}
                	\end{minipage}
                    \begin{minipage}{0.102\textwidth}
                		\centering
                		\includegraphics[width=1\textwidth, height=0.75\textwidth]{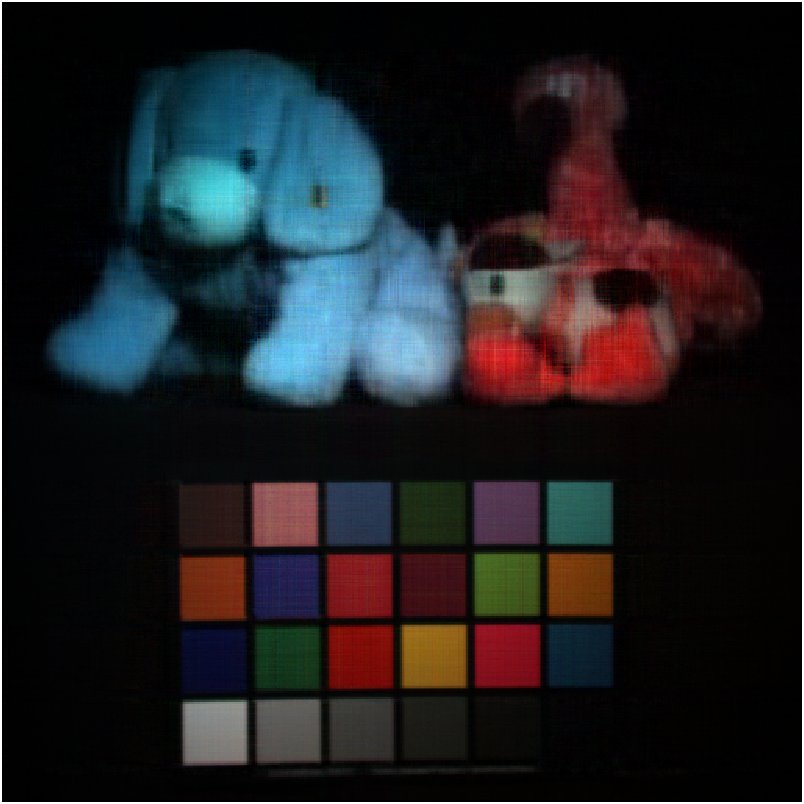}
                	\end{minipage}
                    \begin{minipage}{0.102\textwidth}
                		\centering
                		\includegraphics[width=1\textwidth, height=0.75\textwidth]{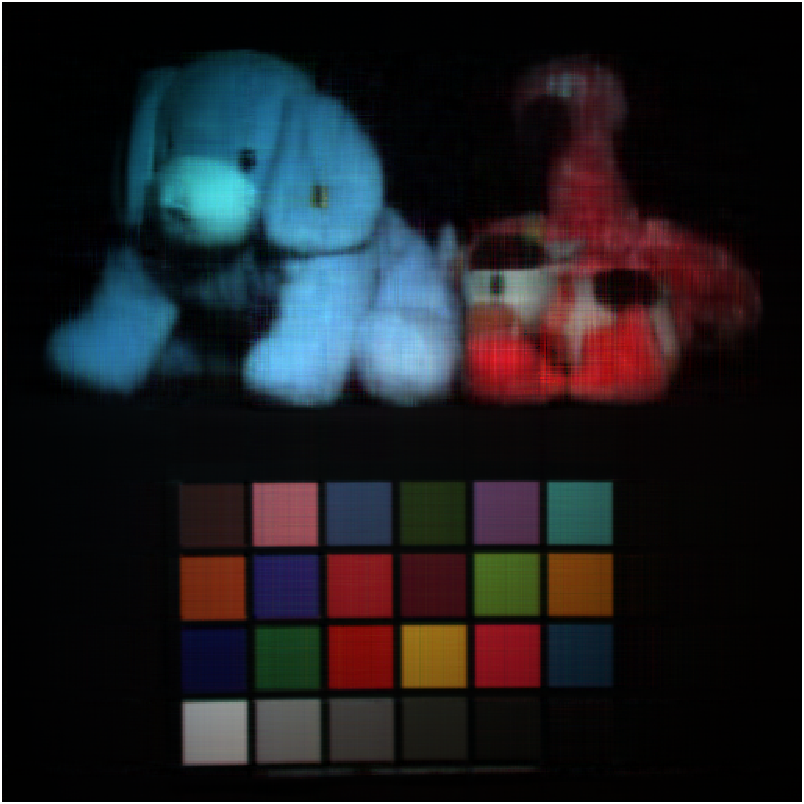}
                	\end{minipage}
                    \begin{minipage}{0.102\textwidth}
                		\centering
                		\includegraphics[width=1\textwidth, height=0.75\textwidth]{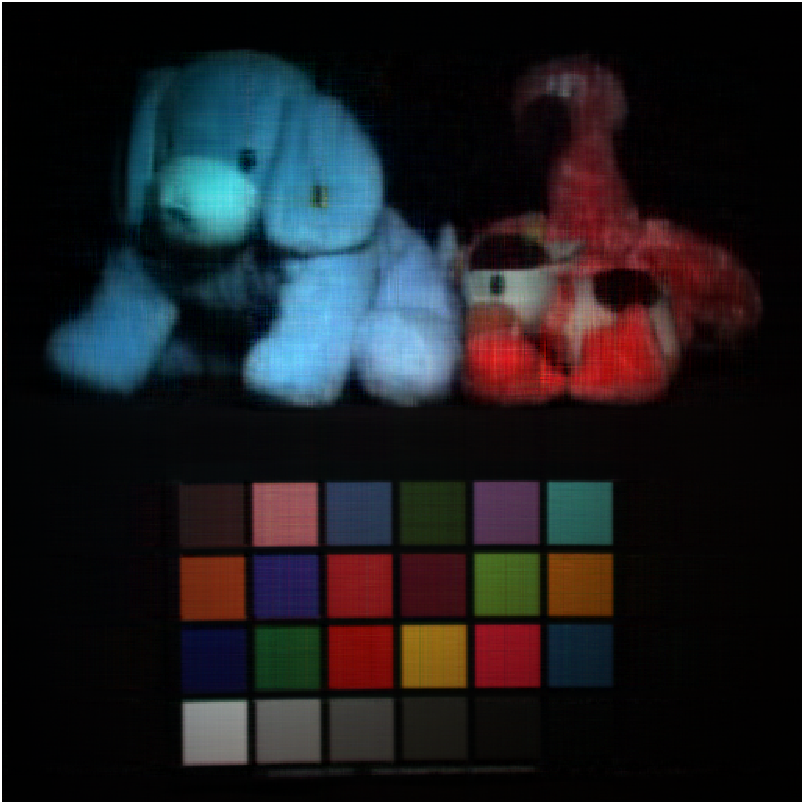}
                	\end{minipage}
                    \begin{minipage}{0.102\textwidth}
                		\centering
                		\includegraphics[width=1\textwidth, height=0.75\textwidth]{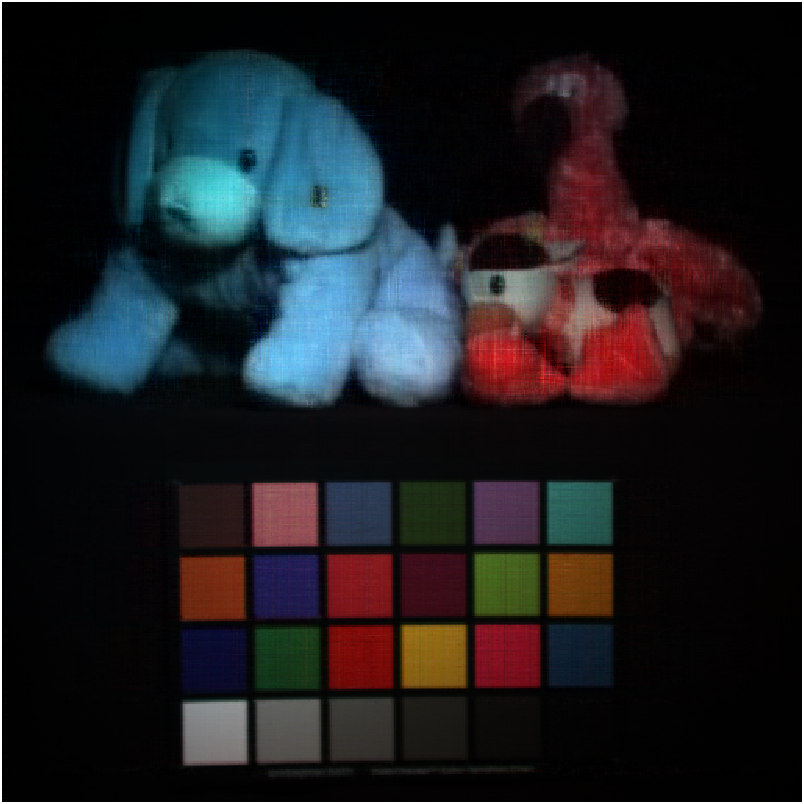}
                	\end{minipage}
                    \begin{minipage}{0.102\textwidth}
                		\centering
                		\includegraphics[width=1\textwidth, height=0.75\textwidth]{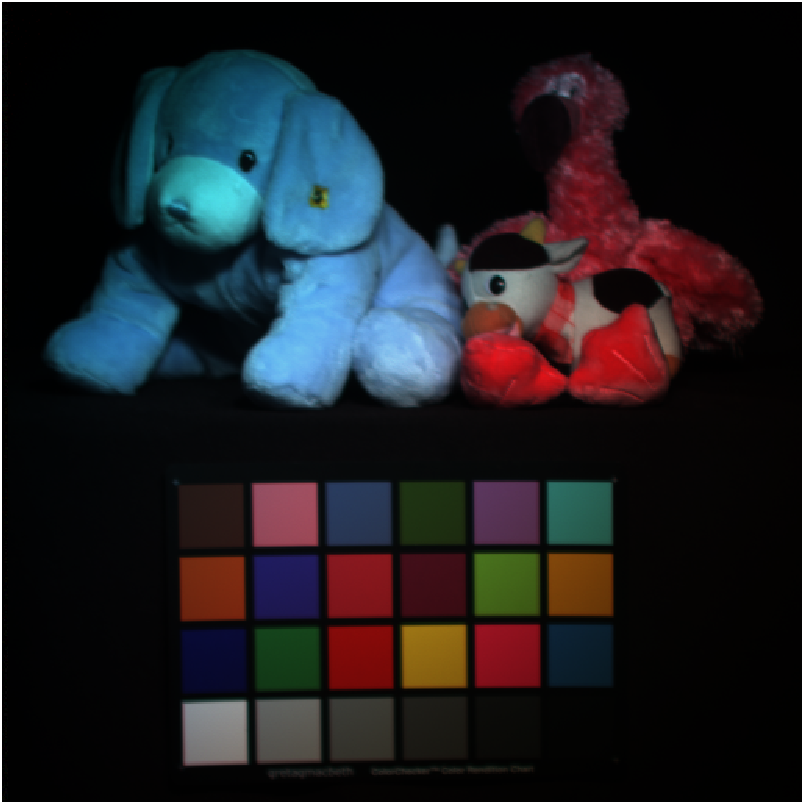}
                	\end{minipage}
                    \begin{minipage}{0.102\textwidth}
                		\centering
                		\includegraphics[width=1\textwidth, height=0.75\textwidth]{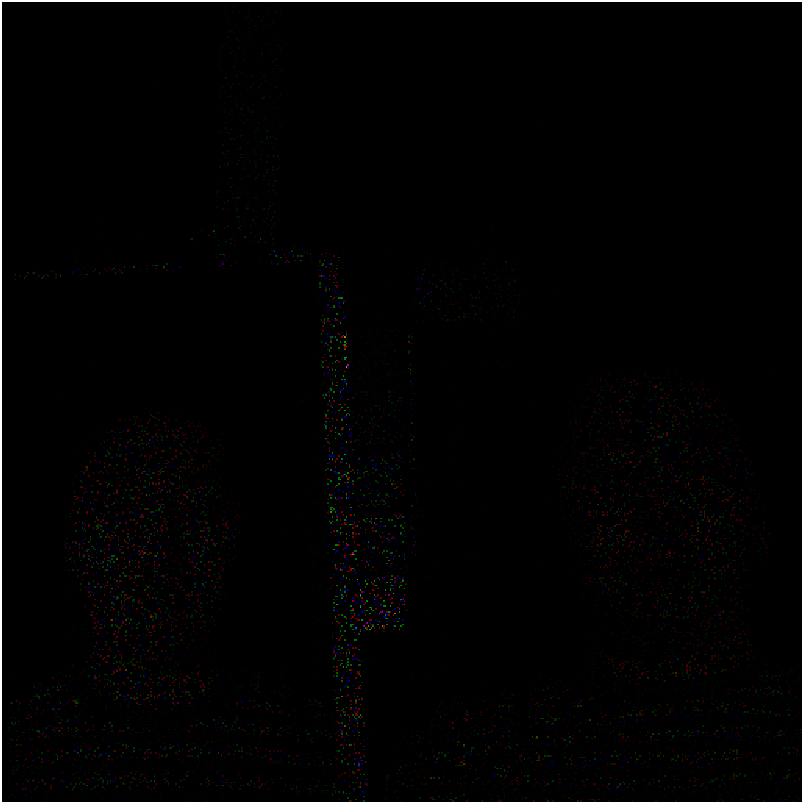}
                	\end{minipage}
                    \begin{minipage}{0.102\textwidth}
                		\centering
                		\includegraphics[width=1\textwidth, height=0.75\textwidth]{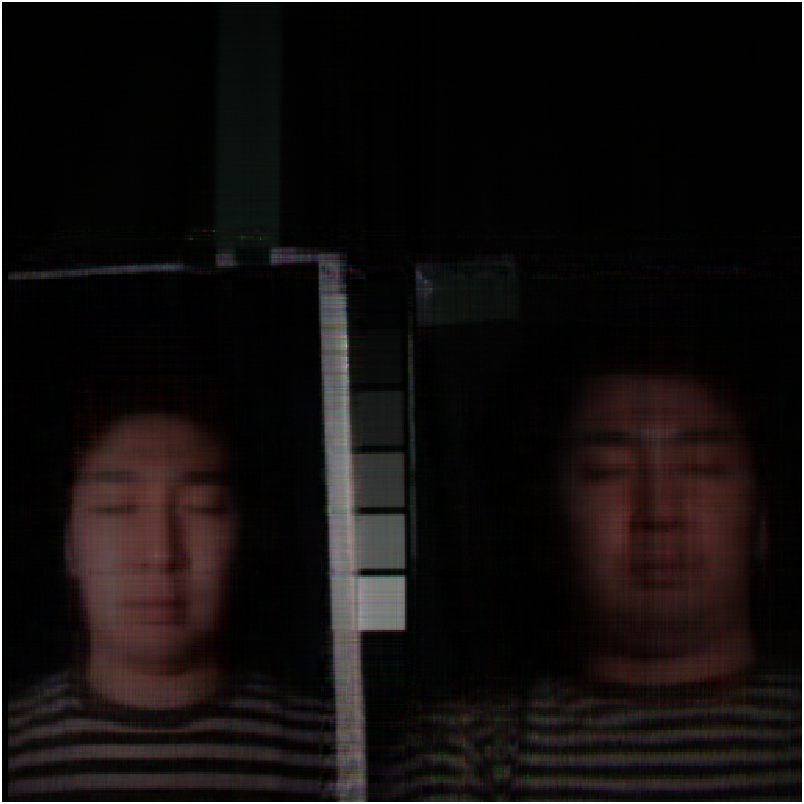}
                	\end{minipage}
                    \begin{minipage}{0.102\textwidth}
                		\centering
                		\includegraphics[width=1\textwidth, height=0.75\textwidth]{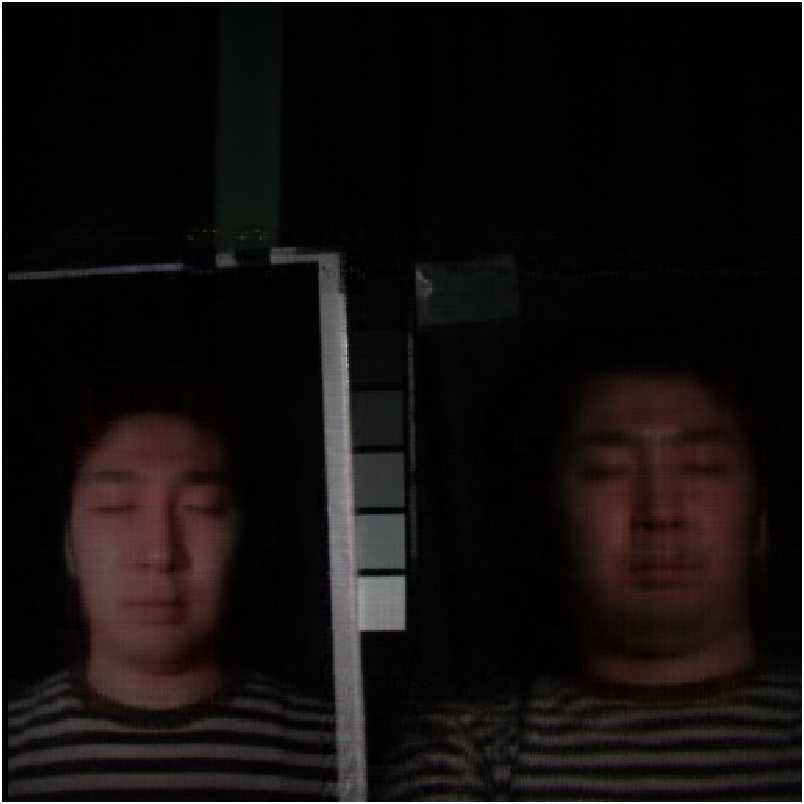}
                	\end{minipage}
                    \begin{minipage}{0.102\textwidth}
                		\centering
                		\includegraphics[width=1\textwidth, height=0.75\textwidth]{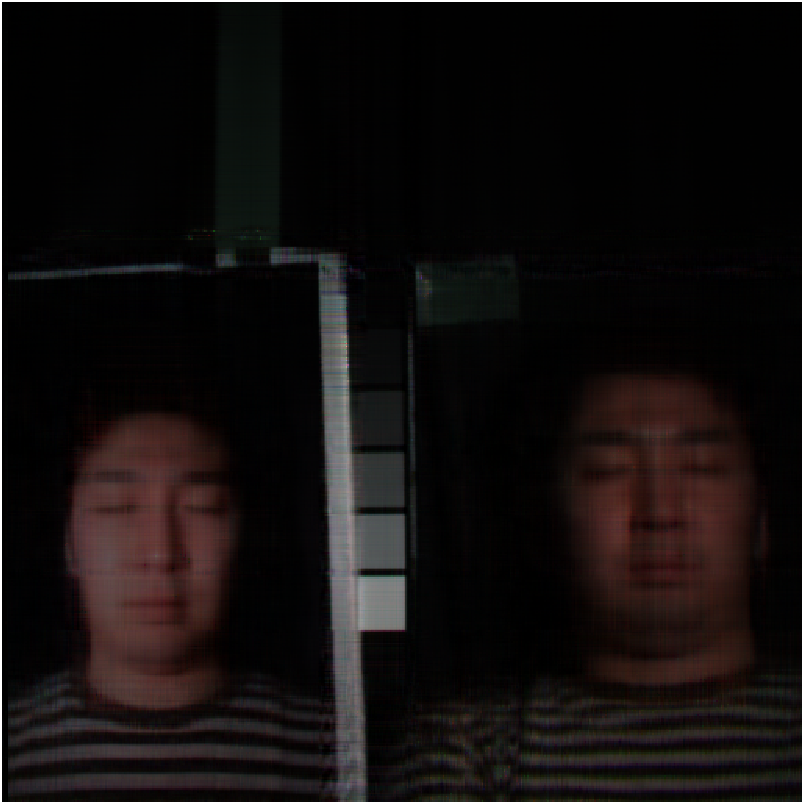}
                	\end{minipage}
                    \begin{minipage}{0.102\textwidth}
                		\centering
                		\includegraphics[width=1\textwidth, height=0.75\textwidth]{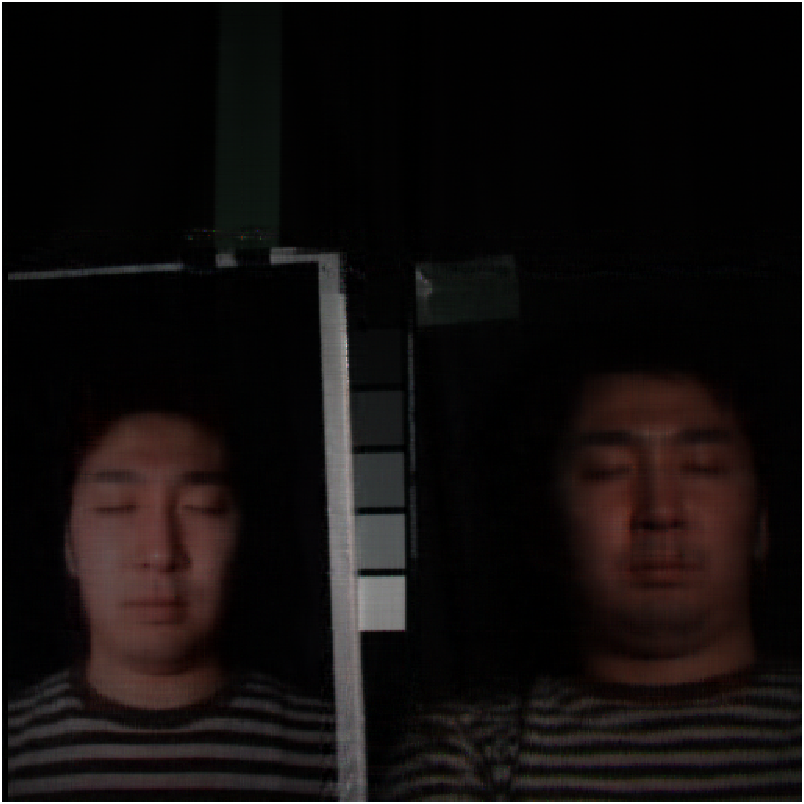}
                	\end{minipage}
                    \begin{minipage}{0.102\textwidth}
                		\centering
                		\includegraphics[width=1\textwidth, height=0.75\textwidth]{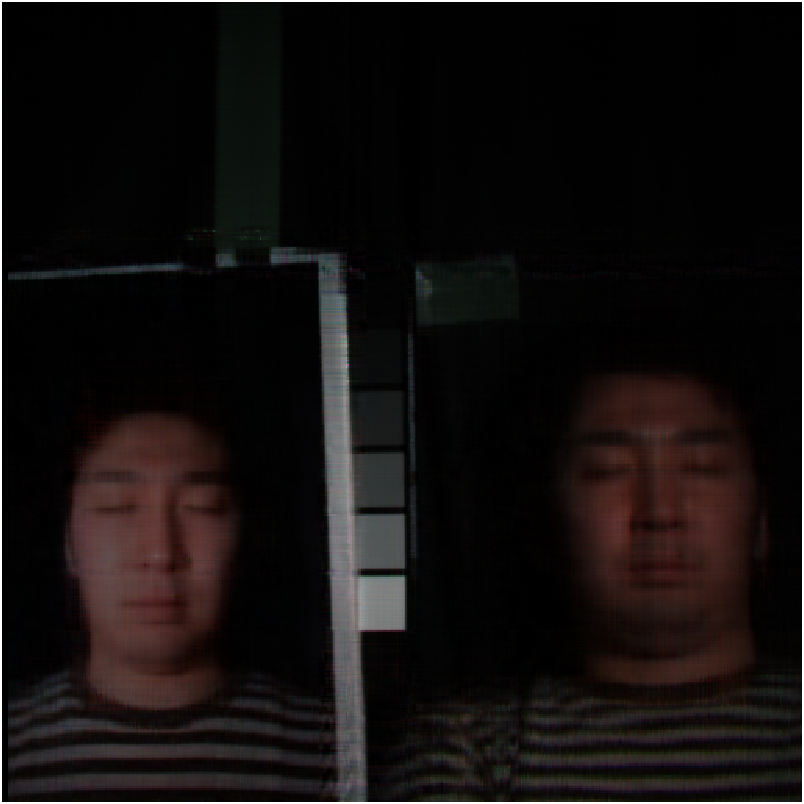}
                	\end{minipage}
                    \begin{minipage}{0.102\textwidth}
                		\centering
                		\includegraphics[width=1\textwidth, height=0.75\textwidth]{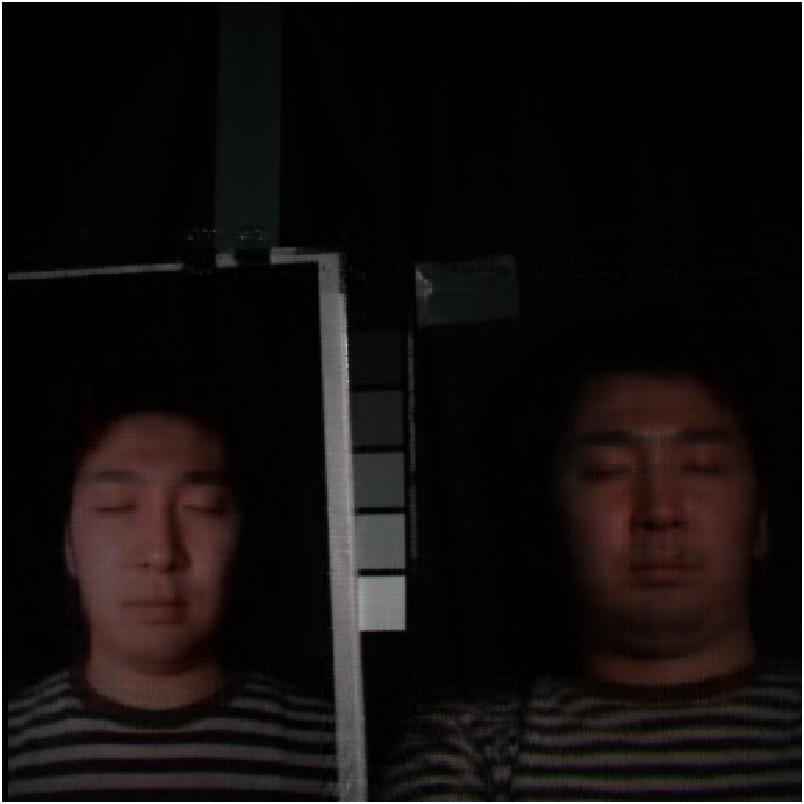}
                	\end{minipage}
                    \begin{minipage}{0.102\textwidth}
                		\centering
                		\includegraphics[width=1\textwidth, height=0.75\textwidth]{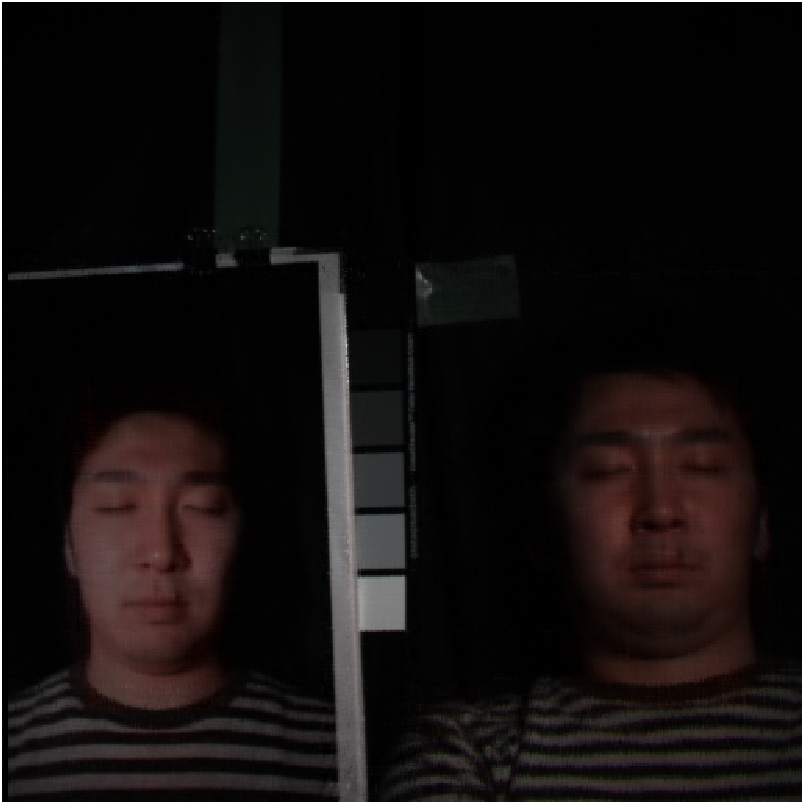}
                	\end{minipage}
                    \begin{minipage}{0.102\textwidth}
                		\centering
                		\includegraphics[width=1\textwidth, height=0.75\textwidth]{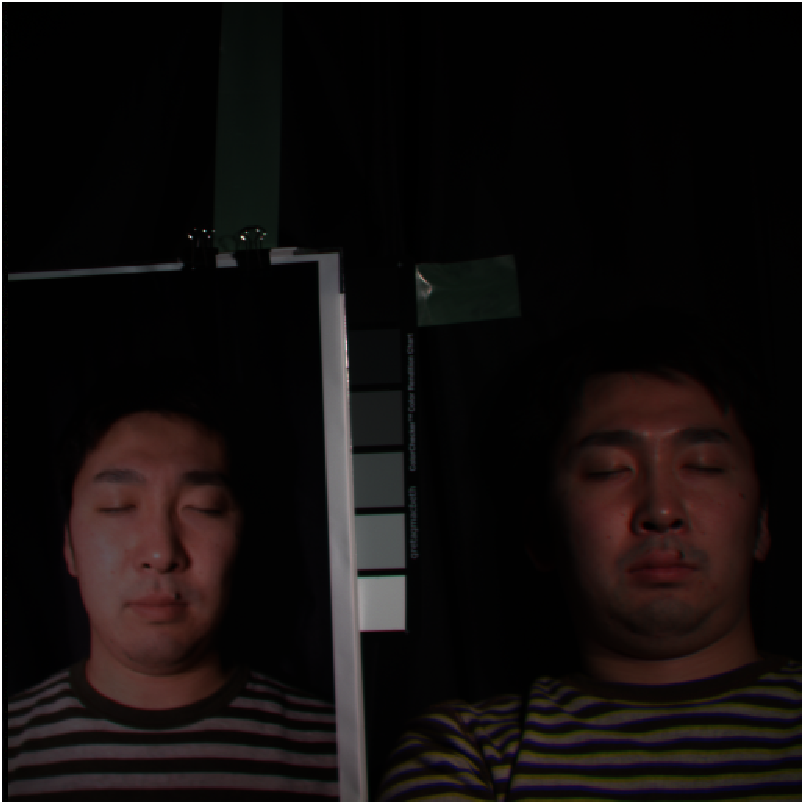}
                	\end{minipage}
                    \begin{minipage}{0.102\textwidth}
                		\centering
                		\includegraphics[width=1\textwidth, height=0.75\textwidth]{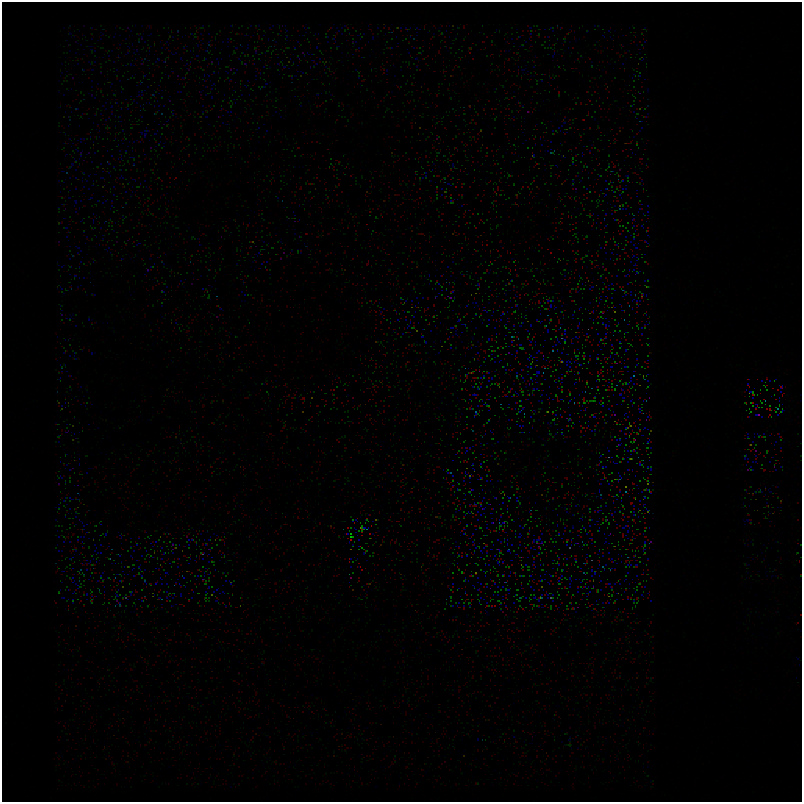}
                	\end{minipage}
                    \begin{minipage}{0.102\textwidth}
                		\centering
                		\includegraphics[width=1\textwidth, height=0.75\textwidth]{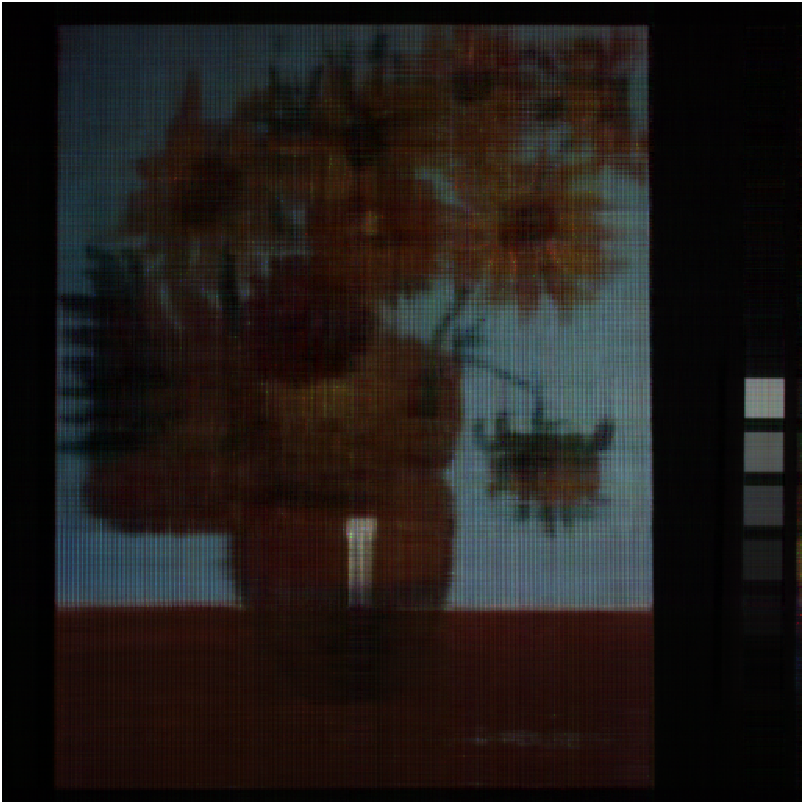}
                	\end{minipage}
                    \begin{minipage}{0.102\textwidth}
                		\centering
                		\includegraphics[width=1\textwidth, height=0.75\textwidth]{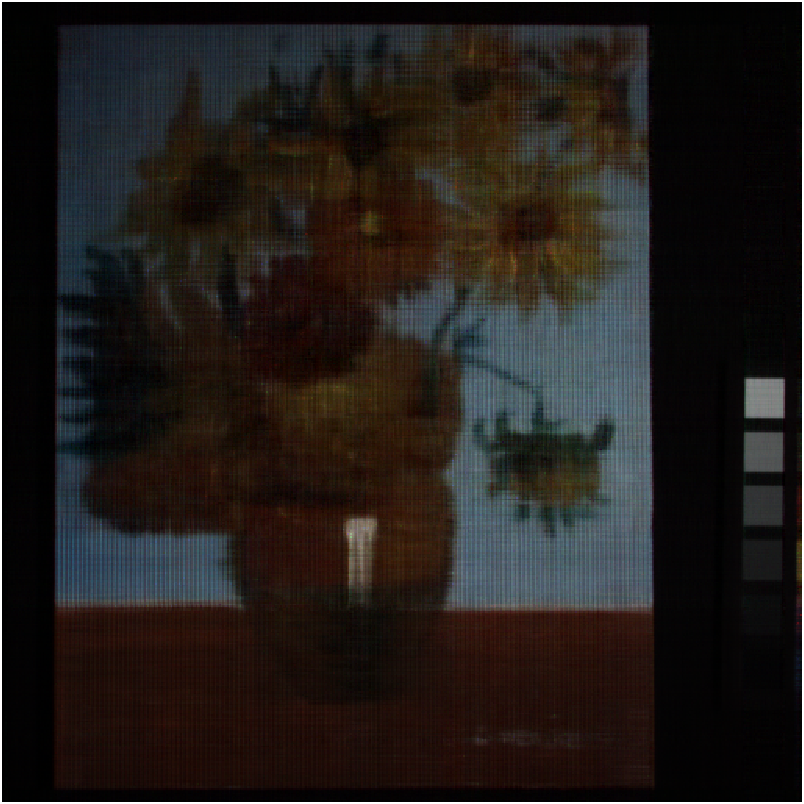}
                	\end{minipage}
                    \begin{minipage}{0.102\textwidth}
                		\centering
                		\includegraphics[width=1\textwidth, height=0.75\textwidth]{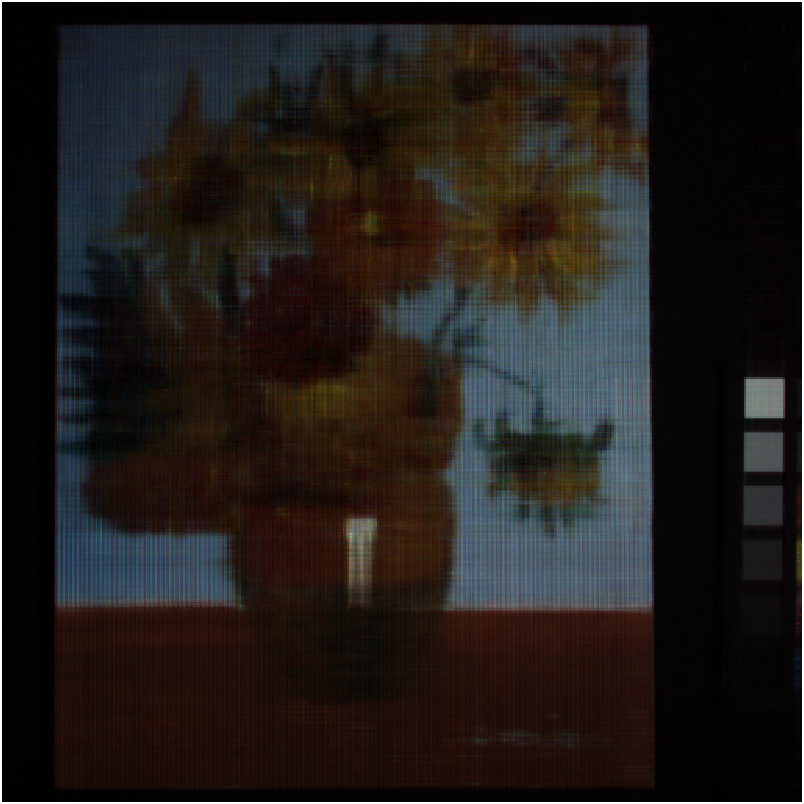}
                	\end{minipage}
                    \begin{minipage}{0.102\textwidth}
                		\centering
                		\includegraphics[width=1\textwidth, height=0.75\textwidth]{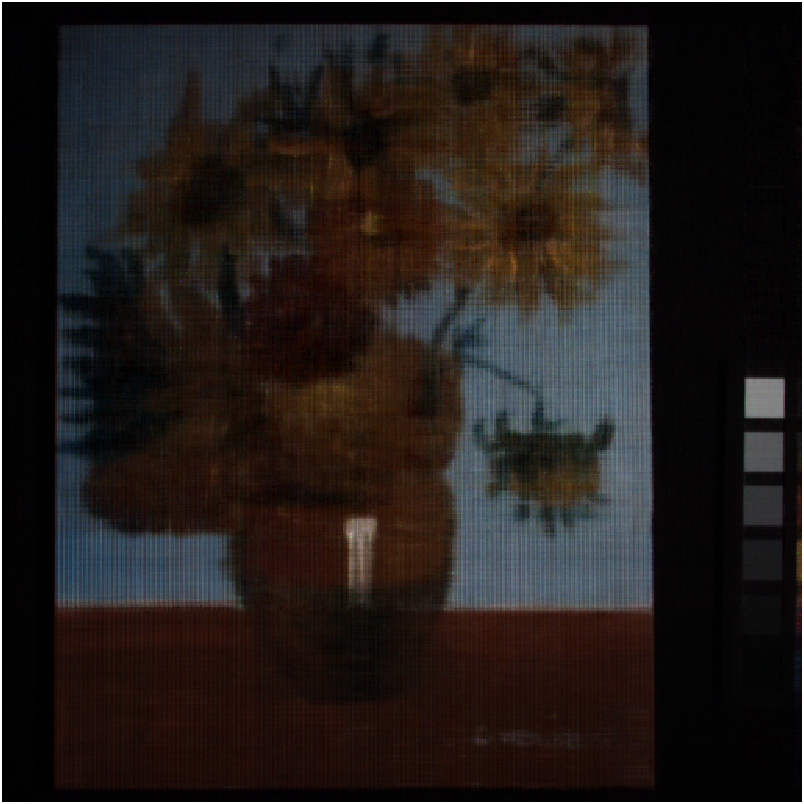}
                	\end{minipage}
                    \begin{minipage}{0.102\textwidth}
                		\centering
                		\includegraphics[width=1\textwidth, height=0.75\textwidth]{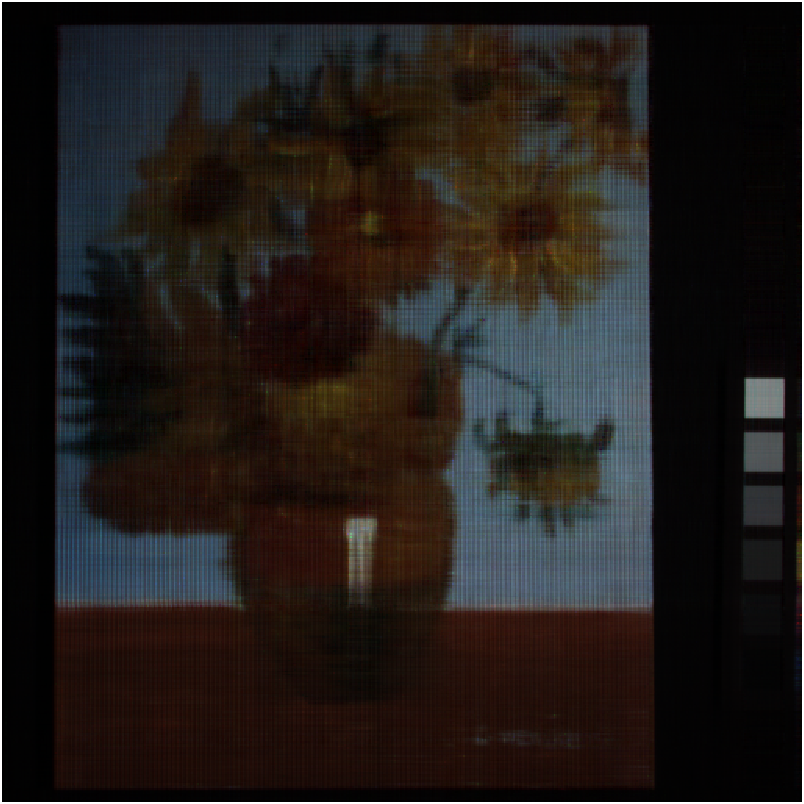}
                	\end{minipage}
                    \begin{minipage}{0.102\textwidth}
                		\centering
                		\includegraphics[width=1\textwidth, height=0.75\textwidth]{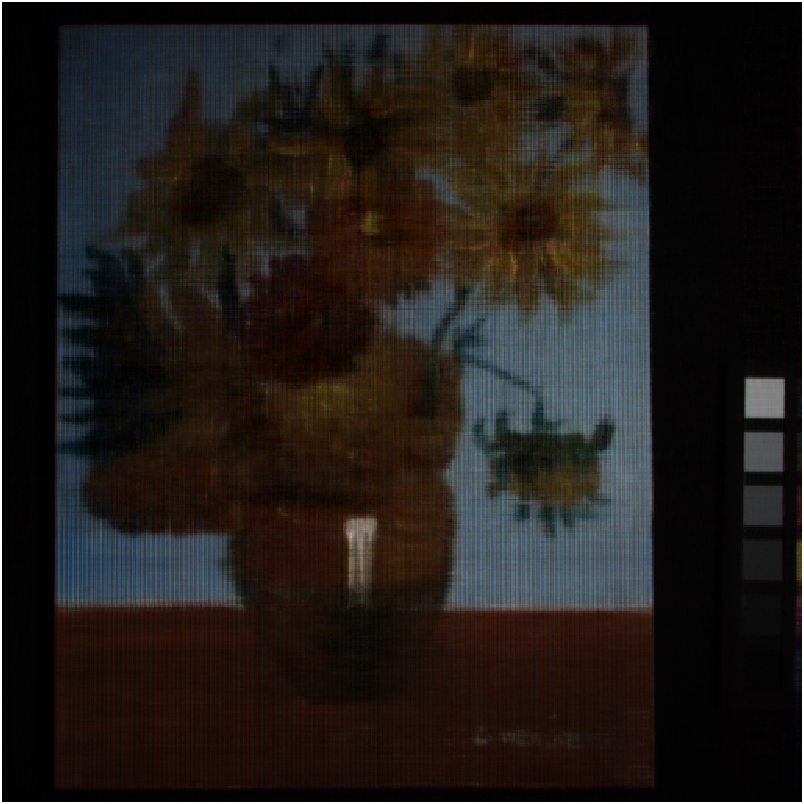}
                	\end{minipage}
                    \begin{minipage}{0.102\textwidth}
                		\centering
                		\includegraphics[width=1\textwidth, height=0.75\textwidth]{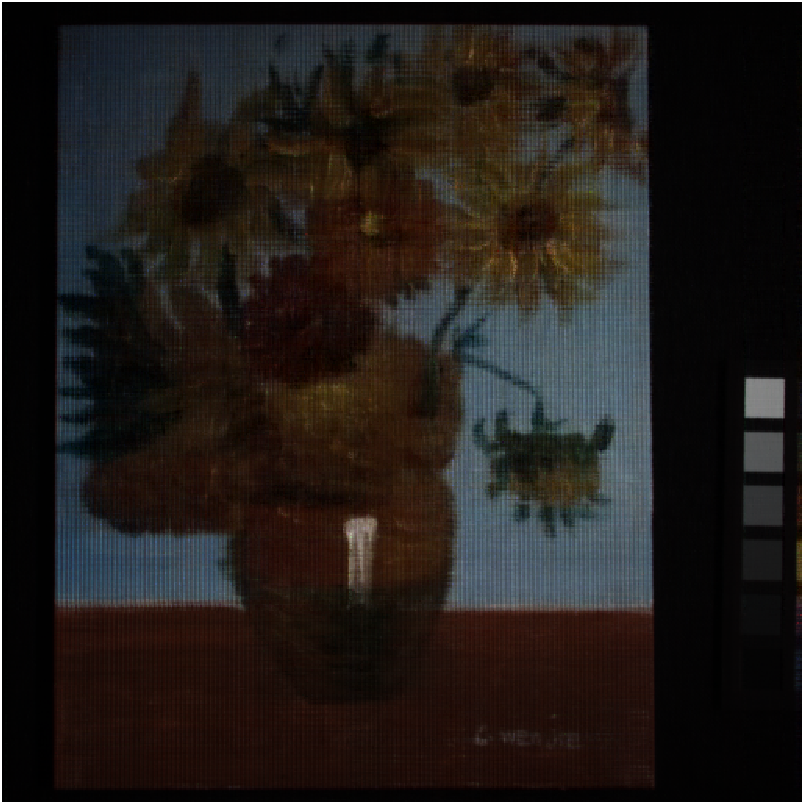}
                	\end{minipage}
                    \begin{minipage}{0.102\textwidth}
                		\centering
                		\includegraphics[width=1\textwidth, height=0.75\textwidth]{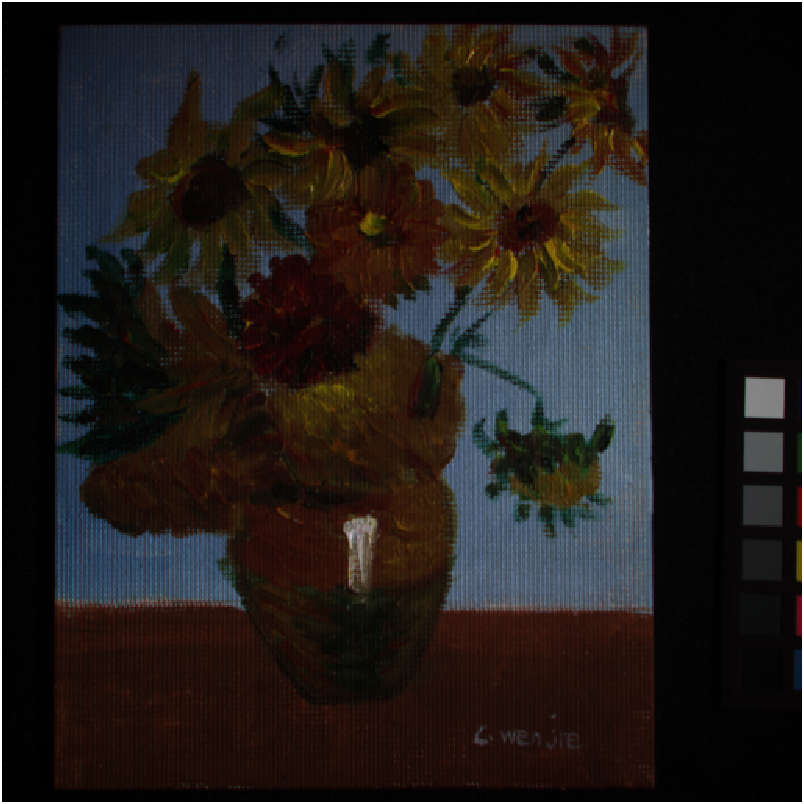}
                	\end{minipage}
                    \begin{minipage}{0.102\textwidth}
                		\centering
                		\includegraphics[width=1\textwidth, height=0.75\textwidth]{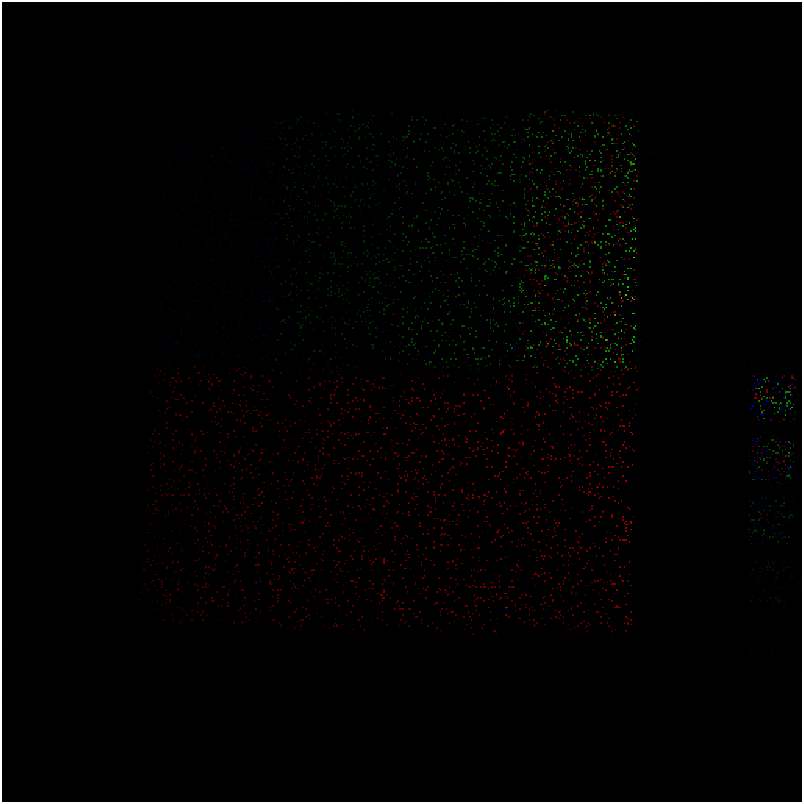}
                        \subcaption{\scriptsize \scriptsize Observed}
                	\end{minipage}
                    \begin{minipage}{0.102\textwidth}
                		\centering
                		\includegraphics[width=1\textwidth, height=0.75\textwidth]{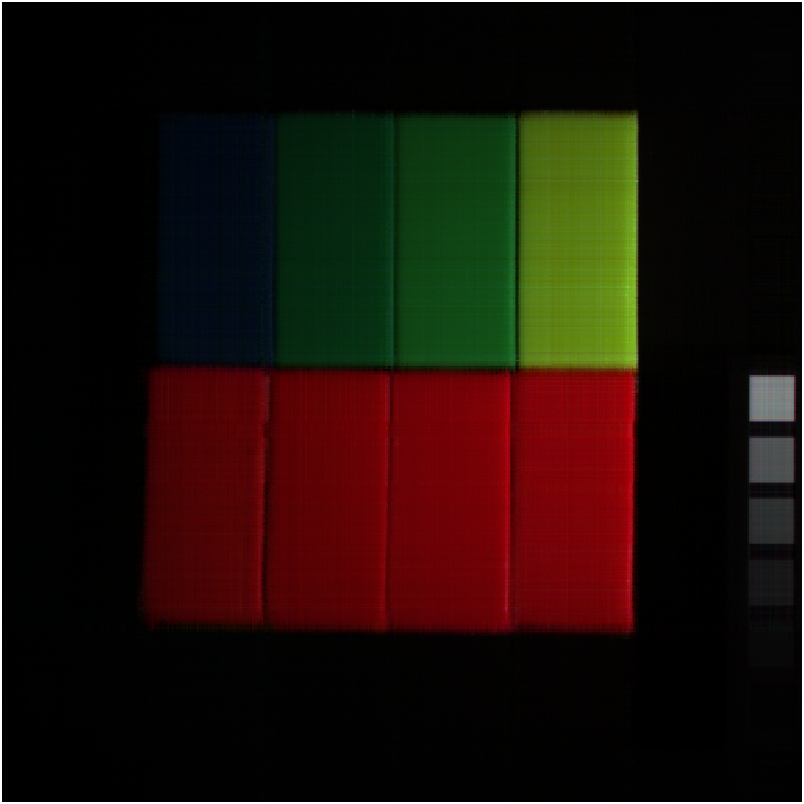}
                        \subcaption{\scriptsize \scriptsize DWHT}
                	\end{minipage}
                    \begin{minipage}{0.102\textwidth}
                		\centering
                		\includegraphics[width=1\textwidth, height=0.75\textwidth]{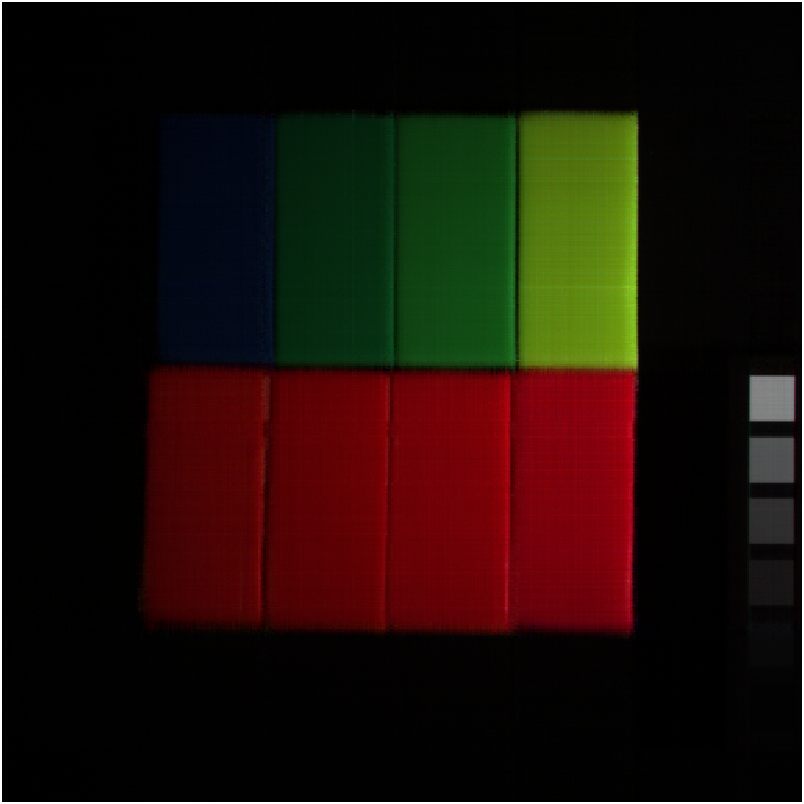}
                        \subcaption{\scriptsize \scriptsize DFT}
                	\end{minipage}
                    \begin{minipage}{0.102\textwidth}
                		\centering
                		\includegraphics[width=1\textwidth, height=0.75\textwidth]{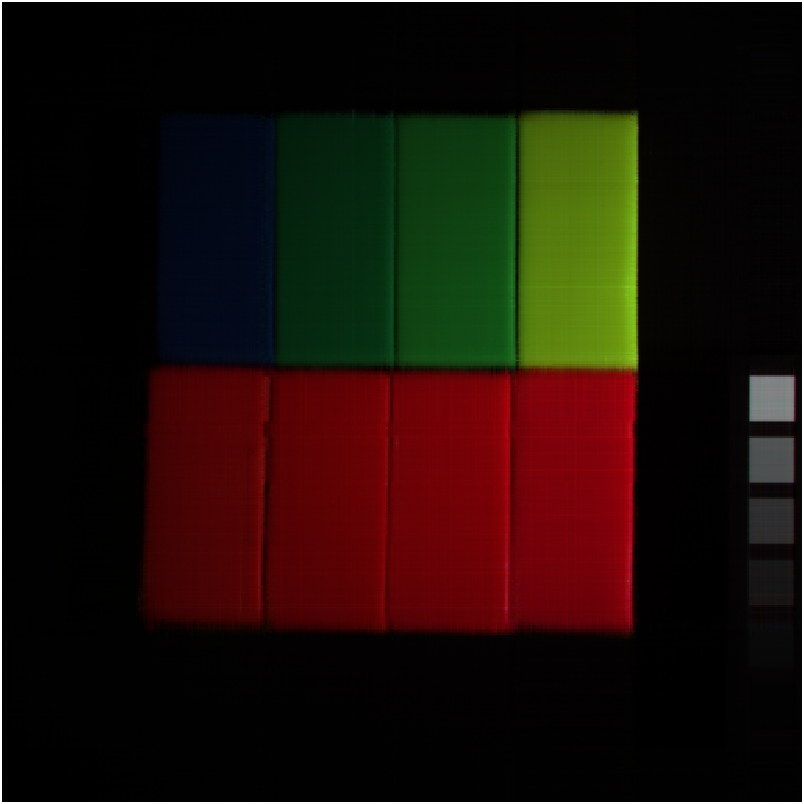}
                        \subcaption{\scriptsize \scriptsize DFT2}
                	\end{minipage}
                    \begin{minipage}{0.102\textwidth}
                		\centering
                		\includegraphics[width=1\textwidth, height=0.75\textwidth]{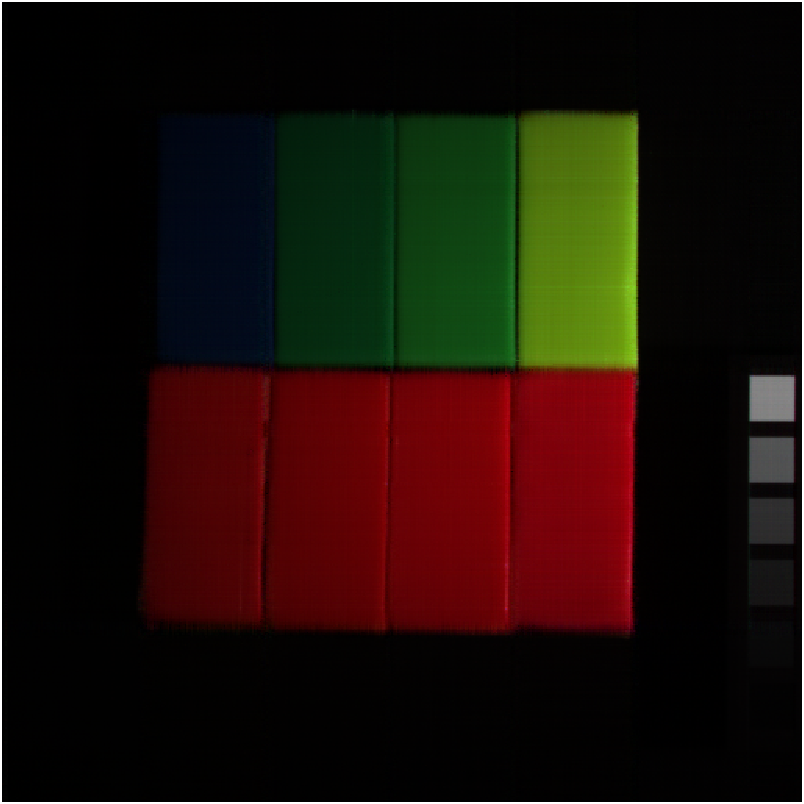}
                        \subcaption{\scriptsize \scriptsize DCT}
                	\end{minipage}
                    \begin{minipage}{0.102\textwidth}
                		\centering
                		\includegraphics[width=1\textwidth, height=0.75\textwidth]{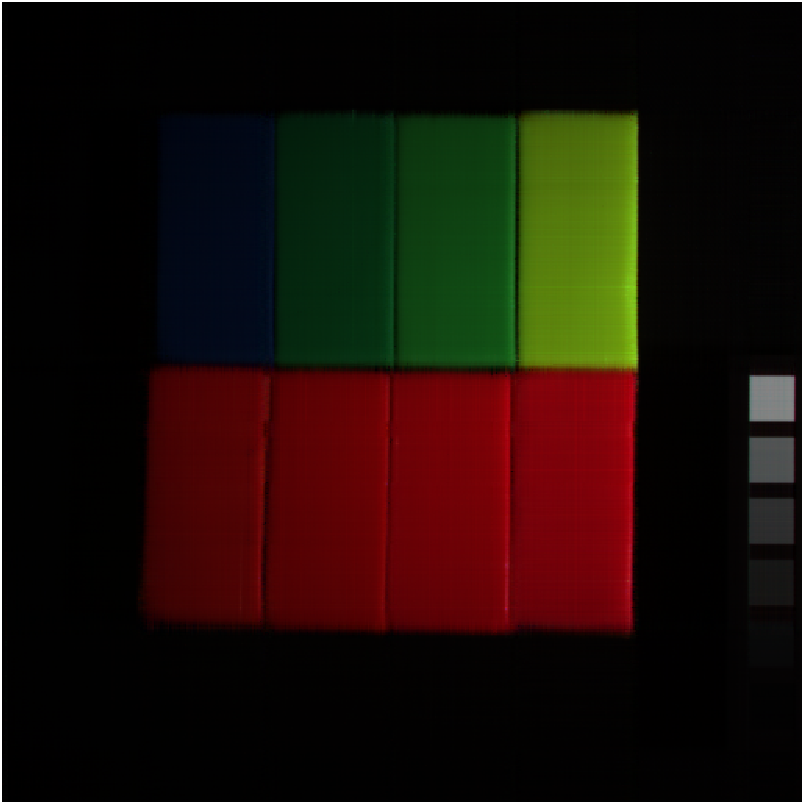}
                        \subcaption{\scriptsize \scriptsize DCT2}
                	\end{minipage}
                    \begin{minipage}{0.102\textwidth}
                		\centering
                		\includegraphics[width=1\textwidth, height=0.75\textwidth]{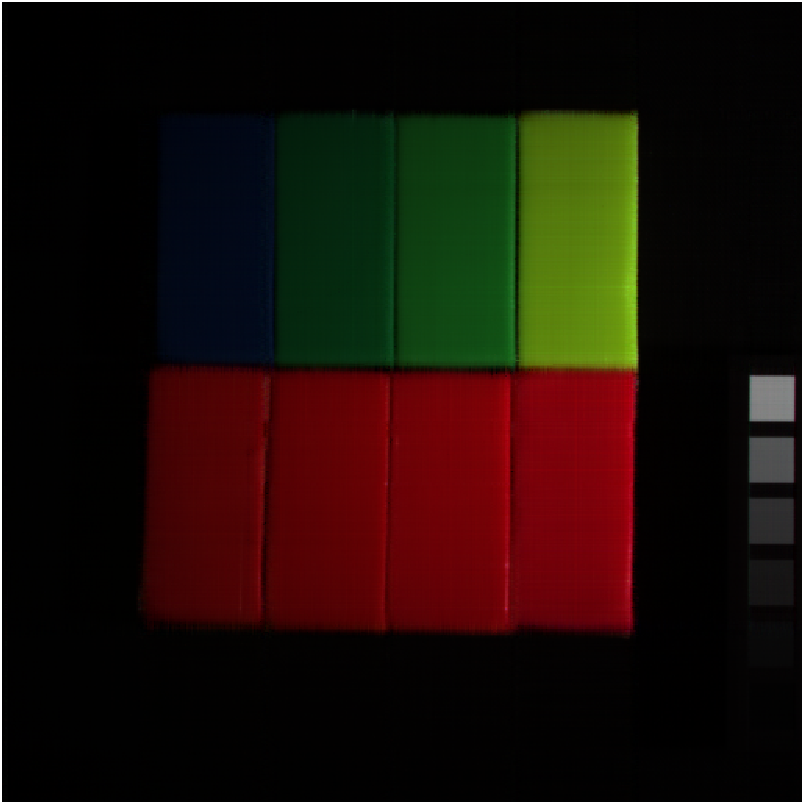}
                        \subcaption{\scriptsize \tiny  DFT-DCT}
                	\end{minipage}
                    \begin{minipage}{0.102\textwidth}
                		\centering
                		\includegraphics[width=1\textwidth, height=0.75\textwidth]{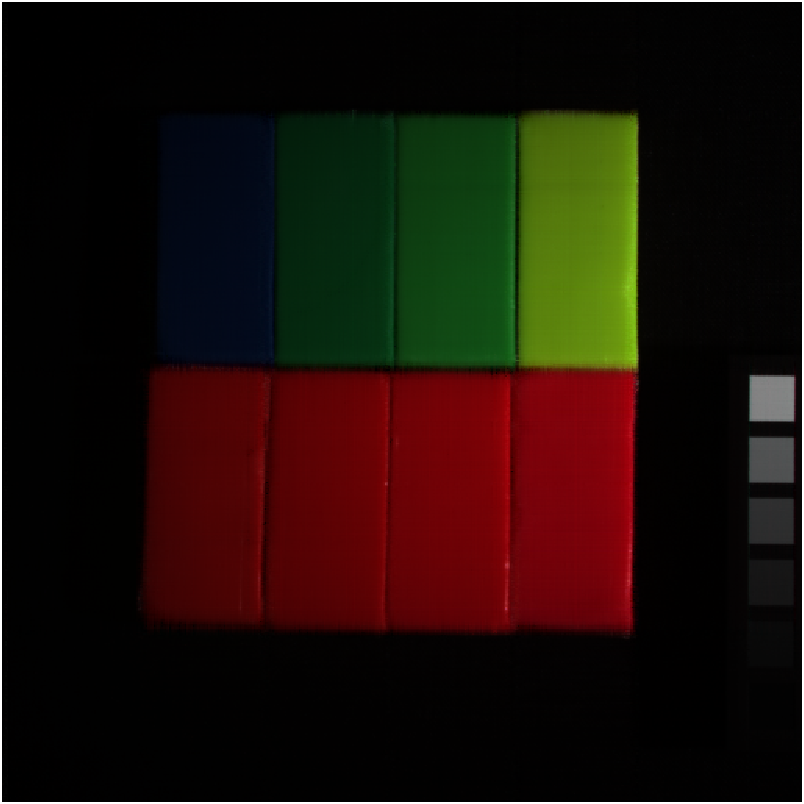}
                        \subcaption{\scriptsize \scriptsize Framelet}
                	\end{minipage}
                    \begin{minipage}{0.102\textwidth}
                		\centering
                		\includegraphics[width=1\textwidth, height=0.75\textwidth]{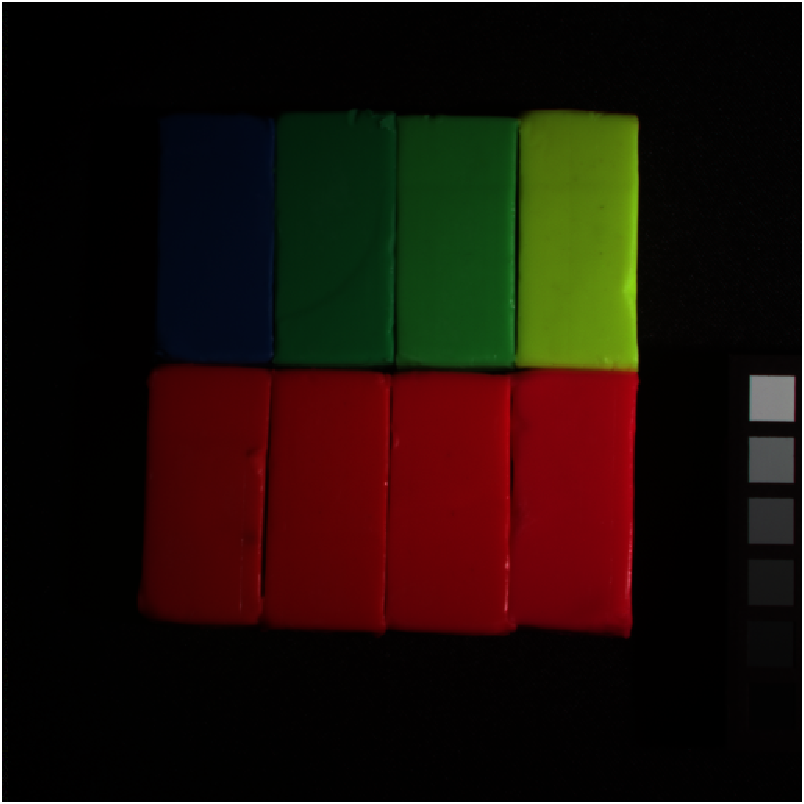}
                        \subcaption{\scriptsize \scriptsize Original}
                	\end{minipage}
                    \caption{\scriptsize Pseudo-RGB (R-28, G-15, B-9) image for performance comparison on MSI ``super balls'', ``cloth'', ``stuffed toys'', ``photo and face'', ``oil painting'' and ``clay'' (top to bottom) with SR $= 0.05$.}
                    \label{fig:MSI_image}
                    \vskip -0.12 in
                \end{figure}

                The visual results are shown by pseudo-RGB image in Figure \ref{fig:MSI_image}. It can be observed that the texture and contour are clearer for DCT2, DFT-DCT and Framelet recovered MSIs. Also, the color fidelity is maintained even under a small sampling rate of $0.05$. \vspace{-0.2cm}

                \subsubsection{Video Data}

                In this experiment, the application of video data recovery is considered. The momentum of applying tensor-based modeling is that the frames of video data are correlated and can be viewed as the third dimension of a tensor. The test videos include ``bus'', ``carphone'', ``foreman'', ``miss'', ``coastguard'', ``news'', ``soccer'', ``suzie'' and ``trevor''. The contents of these videos contain cars, boats, rivers, and humans of different resolutions. Also, both static and dynamic scenes occur in these test videos. The assessment criteria are the mean PSNR (mPSNR) and the mean SSIM (mSSIM) of all frames.
                
                \begin{figure}[ht]
                \vskip -0.1 in
                \centering
                \begin{minipage}{0.102\textwidth}
                    \centering
                    \includegraphics[width=1\textwidth]{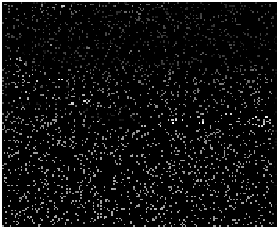}
                \end{minipage}
                \begin{minipage}{0.102\textwidth}
                    \centering
                    \includegraphics[width=1\textwidth]{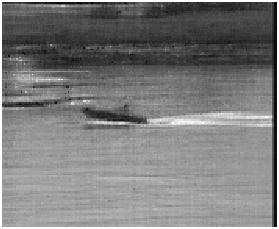}
                \end{minipage}
                \begin{minipage}{0.102\textwidth}
                    \centering
                    \includegraphics[width=1\textwidth]{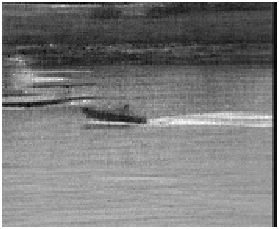}
                \end{minipage}
                \begin{minipage}{0.102\textwidth}
                    \centering
                    \includegraphics[width=1\textwidth]{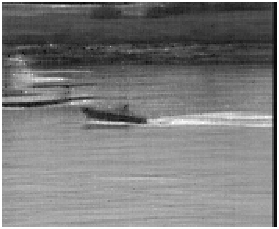}
                \end{minipage}
                \begin{minipage}{0.102\textwidth}
                    \centering
                    \includegraphics[width=1\textwidth]{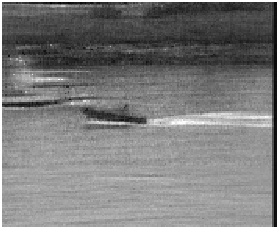}
                \end{minipage}
                \begin{minipage}{0.102\textwidth}
                    \centering
                    \includegraphics[width=1\textwidth]{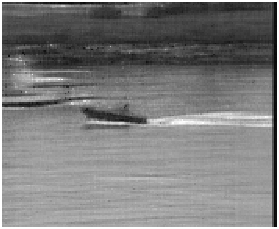}
                \end{minipage}
                \begin{minipage}{0.102\textwidth}
                    \centering
                    \includegraphics[width=1\textwidth]{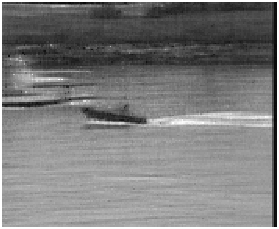}
                \end{minipage}
                \begin{minipage}{0.102\textwidth}
                    \centering
                    \includegraphics[width=1\textwidth]{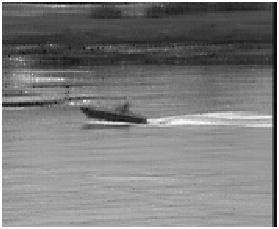}
                \end{minipage}
                \begin{minipage}{0.102\textwidth}
                    \centering
                    \includegraphics[width=1\textwidth]{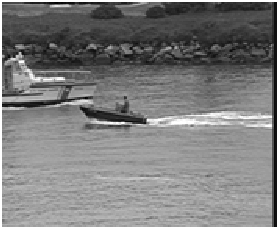}
                \end{minipage}
                \begin{minipage}{0.102\textwidth}
                    \centering
                    \includegraphics[width=1\textwidth]{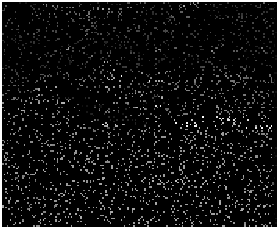}
                        \subcaption{\scriptsize \scriptsize Observed}
                \end{minipage}
                \begin{minipage}{0.102\textwidth}
                    \centering
                    \includegraphics[width=1\textwidth]{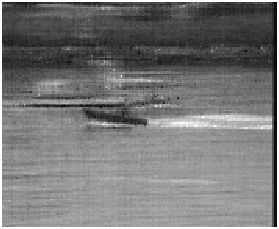}
                        \subcaption{\scriptsize \scriptsize DWHT}
                \end{minipage}
                \begin{minipage}{0.102\textwidth}
                    \centering
                    \includegraphics[width=1\textwidth]{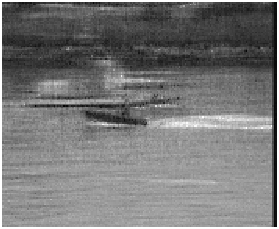}
                        \subcaption{\scriptsize \scriptsize DFT}
                \end{minipage}
                \begin{minipage}{0.102\textwidth}
                    \centering
                    \includegraphics[width=1\textwidth]{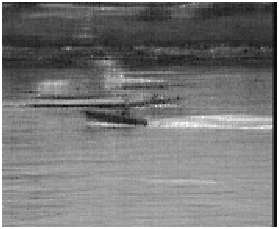}
                        \subcaption{\scriptsize \scriptsize DFT2}
                \end{minipage}
                \begin{minipage}{0.102\textwidth}
                    \centering
                    \includegraphics[width=1\textwidth]{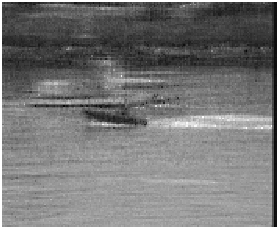}
                        \subcaption{\scriptsize \scriptsize DCT}
                \end{minipage}
                \begin{minipage}{0.102\textwidth}
                    \centering
                    \includegraphics[width=1\textwidth]{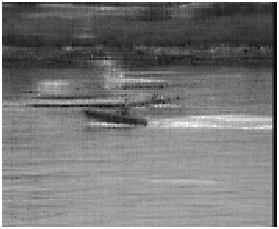}
                        \subcaption{\scriptsize \scriptsize DCT2}
                \end{minipage}
                \begin{minipage}{0.102\textwidth}
                    \centering
                    \includegraphics[width=1\textwidth]{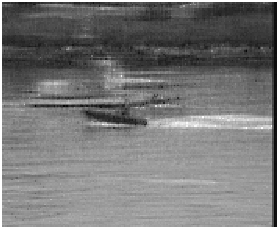}
                        \subcaption{\scriptsize \tiny DFT-DCT}
                \end{minipage}
                \begin{minipage}{0.102\textwidth}
                    \centering
                    \includegraphics[width=1\textwidth]{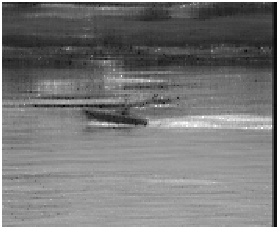}
                        \subcaption{\scriptsize \scriptsize Framelet}
                \end{minipage}
                \begin{minipage}{0.102\textwidth}
                    \centering
                    \includegraphics[width=1\textwidth]{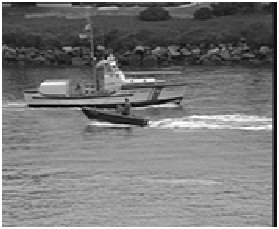}
                    \subcaption{\scriptsize \scriptsize Original}
                \end{minipage}
                \caption{\scriptsize Performance comparison on video ``coastguard'' with SR$=0.1$. The first row shows the $30$th frame and the second row shows the $60$th frame.}
                \label{fig:yuv_p=0.1}
                \vskip -0.12 in
             \end{figure}

            The average mPSNR and mSSIM values are given by the bottom half of Table \ref{tb:MSI_Video} and specifically, the performance indicators of each test video for SR $=0.05$ are shown in Figure \ref{fig:yuv_ind_p=0.05} in Appendix \ref{section:exp_vid}. As can be observed, the non-invertible slim transforms exhibit superior performance among the videos containing different contents and scenes, with a more obvious advantage in mSSIM. The visual recovery results of the $30$th frame and $60$th frame of the video ``coastguard'' with SR$=0.1$ are exhibited in Figure \ref{fig:yuv_p=0.1}. The frames recovered by DCT2, DFT-DCT, Framelet contain fewer noise artifacts, and the margins of the river and the coast are more obvious compared to other methods.\vspace{-0.2cm} 
    
    \section{Conclusion}
        In this paper, the theoretical guarantee of exact tensor completion with arbitrary linear transforms is established by directly operating the tensors in the transform domain. Breaking the constraints of constraints of isotropy and self-adjointness, our model and proof greatly enhance the flexibility of tensor completion. Also, a new analysis obtained by the proof discloses why slim transforms outperform their square counterparts from a theoretical level. Extensive experiments on random data and visual data validate the superiority of the proposed method.

\bibliographystyle{unsrt}
\bibliography{bibfile}

\appendix
\vskip 0.2 in
\textbf{\Large Appendices}

The appendices are organized as follows: Appendix \ref{section:proof_of_main_theorem} contains the proof of the main theorem. Appendix \ref{section:proofs_of_lemmas} proves three crucial lemmas and the energy terms are bounded in Appendix \ref{section:bound_F_norm}. Additional experiment results on video data are in Appendix \ref{section:exp_vid}. The complete analysis of the sampling rate is given in Appendix \ref{section:ananlyzing_sampling_rate}. The algorithm for solving Program \ref{eq:program:min_transformed_TNN} is provided in Appendix \ref{section:solving_main_program}. Methods of generating random tensors in Section \ref{section:exp_random_tensor} are given in Appendix \ref{section:generate_M}. Limitations are in Appendix \ref{section:limitations}.

\section{Proof of Theorem \ref{theorem:exact_completion}}
    \label{section:proof_of_main_theorem}
    The proof of Theorem \ref{theorem:exact_completion} consists of the proofs of two propositions: the optimality condition (Proposition \ref{proposition:optimality_condition}) and construct dual certificate (Proposition \ref{proposition:constructing_dual_certificate}). The latter further requires three norm bounding lemmas and bounding the energy terms, which are contained in Appendix \ref{section:proofs_of_lemmas} and Appendix \ref{section:bound_F_norm}.
    \subsection{Optimality Condition}
        \begin{proposition}
            Tensor $\mathcal{X}$ is the unique optimal solution to \eqref{eq:program:min_transformed_TNN} if the following conditions hold: 
            \begin{enumerate}
                \item 
                \begin{equation}
                    \Vert \mathcal{P}_{\mathbb{S}} \mathit{T} \mathcal{R}_{\Omega} \mathit{T}^{\dag} \mathcal{P}_{\mathbb{S}} - \mathcal{P}_{\mathbb{S}} \mathit{T} \mathit{T}^{\dag} \mathcal{P}_{\mathbb{S}} \Vert \leq \frac{1}{2}
                    \label{eq:cond:approximate_isometry}
                \end{equation}
                where $\mathcal{R}_{\Omega} = \frac{1}{p} \mathcal{S}_{\Omega}$, $\mathbb{S}$ denotes the subspace spanned by singular tensors of $\mathit{T}(\mathcal{X}) = \mathcal{U} \star \mathcal{S} \star \mathcal{V}^{\text{H}}$:
                \begin{equation}\nonumber
                    \begin{aligned}
                        & \mathbb{S} = \{\mathcal{U} \star {\Gamma_1}^{\text{H}} + \Gamma_2 \star \mathcal{V}^{\text{H}}~\vert~\forall~\Gamma_1 \in \mathbb{C}^{n_1 \times r \times N_3}, \Gamma_2 \in \mathbb{C}^{n_2 \times r \times N_3}\} \\
                        & \mathcal{P}_{\mathbb{S}}(\mathcal{Z}) = \mathcal{U} \star \mathcal{U}^{\text{H}} \star \mathcal{Z} + \mathcal{Z} \star \mathcal{V} \star \mathcal{V}^{\text{H}} - \mathcal{U} \star \mathcal{U}^{\text{H}} \star \mathcal{Z} \star \mathcal{V} \star \mathcal{V}^{\text{H}} \\ 
                        & \mathcal{P}_{\mathbb{S}^{\perp}}(\mathcal{Z}) = \mathcal{Z} - \mathcal{P}_{\mathbb{S}}(\mathcal{Z}) = (\mathcal{I} - \mathcal{U} \star \mathcal{U}^{\text{H}}) \star \mathcal{Z} \star (\mathcal{I} - \mathcal{V} \star \mathcal{V}^{\text{H}})
                    \end{aligned}
                \end{equation}
                    
                \item
                $\exists~\mathcal{Y} \in \mathbb{C}^{n_1 \times n_2 \times N_3}, \text{s.t.}$
                \begin{enumerate}[label={(\alph*)}]
                    \item                
                    \begin{equation}
                        (\mathit{T} \mathit{T}^{\dag} - \mathit{T} \mathcal{S}_{\Omega} \mathit{T}^{\dag})^{\text{H}} \mathcal{Y} = \mathbf{0}
                        \label{eq:cond:inner_product_condition}
                    \end{equation}
                    
                    \item
                    \begin{equation}
                        \Vert \mathcal{P}_{\mathbb{S}}(\mathcal{Y}) - \mathcal{U} \star \mathcal{V} \Vert_{\text{F}} \leq \frac{p}{8\kappa(\mathbf{T})}
                        \label{eq:cond:low_distortion}
                    \end{equation}
    
                    \item 
                    \begin{equation}
                        \Vert \mathcal{P}_{\mathbb{S}^{\perp}}(\mathcal{Y}) \Vert \leq \frac{1}{2}
                        \label{eq:cond:spectral_norm_cond}
                    \end{equation}
                \end{enumerate}
            \end{enumerate}
            \label{proposition:optimality_condition}
        \end{proposition}

        \begin{proof}
            The proof of this proposition is inspired by \cite{proof_6}, which aims to recover 1-D data sparse in the transform domain. By noticing the analogy among 1-D sparse signal, 2-D low rank matrices and 3-D low rank tensors, the analysis in \cite{proof_6} can serve as guidance. 
            
            First let us revisit the definition of subgradient in convex optimization and give the subgradient of the natural tensor nuclear norm:
            \begin{recall}[Subgradient]
                \qquad\qquad\qquad\qquad\qquad\qquad
                \label{recall:subgradient}
                \begin{enumerate}
                    \item By convex optimization theory, $\mathcal{A}$ is in the subgradient set of $\Vert \cdot \Vert_*$ defined upon $\mathcal{X}_0$ iff $\forall \mathcal{X}$, 
                    \begin{equation}
                        \Vert \mathcal{X} \Vert_* \geq \Vert \mathcal{X}_0 \Vert_* + \langle \mathcal{A}, \mathcal{X} - \mathcal{X}_0 \rangle. \nonumber
                        \label{eq:applying_definition_of_subgradient}
                    \end{equation}
    
                    \item the subgradient formulation of $\Vert \cdot \Vert_*$ is given as
                    \begin{equation}
                        \mathcal{A} \in \{ \mathcal{U} \star \mathcal{V}^{\text{H}} + \mathcal{W}~\vert~\mathcal{U}^{\text{H}} \star \mathcal{W} = \mathbf{0}, \mathcal{W} \star \mathcal{V} = \mathbf{0}, \Vert \mathcal{W} \Vert \leq 1\}, \nonumber
                        \label{eq:subgradient_of_TNN}
                    \end{equation}
                    where $\mathcal{U}, \mathcal{V}$ are the singular tensors of $\mathit{T}(\mathcal{X})$.
                \end{enumerate}
            \end{recall}

            Suppose $\mathcal{H}$ is a perturbation satisfying $\mathcal{S}_{\Omega}(\mathcal{H}) = \mathbf{0}$, according to Recall \ref{recall:subgradient} and the fact that $\Vert \mathbf{U}_{\perp} \star \mathbf{V}_{\perp}^{\text{H}} \Vert = 1$, we have
            \begin{align}
                \Vert \mathit{T}(\mathcal{X} + \mathcal{H}) \Vert_* \geq \Vert \mathit{T}(\mathcal{X}) \Vert_* + \langle \mathcal{U} \star  \mathcal{V}^{\text{H}} + \mathcal{U}_{\perp} \star \mathcal{V}_{\perp}^{\text{H}}, \mathit{T}(\mathcal{H})\rangle, 
                \label{eq:by_subgradient_condition}
            \end{align}
            where $\mathcal{U}, \mathcal{V}$ are the singular tensors of $\mathit{T}(\mathcal{X}_0)$, $\mathbb{S}$ is spanned by $\mathcal{U}, \mathcal{V}$ and $\mathcal{U}_{\perp}, \mathcal{V}_{\perp}$ are the singular tensors of $\mathcal{P}_{\mathbb{S}^{\perp}}(\mathit{T}(\mathcal{H}))$.
    
            Now continue to analyze the second term in the RHS of \eqref{eq:by_subgradient_condition} and decompose it into three parts,
            \begin{align}
                \langle \mathcal{U} \star \mathcal{V}^{\text{H}} + \mathcal{U}_{\perp} \star \mathcal{V}_{\perp}^{\text{H}}, \mathit{T}(\mathcal{H})\rangle = \langle \mathcal{Y}, \mathit{T}(\mathcal{H}) \rangle + \langle \mathcal{U}_{\perp} \star \mathcal{V}_{\perp}^{\text{H}}, \mathit{T}(\mathcal{H}) \rangle - \langle \mathcal{Y} - \mathcal{U} \star \mathcal{V}^{\text{H}}, \mathit{T}(\mathcal{H})\rangle. \label{eq:main_analysis_term}
            \end{align}
        
            \begin{itemize}
                \item First look into $\langle \mathcal{Y}, \mathit{T}(\mathcal{H}) \rangle$.
                \begin{equation}
                    \begin{aligned}
                        & \langle \mathcal{Y}, \mathit{T}(\mathcal{H}) \rangle \\
                        & = \underbrace{\langle \mathcal{Y}, (\mathcal{I} - \mathit{T}\mathit{T}^{\dag})\mathit{T}(\mathcal{H}) \rangle}_{A_1} + \underbrace{\langle \mathcal{Y}, (\mathit{T}\mathit{T}^{\dag} - \mathit{T}\mathcal{S}_{\Omega}\mathit{T}^{\dag})\mathit{T}(\mathcal{H}) \rangle}_{A_2} + \underbrace{\langle \mathcal{Y}, \mathit{T}\mathcal{S}_{\Omega}\mathit{T}^{\dag}\mathit{T}(\mathcal{H}) \rangle}_{A_3},
                    \end{aligned}
                    \label{eq:3_term_of_<Y,TH>}
                \end{equation}
                which can also be trisected as:
                \begin{align}
                    A_1 & = \langle \mathcal{Y}, (\mathcal{I} - \mathit{T}\mathit{T}^{\dag})\mathit{T}(\mathcal{H}) \rangle = \langle \mathcal{Y}, \mathit{T}(\mathcal{H})- \mathit{T}(\mathcal{H}) \rangle = 0~(\text{by}~\mathbf{T}^{\dag}\mathbf{T} = \mathbf{I}), \\
                    A_2 & = \langle \mathcal{Y}, (\mathit{T}\mathit{T}^{\dag} - \mathit{T}\mathcal{S}_{\Omega}\mathit{T}^{\dag})\mathit{T}(\mathcal{H}) \rangle = \langle (\mathit{T}\mathit{T}^{\dag} - \mathit{T}\mathcal{S}_{\Omega}\mathit{T}^{\dag})^{\text{H}}\mathcal{Y}, \mathit{T}(\mathcal{H}) \rangle = 0~(\text{by}~\eqref{eq:cond:inner_product_condition}), \\
                    A_3 & = \langle \mathcal{Y}, \mathit{T}\mathcal{S}_{\Omega}\mathit{T}^{\dag}\mathit{T}(\mathcal{H}) \rangle = \langle \mathcal{Y}, \mathit{T}\mathcal{S}_{\Omega}(\mathcal{H}) \rangle = 0~(\text{by}~\mathcal{S}_{\Omega}(\mathcal{H}) = 0), 
                \end{align}
                which combines for $\langle \mathcal{Y}, \mathit{T}(\mathcal{H}) \rangle = 0$.
        
                \item For $\langle \mathcal{U}_{\perp} \star \mathcal{V}_{\perp}^{\text{H}}, \mathit{T}(\mathcal{H}) \rangle$, we have         
                \begin{align}
                    \langle \mathcal{U}_{\perp} \star \mathcal{V}_{\perp}^{\text{H}}, \mathit{T}(\mathcal{H}) \rangle = \langle \mathcal{U}_{\perp} \star \mathcal{V}_{\perp}^{\text{H}}, \mathcal{P}_{\mathbb{S}^{\perp}}(\mathit{T}(\mathcal{H})) \rangle = \Vert \mathcal{P}_{\mathbb{S}^{\perp}}(\mathit{T}(\mathcal{H})) \Vert_* 
                \end{align}
        
                \item For $\langle \mathcal{Y} - \mathcal{U} \star \mathcal{V}^{\text{H}}, \mathit{T}(\mathcal{H})\rangle$, an upper bound can be given:
                \begin{equation}
                    \begin{aligned}
                        & \langle \mathcal{Y} - \mathcal{U} \star \mathcal{V}^{\text{H}}, \mathit{T}(\mathcal{H})\rangle \\
                        % = & \langle \mathcal{P}_{\mathbb{S}}(\mathcal{Y} - \mathcal{U} \star \mathcal{V}^{\text{H}}), \mathit{T}(\mathcal{H})\rangle  + \langle \mathcal{P}_{\mathbb{S}^{\perp}}(\mathcal{Y} - \mathcal{U} \star \mathcal{V}^{\text{H}}), \mathit{T}(\mathcal{H})\rangle \\
                        = & \langle \mathcal{P}_{\mathbb{S}}(\mathcal{Y} - \mathcal{U} \star \mathcal{V}^{\text{H}}), \mathit{T}(\mathcal{H})\rangle  + \langle \mathcal{P}_{\mathbb{S}^{\perp}}(\mathcal{Y}), \mathit{T}(\mathcal{H})\rangle - \langle \mathcal{P}_{\mathbb{S}^{\perp}}(\mathcal{U} \star \mathcal{V}^{\text{H}}), \mathit{T}(\mathcal{H})\rangle \\
                        % = & \langle \mathcal{P}_{\mathbb{S}}(\mathcal{Y} - \mathcal{U} \star \mathcal{V}^{\text{H}}), \mathcal{P}_{\mathbb{S}}(\mathit{T}(\mathcal{H}))\rangle  + \langle \mathcal{P}_{\mathbb{S}^{\perp}}(\mathcal{Y}), \mathcal{P}_{\mathbb{S}^{\perp}}(\mathit{T}(\mathcal{H}))\rangle \\
                        \leq & \Vert \mathcal{P}_{\mathbb{S}}(\mathcal{Y} - \mathcal{U} \star \mathcal{V}^{\text{H}}) \Vert_{\text{F}} \cdot \Vert \mathcal{P}_{\mathbb{S}}(\mathit{T}(\mathcal{H})) \Vert_{\text{F}} + \Vert \mathcal{P}_{\mathbb{S}^{\perp}}(\mathcal{Y}) \Vert \cdot \Vert \mathcal{P}_{\mathbb{S}^{\perp}}(\mathit{T}(\mathcal{H})) \Vert_*. 
                    \end{aligned}
                    \label{eq:applying_CR_and_dual_norm}
                \end{equation}
                
                In the last equation of \eqref{eq:applying_CR_and_dual_norm}, we use the Cauchy–Schwarz inequality and the dual definition of tensor spectral norm \ref{def:tensor_spectral_norm} and tensor nuclear norm \ref{def:tensor_nuclear_norm}. 
        
                Plug \eqref{eq:cond:low_distortion} and \eqref{eq:cond:spectral_norm_cond} into \eqref{eq:applying_CR_and_dual_norm}, which yields
                    \begin{align}
                        \langle \mathcal{Y} - \mathcal{U} \star \mathcal{V}^{\text{H}}, \mathit{T}(\mathcal{H})\rangle \leq \frac{p}{8\kappa(\mathbf{T})} \Vert \mathcal{P}_{\mathbb{S}}(\mathit{T}(\mathcal{H})) \Vert_{\text{F}} + \frac{1}{2} \Vert \mathcal{P}_{\mathbb{S}^{\perp}}(\mathit{T}(\mathcal{H})) \Vert_*
                    \end{align}
                \end{itemize}
        
            Now \eqref{eq:main_analysis_term} can be bounded as 
            \begin{align}
                & \langle \mathcal{U} \star \mathcal{V}^{\text{H}} + \mathcal{U}_{\perp} \star \mathcal{V}_{\perp}^{\text{H}}, \mathit{T}(\mathcal{H})\rangle \geq \frac{1}{2} \Vert \mathcal{P}_{\mathbb{S}^{\perp}}(\mathit{T}(\mathcal{H})) \Vert_* - \frac{p}{8\kappa(\mathbf{T})} \Vert \mathcal{P}_{\mathbb{S}}(\mathit{T}(\mathcal{H})) \Vert_{\text{F}}.
            \end{align}
        
            Here the proof is furthered by a case-by-case analysis:
            \begin{enumerate}
                \item case 1: $\frac{1}{2} \Vert \mathcal{P}_{\mathbb{S}^{\perp}}(\mathit{T}(\mathcal{H})) \Vert_* > \frac{p}{8\kappa(\mathbf{T})} \Vert \mathcal{P}_{\mathbb{S}}(\mathit{T}(\mathcal{H})) \Vert_{\text{F}}$ 
                
                In this case, 
                \begin{align*}
                    \Vert \mathit{T}(\mathcal{X} + \mathcal{H}) \Vert_* & - \Vert \mathit{T}(\mathcal{X}) \Vert_* \nonumber
                    \geq \langle \mathcal{U} \star \mathcal{V}^{\text{H}} + \mathcal{U}_{\perp} \star \mathcal{V}_{\perp}^{\text{H}}, \mathit{T}(\mathcal{H})\rangle \\
                    & \geq \frac{1}{2} \Vert \mathcal{P}_{\mathbb{S}^{\perp}}(\mathit{T}(\mathcal{H})) \Vert_* - \frac{p}{8\kappa(\mathbf{T})} \Vert \mathcal{P}_{\mathbb{S}}(\mathit{T}(\mathcal{H})) \Vert_{\text{F}} > 0
                \end{align*}
        
                \item case 2: $\frac{1}{2} \Vert \mathcal{P}_{\mathbb{S}^{\perp}}(\mathit{T}(\mathcal{H})) \Vert_* \leq \frac{p}{8\kappa(\mathbf{T})} \Vert \mathcal{P}_{\mathbb{S}}(\mathit{T}(\mathcal{H})) \Vert_{\text{F}}$
        
                We show that in this complementary case, $\mathcal{H} = \mathbf{0}$. First, $\mathcal{S}_{\Omega}(\mathcal{H}) = 0$ $\Rightarrow$ $(\mathcal{S}_{\Omega} \mathit{T}^{\dag} \mathit{T}) (\mathcal{H}) = 0$ $\Rightarrow$ $(\mathit{T} \mathcal{R}_{\Omega} \mathit{T}^{\dag} + \mathcal{I} - \mathit{T}\mathit{T}^{\dag}) \mathit{T} (\mathcal{H}) = 0$. Thus we have 
                \begin{align*}
                    0 
                    % & = \langle \mathcal{P}_{\mathbb{S}}(\mathit{T}(\mathcal{H})), (\mathit{T} \mathcal{R}_{\Omega} \mathit{T}^{\dag} + \mathcal{I} - \mathit{T}\mathit{T}^{\dag}) \mathit{T} (\mathcal{H}) \rangle \nonumber \\
                    & = \underbrace{\langle \mathcal{P}_{\mathbb{S}}(\mathit{T}(\mathcal{H})), (\mathit{T} \mathcal{R}_{\Omega} \mathit{T}^{\dag} + \mathcal{I} - \mathit{T}\mathit{T}^{\dag}) \mathcal{P}_{\mathbb{S}}(\mathit{T} (\mathcal{H})) \rangle}_{B_1} \\ 
                    & + \underbrace{\langle \mathcal{P}_{\mathbb{S}}(\mathit{T}(\mathcal{H})), (\mathit{T} \mathcal{R}_{\Omega} \mathit{T}^{\dag} - \mathit{T}\mathit{T}^{\dag}) \mathcal{P}_{\mathbb{S}^{\perp}}(\mathit{T} (\mathcal{H})) \rangle}_{B_2}. \nonumber
                \end{align*}
    
                \begin{equation}
                    \begin{aligned}
                        \vert B_1 \vert & = \vert \langle \mathcal{P}_{\mathbb{S}}(\mathit{T}(\mathcal{H})), \mathcal{P}_{\mathbb{S}}(\mathit{T}(\mathcal{H})) \rangle + \langle \mathcal{P}_{\mathbb{S}}(\mathit{T}(\mathcal{H})), (\mathit{T} \mathcal{R}_{\Omega} \mathit{T}^{\dag} - \mathit{T}\mathit{T}^{\dag}) \mathcal{P}_{\mathbb{S}}(\mathit{T}(\mathcal{H})) \vert\\
                        & \geq (1 - \Vert \mathcal{P}_{\mathbb{S}} \mathit{T} \mathcal{R}_{\Omega} \mathit{T}^{\dag} \mathcal{P}_{\mathbb{S}} - \mathcal{P}_{\mathbb{S}} \mathit{T} \mathit{T}^{\dag} \mathcal{P}_{\mathbb{S}} \Vert) \cdot \Vert \mathcal{P}_{\mathbb{S}}(\mathit{T}(\mathcal{H})) \Vert_{\text{F}}^2 \\
                        & \geq \frac{1}{2} \Vert \mathcal{P}_{\mathbb{S}}(\mathit{T}(\mathcal{H})) \Vert_{\text{F}}^2~(\text{by \eqref{eq:cond:approximate_isometry}})
                    \end{aligned}
                \end{equation}
            
                Using the submultiplicative property of operator norm, 
                \begin{align}
                    \Vert \mathit{T} \mathcal{R}_{\Omega} \mathit{T}^{\dag} - \mathit{T}\mathit{T}^{\dag} \Vert \leq \Vert \mathbf{T} \Vert \cdot \Vert \mathcal{R}_{\Omega} - \mathcal{I} \Vert \cdot \Vert \mathbf{T}^{\dag} \Vert \leq \frac{\Vert \mathbf{T} \Vert \cdot \Vert \mathbf{T}^{\dag} \Vert}{p},
                \end{align}
                which gives  
                \begin{align}
                    \vert B_2 \vert \leq \frac{\kappa(\mathbf{T})}{p} \Vert \mathcal{P}_{\mathbb{S}}(\mathit{T}(\mathcal{H})) \Vert_{\text{F}} \Vert \mathcal{P}_{\mathbb{S}^{\perp}}(\mathit{T}(\mathcal{H})) \Vert_{\text{F}}.
                \end{align}

                Thus we have 
                \begin{align}
                    0 & = \vert B_1 + B_2 \vert \nonumber \geq \vert B_1 \vert - \vert B_2 \vert \\
                    & \geq \frac{1}{2} \Vert \mathcal{P}_{\mathbb{S}}(\mathit{T}(\mathcal{H})) \Vert_{\text{F}}^2 - \frac{\Vert \mathbf{T} \Vert \cdot \Vert \mathbf{T}^{\dag} \Vert}{p} \Vert \mathcal{P}_{\mathbb{S}}(\mathit{T}(\mathcal{H})) \Vert_{\text{F}} \Vert \mathcal{P}_{\mathbb{S}^{\perp}}(\mathit{T}(\mathcal{H})) \Vert_{\text{F}} \\
                    & \geq (\frac{1}{2} - 2 \frac{p}{8\kappa(\mathbf{T})} \frac{\kappa(\mathbf{T})}{p})\Vert \mathcal{P}_{\mathbb{S}}(\mathit{T}(\mathcal{H})) \Vert_{\text{F}}^2 = \frac{1}{4} \Vert \mathcal{P}_{\mathbb{S}}(\mathit{T}(\mathcal{H}))\Vert_{\text{F}}^2 \label{eq:by_case2_assump_and_F_norm<nuclear_norm},  
                \end{align}
                where \eqref{eq:by_case2_assump_and_F_norm<nuclear_norm} is followed by case 2 assumption and the inequality between Frobenius norm and nuclear norm that $\Vert \mathcal{P}_{\mathbb{S}^{\perp}}(\mathit{T}(\mathcal{H})) \Vert_{\text{F}} \leq \Vert \mathcal{P}_{\mathbb{S}^{\perp}}(\mathit{T}(\mathcal{H})) \Vert_{*}$. Thus, 
                \begin{align}
                    \Vert \mathcal{P}_{\mathbb{S}}(\mathit{T}(\mathcal{H})) \Vert_{\text{F}}^2 \leq 0 \Rightarrow \mathcal{P}_{\mathbb{S}}(\mathit{T}(\mathcal{H})) = \mathbf{0}. \label{eq:P_S_T_H=0}
                \end{align}
                
                If $\mathcal{H} \neq \mathbf{0}$, then $\mathit{T}(\mathcal{H}) \neq \mathbf{0}$ due to the injectivity of $\mathit{T}(\cdot)$. Thus $\mathcal{P}_{\mathbb{S}}(\mathit{T}(\mathcal{H}))$ and $\mathcal{P}_{\mathbb{S}^{\perp}}(\mathit{T}(\mathcal{H}))$ cannot be $\mathbf{0}$ at the same time following the fact that $\mathcal{P}_{\mathbb{S}}(\mathit{T}(\mathcal{H})) + \mathcal{P}_{\mathbb{S}^{\perp}}(\mathit{T}(\mathcal{H})) = \mathit{T}(\mathcal{H})$. By the assumption in case 2: $\mathcal{P}_{\mathbb{S}}(\mathit{T}(\mathcal{H})) \geq \frac{1}{2 c_{\epsilon}} \mathcal{P}_{\mathbb{S}^{\perp}}(\mathit{T}(\mathcal{H}))$, $\mathcal{P}_{\mathbb{S}}(\mathit{T}(\mathcal{H})) \neq 0$, which contradicts \eqref{eq:P_S_T_H=0}. Thus $\mathcal{H} = \mathbf{0}$.
            \end{enumerate}
            
            Combining case 1 and case 2, we obtain that: if $\mathcal{H} \neq \mathbf{0}$, then $\mathit{T}(\mathcal{X} + \mathcal{H}) > \mathit{T}(\mathcal{X})$, which shows that $\mathcal{X}$ is the unique optimal solution to \eqref{eq:program:min_transformed_TNN}. 
            \end{proof}
        
            \subsection{Construct Dual Certificate}
            \label{section:constructing_dual_certificate}
            We need to find a construction of dual certificate to prove that under certain conditions, the requirements in Proposition \ref{proposition:optimality_condition} can be satisfied.
    
            \begin{proposition}
                \label{proposition:constructing_dual_certificate}
                If $(\mathcal{X}, \mathit{T})$ satisfies \eqref{eq:inc_cond_U_V}-\eqref{eq:inc_cond_U_V_T} and \eqref{eq:sampling_rate_p} holds, then \eqref{eq:cond:approximate_isometry}\eqref{eq:cond:inner_product_condition}\eqref{eq:cond:low_distortion} and \eqref{eq:cond:spectral_norm_cond} can be satisfied. 
            \end{proposition}
    
            \begin{proof}
            We show these conditions can be satisfied in turn.
            \subsection{Show that \eqref{eq:cond:approximate_isometry} Can Be Satisfied}
            We show \eqref{eq:cond:approximate_isometry} can be satisfied in steps:
            \begin{enumerate}
                \item Express the operator $\mathcal{P}_{\mathbb{S}} \mathit{T} \mathcal{R}_{\Omega} \mathit{T}^{\dag} \mathcal{P}_{\mathbb{S}}$ as sum of inner product operators:
                \begin{equation}
                    \begin{aligned}
                        &\mathit{T}^{\dag} \mathcal{P}_{\mathbb{S}} (\mathcal{A}) = \sum_{i, j, k}\langle\mathit{T}^{\dag} \mathcal{P}_{\mathbb{S}} (\mathcal{A}), \mathdutchcal{e}_{ijk}\rangle \mathdutchcal{e}_{ijk} = \sum_{i, j, k}\langle\mathcal{A}, \mathcal{P}_\mathbb{S}\mathit{\tilde{T}}(\mathdutchcal{e}_{ijk})\rangle~\mathdutchcal{e}_{ijk}, \\
                        & \mathcal{R}_{\Omega} \mathit{T}^{\dag} \mathcal{P}_{\mathbb{S}} (\mathcal{A}) = \frac{1}{p} \sum_{i, j, k}\delta_{ijk}\langle\mathcal{A}, \mathcal{P}_\mathbb{S}\mathit{\tilde{T}}(\mathdutchcal{e}_{ijk})\rangle~\mathdutchcal{e}_{ijk}, \\
                        &\mathcal{P}_{\mathbb{S}} \mathit{T} \mathcal{R}_{\Omega} \mathit{T}^{\dag} \mathcal{P}_{\mathbb{S}} (\mathcal{A}) = \frac{1}{p} \sum_{i, j, k}\delta_{ijk}\langle\mathcal{A}, \mathcal{P}_\mathbb{S}\mathit{\tilde{T}}(\mathdutchcal{e}_{ijk})\rangle~\mathcal{P}_{\mathbb{S}} \mathit{T}(\mathdutchcal{e}_{ijk}), 
                    \end{aligned}
                \end{equation}
                where $\{\delta_{ijk}\}_{ijk}$ is a list of Bernoulli random variables: if $(i, j, k) \in \Omega$, $\delta_{ijk} = 1$ and equaling $0$ elsewhere. $\mathdutchcal{e}_{ijk} \in \mathbb{C}^{n_1 \times n_2 \times n_3}$ is a tensor with $\mathdutchcal{e}_{ijk}(i, j, k) = 1$ the rest equaling $0$.
    
                \item Apply the non-commutative Bernstein inequality:
                \begin{theorem}[Non-commutative Bernstein inequality]\cite{tropp2012user}
                    $\{\mathbf{Z}_K\}$ is a list of independent random $n_1 \times n_2$ matrices, satisfying $\mathbb{E}(\mathbf{Z}_k) = \mathbf{0}, \Vert \mathbf{Z}_k \Vert \leq R$ almost surely. Let $\sigma^2 = \max(\Vert \sum_k \mathbb{E}\left[\mathbf{Z}_k\mathbf{Z}_k^{\text{H}}\right]\Vert, \Vert\sum_k\mathbb{E}\left[\mathbf{Z}_k\mathbf{Z}_k^{\text{H}}\right]\Vert)$, then $\forall t > 0$, 
                    \begin{align}
                        \mathbb{P}\left[\Vert \sum_k\mathbf{Z}_k \Vert \geq t\right] &\leq (n_1 + n_2) \exp(-\frac{t^2}{2\sigma^2 + \frac{2}{3}Rt})\nonumber \\
                        & \leq (n_1 + n_2)\exp(-\frac{3t^2}{8\sigma^2})~(\text{if } t \leq \frac{\sigma^2}{R})\nonumber, 
                    \end{align}
                    or $\forall c > 0$, with probability at least $1 - (n_1 + n_2)^{1-c}$,
                    \begin{align}
                        \Vert \sum_k \mathbf{Z}_k \Vert \leq 2\sqrt{c\sigma^2\log(n_1 + n_2)} + c R \log(n_1 + n_2)\nonumber
                    \end{align}
                    \label{theorem:Bernstein}
                \end{theorem}
    
                Now we apply Theorem \ref{theorem:Bernstein}:
                let
                \begin{align*}
                    \mathcal{H}_{ijk}: \mathcal{A} \mapsto (\frac{\delta_{ijk}}{p} - 1)\langle\mathcal{A}, \mathcal{P}_\mathbb{S}\mathit{\tilde{T}}(\mathdutchcal{e}_{ijk})\rangle\mathcal{P}_{\mathbb{S}} \mathit{T}(\mathdutchcal{e}_{ijk}), \mathbb{E}(\mathcal{H}_{ijk}) = \mathbf{0},
                \end{align*}
                and we have
                \begin{align*}
                    \Vert\mathcal{H}_{ijk}(\mathcal{A})\Vert_{\text{F}}
                    & = \Vert (\frac{\delta_{ijk}}{p} - 1)\langle\mathcal{A}, \mathcal{P}_\mathbb{S}\mathit{\tilde{T}}(\mathdutchcal{e}_{ijk})\rangle\mathcal{P}_{\mathbb{S}} \mathit{T}(\mathdutchcal{e}_{ijk}) \Vert_{\text{F}} \\
                    & \leq \frac{1}{p}\Vert\mathcal{P}_\mathbb{S}\mathit{\tilde{T}}(\mathdutchcal{e}_{ijk})\Vert_{\text{F}}\cdot\Vert\mathcal{P}_\mathbb{S}\mathit{T}(\mathdutchcal{e}_{ijk})\Vert_{\text{F}}\cdot\Vert\mathcal{A}\Vert_{\text{F}}. 
                \end{align*}
                Define the matrix operator $\mathbf{H}_{ijk}(\cdot) = \texttt{bdiag}(\mathcal{H}_{ijk}(\texttt{fold}(\cdot)))$
                \begin{align}
                    \Vert \mathbf{H}_{ijk} \Vert \leq \frac{1}{p}\max_{ijk}\Vert\mathcal{P}_\mathbb{S}\mathit{\tilde{T}}(\mathdutchcal{e}_{ijk})\Vert_{\text{F}}\cdot\max_{ijk}\Vert\mathcal{P}_\mathbb{S}\mathit{T}(\mathdutchcal{e}_{ijk})\Vert_{\text{F}}.
                \end{align}
    
                Also, 
                \begin{align}
                    \mathcal{H}_{ijk}^{\text{H}}\mathcal{H}_{ijk}(\mathcal{A})
                    & = \mathcal{H}_{ijk}^{\text{H}}\left((\frac{\delta_{ijk}}{p} - 1)\langle\mathcal{A}, \mathcal{P}_\mathbb{S}\mathit{\tilde{T}}(\mathdutchcal{e}_{ijk})\rangle\mathcal{P}_{\mathbb{S}} \mathit{T}(\mathdutchcal{e}_{ijk})\right)\nonumber\\
                    & = (\frac{\delta_{ijk}}{p} - 1)^2\langle\mathcal{P}_\mathbb{S}\mathit{T}(\mathdutchcal{e}_{ijk}), \mathcal{P}_\mathbb{S}\mathit{T}(\mathdutchcal{e}_{ijk})\rangle\langle\mathcal{P}_\mathbb{S}\mathit{\tilde{T}}(\mathdutchcal{e}_{ijk}), \mathcal{A}\rangle\mathcal{P}_{\mathbb{S}}\mathit{\tilde{T}}(\mathdutchcal{e}_{ijk}).
                \end{align}
    
                Thus 
                \begin{align}
                    & \Vert \sum_{ijk} \mathbb{E}[\mathbf{H}_{ijk}^{\text{H}}\mathbf{H}_{ijk}(\mathcal{A})]\Vert_{\text{F}}\nonumber\\
                    & = \Vert \sum_{ijk} \mathbb{E}[(\frac{\delta_{ijk}}{p} - 1)^2]\langle\mathcal{P}_\mathbb{S}\mathit{T}(\mathdutchcal{e}_{ijk}), \mathcal{P}_\mathbb{S}\mathit{T}(\mathdutchcal{e}_{ijk})\rangle\langle\mathcal{P}_\mathbb{S}\mathit{\tilde{T}}(\mathdutchcal{e}_{ijk}), \mathcal{A}\rangle\mathcal{P}_{\mathbb{S}}\mathit{\tilde{T}}(\mathdutchcal{e}_{ijk})\Vert_{\text{F}}\nonumber\\
                    % & = \Vert\sum_{ijk}\frac{1-p}{p}\langle\mathcal{P}_\mathbb{S}\mathit{T}(\mathdutchcal{e}_{ijk}), \mathcal{P}_\mathbb{S}\mathit{T}(\mathdutchcal{e}_{ijk})\rangle\langle\mathcal{P}_\mathbb{S}\mathit{\tilde{T}}(\mathdutchcal{e}_{ijk}), \mathcal{A}\rangle\mathcal{P}_{\mathbb{S}}\mathit{\tilde{T}}(\mathdutchcal{e}_{ijk})\Vert_{\text{F}}\nonumber\\
                    &\leq\frac{1-p}{p}\max_{ijk}\Vert\mathcal{P}_\mathbb{S}\mathit{T}(\mathdutchcal{e}_{ijk})\Vert_{\text{F}}^2\cdot\Vert\sum_{ijk}\langle\mathcal{P}_\mathbb{S}\mathit{\tilde{T}}(\mathdutchcal{e}_{ijk}), \mathcal{A}\rangle\mathcal{P}_{\mathbb{S}}\mathit{\tilde{T}}(\mathdutchcal{e}_{ijk})\Vert_{\text{F}}\nonumber\\
                    &=\frac{1-p}{p}\max_{ijk}\Vert\mathcal{P}_\mathbb{S}\mathit{T}(\mathdutchcal{e}_{ijk})\Vert_{\text{F}}^2\cdot\Vert\mathcal{P}_{\mathbb{S}}\mathit{\tilde{T}}(\mathit{T}^{\dag}\mathcal{P}_{\mathbb{S}}(\mathcal{A}))^*\Vert_{\text{F}}\nonumber\\
                    &\leq\frac{1-p}{p}\max_{ijk}\Vert\mathcal{P}_\mathbb{S}\mathit{T}(\mathdutchcal{e}_{ijk})\Vert_{\text{F}}^2\cdot\Vert\mathcal{P}_{\mathbb{S}}\mathit{\tilde{T}}\Vert\cdot\Vert\mathit{T}^{\dag}\mathcal{P}_{\mathbb{S}}\Vert\cdot\Vert\mathcal{A}\Vert_{\text{F}}\label{eq:by_submul_op_1}\\
                    &\leq\frac{1-p}{p}\max_{ijk}\Vert\mathcal{P}_\mathbb{S}\mathit{T}(\mathdutchcal{e}_{ijk})\Vert_{\text{F}}^2\cdot\kappa(\mathbf{T})\cdot\Vert\mathcal{A}\Vert_{\text{F}},\label{eq:by_submul_op_2} 
                \end{align}
                where $\kappa(\mathbf{T})$ condition number of $\mathbf{T}$ and \eqref{eq:by_submul_op_1}, \eqref{eq:by_submul_op_2} follow from the submultiplicative property of operator norm and the fact that $\mathcal{P}_{\mathbb{S}}$ is a constrictive mapping, respectively.
                
                In the same fashion:
                \begin{align}
                    \Vert \sum_{ijk} \mathbb{E}[\mathbf{H}_{ijk}^{\text{H}}\mathbf{H}_{ijk}]\Vert&\leq\frac{\kappa(\mathbf{T})}{p}\max_{ijk}\Vert\mathcal{P}_\mathbb{S}\mathit{T}(\mathdutchcal{e}_{ijk})\Vert_{\text{F}}^2\\
                    \Vert \sum_{ijk} \mathbb{E}[\mathbf{H}_{ijk}\mathbf{H}_{ijk}^{\text{H}}]\Vert&\leq\frac{\kappa(\mathbf{T})}{p}\max_{ijk}\Vert\mathcal{P}_\mathbb{S}\mathit{\tilde{T}}(\mathdutchcal{e}_{ijk})\Vert_{\text{F}}^2
                \end{align}

                Now we apply Theorem \ref{theorem:Bernstein}:
                \begin{equation}
                    \begin{aligned}
                        & \mathbb{P}\left[ \Vert \mathcal{P}_{\mathbb{S}} \mathit{T} \mathcal{R}_{\Omega} \mathit{T}^{\dag} \mathcal{P}_{\mathbb{S}} - \mathcal{P}_{\mathbb{S}} \mathit{T} \mathit{T}^{\dag} \mathcal{P}_{\mathbb{S}} \Vert \geq t \right] \\
                        = & \mathbb{P}\left[ \Vert \sum_{ijk} \mathcal{H}_{ijk} \Vert \geq t \right] = \mathbb{P}\left[ \Vert \sum_{ijk} \mathbf{H}_{ijk} \Vert \geq t \right]\\
                        \leq & (n_1 + n_2)N_3 \cdot \\
                        & \exp\big(\frac{- t^2}{
                        \begin{matrix}
                            & \frac{2}{p}\kappa(\mathbf{T})\cdot\max\left\{\max_{ijk}\Vert\mathcal{P}_\mathbb{S}\mathit{T}(\mathdutchcal{e}_{ijk})\Vert_{\text{F}}^2, \max_{ijk}\Vert\mathcal{P}_\mathbb{S}\mathit{\tilde{T}}(\mathdutchcal{e}_{ijk})\Vert_{\text{F}}^2\right\} + \\ & \frac{2}{3p}\Vert\mathcal{P}_\mathbb{S}\mathit{\tilde{T}}(\mathdutchcal{e}_{ijk})\Vert_{\text{F}}\cdot\Vert\mathcal{P}_\mathbb{S}\mathit{T}(\mathdutchcal{e}_{ijk})\Vert_{\text{F}}
                        \end{matrix}}\big)
                    \end{aligned}
                    \label{eq:prob_approximate_isometry}
                \end{equation}
        
                Let $t = \frac{1}{2}$, by Appendix \ref{section:bound_F_norm} $\Vert \mathcal{P}_\mathbb{S}\mathit{T}(\mathdutchcal{e}_{ijk}) \Vert_{\text{F}}^2 \leq \frac{\nu r (n_1 + n_2)}{n_1 n_2} \Vert \mathbf{T} \Vert_{1\mapsto2}^2, \Vert \mathcal{P}_\mathbb{S}\mathit{\tilde{T}}(\mathdutchcal{e}_{ijk}) \Vert_{\text{F}}^2 \leq \frac{\nu r (n_1 + n_2)}{n_1 n_2} \Vert \tilde{\mathbf{T}} \Vert_{1\mapsto2}^2$ and use the fact that for a matrix $\mathbf{A}$, $\Vert \mathbf{A} \Vert_{\infty, 2} \leq \Vert \mathbf{A} \Vert$ holds, we can obtain that given the sampling rate in \eqref{eq:sampling_rate_p}, $\exists~c_3 > 0, \beta > 1$ s.t.
                \begin{align}
                    \mathbb{P}\left[ \Vert \mathcal{P}_{\mathbb{S}} \mathit{T} \mathcal{R}_{\Omega} \mathit{T}^{\dag} \mathcal{P}_{\mathbb{S}} - \mathcal{P}_{\mathbb{S}} \mathit{T} \mathit{T}^{\dag} \mathcal{P}_{\mathbb{S}} \Vert \geq \frac{1}{2} \right] \leq c_3^{-\beta}((n_1 + n_2)N_3)^{1 - \beta},
                \end{align}
                which means \eqref{eq:cond:approximate_isometry} is satisfied w.h.p..
            \end{enumerate}

            \subsection{Prove that \eqref{eq:cond:inner_product_condition}, \eqref{eq:cond:low_distortion}, \eqref{eq:cond:spectral_norm_cond} Can Be Satisfied}
                Use Golfing Scheme \cite{Golfing_Scheme} and construct $\mathcal{Y}$ as
                \begin{align}
                    & \mathcal{Y} = \sum_{i=1}^l (\frac{1}{p_i} \mathit{\tilde{T}} \mathcal{S}_{\Omega_i} \mathit{T}^{\text{H}} + \mathcal{I} - \mathit{\tilde{T}}\mathit{T}^{\text{H}}) \mathcal{M}_{i-1} \label{eq:constructing_Y}, \\
                    & \mathcal{M}_i = \mathcal{P}_{\mathbb{S}}(\mathit{\tilde{T}} \mathit{T}^{\text{H}} - \frac{1}{p_i} \mathit{\tilde{T}}\mathcal{S}_{\Omega_i}{\mathit{T}^{\text{H}}})\mathcal{M}_{i-1} \label{eq:constructing_M}, \\
                    & \mathcal{M}_0 = \mathcal{U} \star \mathcal{V}^{\text{H}}, \label{eq:M_0}
                \end{align}
                and $\{\Omega_i\}_{i=1}^{l}$ are $l$ independent sets of indices sampled by $\text{Ber}(p_i), p_i = 1 - (1 - p)^{\frac{1}{l}}$, $\bigcup_{i=1}^l \Omega_i = \Omega$, $l = 10\log(\kappa(\mathbf{T})(n_1 + n_2) N_3)$.
    
            \subsubsection{Prove that \eqref{eq:cond:inner_product_condition} Can Be Satisfied}
                First we need to prove the construction \eqref{eq:constructing_Y}, \eqref{eq:constructing_M}, \eqref{eq:M_0} satisfies \eqref{eq:cond:inner_product_condition}:
                \begin{equation}
                    \begin{aligned}
                        & (\mathit{T} \mathit{T}^{\dag} - \mathit{T} \mathcal{S}_{\Omega} \mathit{T}^{\dag})^{\text{H}} (\frac{1}{p_i} \mathit{\tilde{T}} \mathcal{S}_{\Omega_i} \mathit{T}^{\text{H}} + \mathcal{I} - \mathit{\tilde{T}}\mathit{T}^{\text{H}}) \\
                        = & \frac{1}{p_i} (\mathit{T} \mathit{T}^{\dag} - \mathit{T} \mathcal{S}_{\Omega} \mathit{T}^{\dag})^{\text{H}} \mathit{\tilde{T}} \mathcal{S}_{\Omega_i} \mathit{T}^{\text{H}} + (\mathit{T} \mathit{T}^{\dag} - \mathit{T} \mathcal{S}_{\Omega} \mathit{T}^{\dag})^{\text{H}} (\mathcal{I} - \mathit{\tilde{T}}\mathit{T}^{\text{H}}) \\
                        = & \mathbf{0}~(\text{by}~\mathbf{T}^{\text{H}}\mathbf{\tilde{T}} = (\mathbf{T}^{\dag}\mathbf{T})^{\text{H}} = \mathbf{I}, \mathcal{S}_{\Omega}\mathcal{S}_{\Omega_i} = \mathcal{S}_{\Omega_i}), 
                    \end{aligned}
                \end{equation}
                thus \eqref{eq:cond:inner_product_condition} is satisfied.

            \subsubsection{Prove that \eqref{eq:cond:low_distortion} Can Be Satisfied}
                Following the construction, 
                \begin{equation}
                    \begin{aligned}
                        \mathcal{P}_{\mathbb{S}}(\mathcal{Y}) & = \sum_{i=1}^l \mathcal{P}_{\mathbb{S}}(\frac{1}{p_i}\mathit{\tilde{T}}\mathcal{S}_{\Omega_i}\mathit{T}^{\text{H}} + \mathcal{I} - \mathit{\tilde{T}}\mathit{T}^{\text{H}})(\mathcal{M}_{i-1}) \\
                        & = \sum_{i=1}^l \mathcal{P}_{\mathbb{S}} \mathcal{M}_{i-1} - \mathcal{P}_{\mathbb{S}}(\mathit{\tilde{T}}\mathit{T}^{\text{H}} - \frac{1}{p_i}\mathit{\tilde{T}}\mathcal{S}_{\Omega_i}\mathit{T}^{\text{H}}) \mathcal{M}_{i-1} \\
                        & = \sum_{i=1}^l \mathcal{M}_{i-1} - \mathcal{M}_i = \mathcal{U} \star \mathcal{V}^{\text{H}} - \mathcal{M}_l.
                    \end{aligned}
                \end{equation}
                Therefore, $\Vert \mathcal{P}_{\mathbb{S}}(\mathcal{Y}) - \mathcal{U} \star \mathcal{V}^{\text{H}} \Vert_{\text{F}} = \Vert \mathcal{M}_l \Vert_{\text{F}}$.
        
                Since $\mathcal{M}_0 \in \mathbb{S}$, we desire that
                \begin{equation}
                    \begin{aligned}
                        \Vert \mathcal{M}_i \Vert_{\text{F}}& = \Vert \mathcal{P}_{\mathbb{S}}(\mathit{\tilde{T}} \mathit{T}^{\text{H}} - \frac{1}{p_i} \mathit{\tilde{T}}\mathcal{S}_{\Omega_i}{\mathit{T}^{\text{H}}})\mathcal{M}_{i-1} \Vert_{\text{F}} \\
                        & \leq \Vert\mathcal{P}_{\mathbb{S}}\mathit{\tilde{T}} \mathit{T}^{\text{H}}\mathcal{P}_{\mathbb{S}} - \frac{1}{p_i}\mathcal{P}_{\mathbb{S}}\mathit{\tilde{T}}\mathcal{S}_{\Omega_i}{\mathit{T}^{\text{H}}}\mathcal{P}_{\mathbb{S}}\Vert\cdot\Vert\mathcal{M}_{i-1}\Vert_{\text{F}} \leq \frac{1}{2} \Vert\mathcal{M}_{i-1}\Vert_{\text{F}} .
                    \end{aligned}
                    \label{eq:step_p_i}
                \end{equation}
        
                Note that \eqref{eq:step_p_i} puts a requirement similar to $\eqref{eq:cond:approximate_isometry}$ on the per-step sampling rate $p_i$. Also, $p_i = 1 - (1 - p)^{\frac{1}{l}} \geq \frac{p}{l}, l = 10\log(\kappa(\mathbf{T})(n_1 + n_2) N_3)$, thus $\exists~c_4 > 0$, 
                \begin{align}
                    p_i \geq c_4 (\kappa^2(\mathbf{T}) + \rho^2(\mathbf{T}))\frac{\lambda r (n_1 + n_2)}{n_1 n_2} \log(\kappa(\mathbf{T})(n_1 + n_2)N_3).
                    \label{eq:sampling_rate_p_i}
                \end{align}
                Using \eqref{eq:prob_approximate_isometry}, $\exists~c_5 > 0, \beta^{'} > 1$ by union bound $\Vert \mathcal{M}_i \Vert_{\text{F}} \leq \frac{1}{2}\Vert \mathcal{M}_{i-1} \Vert_{\text{F}}$ holds for all $i = 1, 2, \hdots, l$ with probability at least $1 - c_5((n_1 + n_2)N_3)^{1-\beta^{'}}$, which ensures \eqref{eq:step_p_i} holds. 
    
                As $l = 10\log(\kappa(\mathbf{T})(n_1 + n_2) N_3)$, then w.h.p.
                \begin{align}
                    \Vert \mathcal{P}_{\mathbb{S}}(\mathcal{Y}) - \mathcal{U} \star \mathcal{V}^{\text{H}} \Vert_{\text{F}} = \Vert \mathcal{M}_l \Vert_{\text{F}} \leq (\frac{1}{2})^l \Vert \mathcal{U} \star \mathcal{V} \Vert_{\text{F}} \leq (\frac{1}{2})^l \sqrt{r N_3} \leq \frac{p}{8\kappa(\mathbf{T})}.
                \end{align}
    
             \subsubsection{Prove that \eqref{eq:cond:spectral_norm_cond} Can Be Satisfied}
            
                Now we prove $\eqref{eq:cond:spectral_norm_cond}$. 
    
                First we use the construction \eqref{eq:constructing_Y} and the triangle inequality:
                \begin{equation}
                    \begin{aligned}
                        \Vert \mathcal{P}_{\mathbb{S}^{\perp}}(\mathcal{Y}) \Vert 
                        % & = \Vert \sum_{i=1}^l \mathcal{P}_{\mathbb{S}^{\perp}}(\frac{1}{p_i}\mathit{\tilde{T}}\mathcal{S}_{\Omega_i}\mathit{T}^{\text{H}} + \mathcal{I} - \mathit{\tilde{T}}\mathit{T}^{\text{H}})(\mathcal{M}_{i-1}) \Vert\\
                        = & \Vert \sum_{i=1}^l \mathcal{P}_{\mathbb{S}^{\perp}}(\frac{1}{p_i}\mathit{\tilde{T}}\mathcal{S}_{\Omega_i}\mathit{T}^{\text{H}} - \mathit{\tilde{T}}\mathit{T}^{\text{H}})(\mathcal{M}_{i-1}) \Vert\\
                        \leq & \sum_{i=1}^l \Vert \mathcal{P}_{\mathbb{S}^{\perp}}(\frac{1}{p_i}\mathit{\tilde{T}}\mathcal{S}_{\Omega_i}\mathit{T}^{\text{H}} - \mathit{\tilde{T}}\mathit{T}^{\text{H}})(\mathcal{M}_{i-1}) \Vert\\
                        \leq & \sum_{i=1}^l \Vert (\frac{1}{p_i}\mathit{\tilde{T}}\mathcal{S}_{\Omega_i}\mathit{T}^{\text{H}} - \mathit{\tilde{T}}\mathit{T}^{\text{H}})(\mathcal{M}_{i-1}) \Vert. 
                    \end{aligned}
                \end{equation}
        
                Here we need a lemma for the initial control of the spectral norm and the lemma is proved in Appendix~\ref{proof:bounding_spectral_norm}.
                \begin{lemma}
                    Assuming that the Bernoulli sampling scheme is employed and $\forall c > 0$, 
                    \begin{align}
                        \Vert \mathit{\tilde{T}}(\frac{1}{p} \mathcal{S}_{\Omega} - \mathcal{I})\mathcal{Z} \Vert \leq 2\sqrt{\frac{c}{p} \log((n_1 + n_2)N_3)} \Vert \tilde{\mathbf{T}} \Vert_{\infty} \Vert \mathcal{Z} \Vert_{\infty, 2} +  \frac{c}{p} \log((n_1 + n_2)N_3) \Vert \tilde{\mathbf{T}} \Vert_{\infty} \Vert \mathcal{Z} \Vert_{\infty} \nonumber
                    \end{align}
                    holds with probability at least $1 - (n_1 + n_2)^{1 - c}$.
                    \label{lemma:bounding_spectral_norm}
                \end{lemma}
    
                By Lemma~\ref{lemma:bounding_spectral_norm}, for $c_5 > 1$, w.h.p.
                \begin{equation}
                    \begin{aligned}
                        & \Vert \mathcal{P}_{\mathbb{S}^{\perp}}(\mathcal{Y}) \Vert \leq \sum_{i=1}^l \Vert (\frac{1}{p_i}\mathit{\tilde{T}}\mathcal{S}_{\Omega_i}\mathit{T}^{\text{H}} - \mathit{\tilde{T}}\mathit{T}^{\text{H}})(\mathcal{M}_{i-1}) \Vert \\
                        % \leq & \sum_{i=1}^{l}c_5 \frac{1}{p_i} \log((n_1 + n_2)N_3) \Vert \tilde{\mathbf{T}} \Vert_{\infty} \Vert \mathit{T}^{\text{H}}(\mathcal{M}_{i-1}) \Vert_{\infty} + 2\sqrt{c_5\frac{1}{p_i} \log((n_1 + n_2)N_3)} \Vert \tilde{\mathbf{T}} \Vert_{\infty} \Vert \mathit{T}^{\text{H}}(\mathcal{M}_{i-1}) \Vert_{\infty, 2} \\
                        \leq & c_5\sum_{i=1}^{l} \frac{1}{p_i} \log((n_1 + n_2)N_3) \Vert \tilde{\mathbf{T}} \Vert_{\infty} \Vert \mathit{T}^{\text{H}}(\mathcal{M}_{i-1}) \Vert_{\infty} \\ + & 2\sqrt{\frac{1}{p_i} \log((n_1 + n_2)N_3)} \Vert \tilde{\mathbf{T}} \Vert_{\infty} \Vert \mathit{T}^{\text{H}}(\mathcal{M}_{i-1}) \Vert_{\infty, 2}.
                    \end{aligned}
                    \label{eq:bounding_spectral_norm}
                \end{equation}
    
                Now another lemma is introduced to control the $\Vert\cdot\Vert_\infty$ norm and we prove it in Appendix~\ref{proof:bounding_infinity_norm}.
                \begin{lemma}
                    Assuming that the Bernoulli sampling scheme is employed, if $p \geq c_0 \kappa^2(\mathbf{T})\frac{\nu r (n_1 + n_2)}{n_1 n_2}\log((n_1 + n_2)N_3)$, then 
                    \begin{align}
                        \Vert \mathit{T}^{\text{H}}\mathcal{P}_{\mathbb{S}}\mathit{\tilde{T}}(\mathcal{I} - \frac{1}{p}\mathcal{S}_{\Omega})\mathcal{Z} \Vert_{\infty} \leq \frac{1}{2} \Vert \mathcal{Z} \Vert_{\infty} \nonumber                    \end{align}
                    holds with probability at least $1 - 2(n_1 + n_2)^{2-\frac{3c_0}{32}}N_3^{1-\frac{3c_0}{32}}$.
                    \label{lemma:bounding_infinity_norm}
                \end{lemma}
    
                By Lemma~\ref{lemma:bounding_infinity_norm} and \eqref{eq:sampling_rate_p_i}, w.h.p. 
                \begin{equation}
                    \begin{aligned}
                        \Vert \mathit{T}^{\text{H}}(\mathcal{M}_{i-1}) \Vert_{\infty} &= \Vert \mathit{T}^{\text{H}}\mathcal{P}_{\mathbb{S}}\mathit{\tilde{T}}(\mathcal{I} - \frac{1}{p_i}\mathcal{S}_{\Omega})\mathit{T}^{\text{H}}\mathcal{M}_{i-2} \Vert_{\infty}\\
                        & \leq \frac{1}{2} \Vert \mathit{T}^{\text{H}}(\mathcal{M}_{i-2}) \Vert_{\infty} \leq (\frac{1}{2})^{i-1} \Vert \mathit{T}^{\text{H}}(\mathcal{U} \star \mathcal{V}^{\text{H}}) \Vert_{\infty}. 
                    \end{aligned}
                    \label{eq:bounding_infinity_norm}
                \end{equation}
        
                Also a lemma that controls $\Vert\cdot\Vert_{\infty, 2}$ norm is needed and the proof is given in Appendix~\ref{proof:bounding_infinity_2_norm}.
                \begin{lemma}
                    Assuming that the Bernoulli sampling scheme is employed, if $p \geq c_0 (\kappa^2(\mathbf{T}) + \rho^2(\mathbf{T}))\frac{\lambda r (n_1 + n_2)}{n_1 n_2}\log((n_1 + n_2)N_3)$ with $c_0$ large enough, then $\forall~c > 0$, 
                    \begin{align}
                        \Vert \mathit{T}^{\text{H}} \mathcal{P}_{\mathbb{S}} \mathit{\tilde{T}}(\mathcal{I} - \frac{1}{p}\mathcal{S}_{\Omega}) \mathcal{Z} \Vert_{\infty, 2} \leq \frac{1}{2} \Vert \mathcal{Z} \Vert_{\infty, 2} + \frac{1}{2} \sqrt{\frac{n_1 n_2}{\mu r (n_1 + n_2)}}\Vert \mathcal{Z} \Vert_{\infty} \nonumber
                    \end{align}
                    holds with probability at least $1 - 2((n_1 + n_2)n_3)^{1-c}$.
                    \label{lemma:bounding_infinity_2_norm}
                \end{lemma}
            
                By Lemma~\ref{lemma:bounding_infinity_2_norm}, w.h.p. 
                \begin{equation}
                    \begin{aligned}
                        & \Vert \mathit{T}^{\text{H}}(\mathcal{M}_{i-1}) \Vert_{\infty, 2} \\
                        = & \Vert \mathit{T}^{\text{H}}\mathcal{P}_{\mathbb{S}}\mathit{\tilde{T}}(\mathcal{I} - \frac{1}{p}\mathcal{S}_{\Omega})\mathit{T}^{\text{H}}(\mathcal{M}_{i-2}) \Vert_{\infty, 2} \\
                        \leq & \frac{1}{2} \sqrt{\frac{n_1 n_2}{\mu r (n_1 + n_2)}} \Vert \mathit{T}^{\text{H}}(\mathcal{M}_{i-2}) \Vert_{\infty} + \frac{1}{2} \Vert \mathit{T}^{\text{H}}(\mathcal{M}_{i-2}) \Vert_{\infty, 2} \\
                        \leq & (\frac{1}{2})^2 \sqrt{\frac{n_1 n_2}{\mu r (n_1 + n_2)}} \Vert \mathit{T}^{\text{H}}(\mathcal{M}_{i-3}) \Vert_{\infty} \\ + & \frac{1}{2} (\frac{1}{2} \sqrt{\frac{n_1 n_2}{\mu r (n_1 + n_2)}} \Vert \mathit{T}^{\text{H}}(\mathcal{M}_{i-3}) \Vert_{\infty} + \frac{1}{2} \Vert \mathit{T}^{\text{H}}(\mathcal{M}_{i-3}) \Vert_{\infty, 2}) \\
                        \leq & (i-1)(\frac{1}{2})^{i-1} \sqrt{\frac{n_1 n_2}{\mu r (n_1 + n_2)}} \Vert \mathit{T}^{\text{H}}\mathcal{M}_{0} \Vert_{\infty} + (\frac{1}{2})^{i-1} \Vert \mathit{T}^{\text{H}}\mathcal{M}_{0} \Vert_{\infty, 2}
                    \end{aligned}
                \label{eq:bounding_infinity_2_norm}
                \end{equation}
        
                Plug \eqref{eq:bounding_infinity_norm} and \eqref{eq:bounding_infinity_2_norm} into \eqref{eq:bounding_spectral_norm}, yielding
                \begin{equation}
                    \begin{aligned}
                        & \Vert \mathcal{P}_{\mathbb{S}^{\perp}}(\mathcal{Y}) \Vert \\
                        \leq & c_5\Vert \tilde{\mathbf{T}} \Vert_{\infty} \sum_{i=1}^{l} \frac{\log((n_1 + n_2)N_3)}{p_i}  \Vert \mathit{T}^{\text{H}}(\mathcal{M}_{i-1}) \Vert_{\infty} + 2\sqrt{\frac{\log((n_1 + n_2)N_3)}{p_i} } \Vert \mathit{T}^{\text{H}}(\mathcal{M}_{i-1}) \Vert_{\infty, 2} \\
                        \leq & c_5 \Vert \tilde{\mathbf{T}} \Vert_{\infty} \sum_{i=1}^l \frac{\log((n_1 + n_2)N_3)}{p_i}  \Vert \mathit{T}^{\text{H}}(\mathcal{M}_{i-1}) \Vert_{\infty} + 2\sqrt{\frac{\log((n_1 + n_2)N_3)}{p_i} } \Vert \mathit{T}^{\text{H}}(\mathcal{M}_{i-1}) \Vert_{\infty, 2} \\
                        \leq & c_5 \Vert \tilde{\mathbf{T}} \Vert_{\infty} \frac{\log((n_1 + n_2)N_3)}{p_i}  \sum_{i=1}^l (\frac{1}{2})^{i-1} \Vert \mathit{T}^{\text{H}}(\mathcal{U} \star \mathcal{V}^{\text{H}}) \Vert_{\infty} + 2 c_5 \Vert \tilde{\mathbf{T}} \Vert_{\infty}
                        \sqrt{\frac{\log((n_1 + n_2)N_3)}{p_i} } \\ & \sum_{i=1}^l (i-1)(\frac{1}{2})^{i-1} \sqrt{\frac{n_1 n_2}{\mu r (n_1 + n_2)}} \Vert \mathit{T}^{\text{H}}(\mathcal{U} \star \mathcal{V}^{\text{H}}) \Vert_{\infty} + (\frac{1}{2})^{i-1} \Vert \mathit{T}^{\text{H}}(\mathcal{U} \star \mathcal{V}^{\text{H}}) \Vert_{\infty, 2} \\
                        \leq & c_5 \Vert \tilde{\mathbf{T}} \Vert_{\infty} \frac{\log((n_1 + n_2)N_3)}{p_i}  \Vert \mathit{T}^{\text{H}}(\mathcal{U} \star \mathcal{V}^{\text{H}}) \Vert_{\infty} + 8 c_5 \Vert \tilde{\mathbf{T}} \Vert_{\infty} \sqrt{\frac{n_1 n_2 \log((n_1 + n_2)N_3)}{\mu r (n_1 + n_2) p_i}} \\ & \Vert \mathit{T}^{\text{H}}(\mathcal{U} \star \mathcal{V}^{\text{H}}) \Vert_{\infty} + 4 c_5 \Vert \tilde{\mathbf{T}} \Vert_{\infty} \sqrt{\frac{\log((n_1 + n_2)N_3)}{p_i} } \Vert \mathit{T}^{\text{H}}(\mathcal{U} \star \mathcal{V}^{\text{H}}) \Vert_{\infty, 2}
                    \end{aligned}
                    \label{eq:P_S_perp_Y}
                \end{equation}
        
                Now it remains to bound $\Vert \mathit{T}^{\text{H}}(\mathcal{U} \star \mathcal{V}^{\text{H}}) \Vert_{\infty}$ and $\Vert \mathit{T}^{\text{H}}(\mathcal{U} \star \mathcal{V}^{\text{H}}) \Vert_{\infty, 2}$.
        
                Here, in order to bound $\Vert \mathit{T}^{\text{H}}(\mathcal{U} \star \mathcal{V}^{\text{H}}) \Vert_{\infty}$, a domain transform strategy is employed and we revisit the definition of generalized convolution $\star_{\mathbf{A}, \mathbf{B}}$ induced by two transforms $\mathbf{A}, \mathbf{B}$ \cite{Generalized_convolution}, which gives the form induced by $\mathit{T}(\cdot)$ of each entry in the original domain: 
                \begin{theorem}[Generalized convolution]\cite{Generalized_convolution}
                    Suppose $\mathbf{x}, \mathbf{y} \in \mathbb{C}^{m \times 1}, \mathbf{A} \in \mathbb{C}^{M \times m}, \mathbf{B} \in \mathbb{C}^{m \times M}$
                    \begin{align}
                        \mathbf{z} = \mathbf{x} \star_{\mathbf{A}, \mathbf{B}} \mathbf{y} = \mathbf{B}(\mathbf{A}(\mathbf{x}) \star \mathbf{A}(\mathbf{y})). \nonumber
                    \end{align}
                    
                    Each entry of $\mathbf{z}$ satisfies that
                    \begin{equation}
                        \begin{aligned}
                            \mathbf{z}(m) &= \mathbf{x}^{\text{T}} \mathbf{W}_{\mathbf{A}, \mathbf{B}}^m \mathbf{y} \\
                            \mathbf{W}_{\mathbf{A}, \mathbf{B}}^m(i, j) &= \sum_{k} \mathbf{B}(m, k) \cdot \mathbf{A}(k, i) \cdot \mathbf{A}(k, j). 
                        \end{aligned}
                        \label{eq:W_A_B}
                    \end{equation}
                    \label{theorem:generalized_convolution}
                \end{theorem}
        
                Let $\mathbf{A} = \mathbf{I}, \mathbf{B} = \mathbf{T}^{\text{H}}$ and define $\mathbf{u}_t^i = \mathcal{U}(i, t, :), {\mathbf{v}_t^j}^{\text{H}} = \mathcal{V}^{\text{H}}(t, j, :)$, 
                \begin{equation}
                    \begin{aligned}
                        & \Vert \mathit{T}^{\text{H}}(\mathcal{U} \star \mathcal{V}^{\text{H}}) \Vert_{\infty} 
                        = \Vert \mathcal{U} \star_{\mathit{T}} \mathcal{V}^{\text{H}} \Vert_{\infty} \\
                        = & \max_{ij} \Vert \mathcal{U}(i, :, :) \star_{\mathit{T}} \mathcal{V}^{\text{H}}(:, j, :) \Vert_{\infty} = \max_{ij} \Vert \sum_{t=1}^r \mathbf{u}_t^i \star_{\mathit{T}} {\mathbf{v}_t^j}^{\text{H}}\Vert_{\infty} \\
                        \leq & \sum_{t=1}^r \max_{ij} \Vert \mathbf{u}_t^i \star_{\mathit{T}} {\mathbf{v}_t^j}^{\text{H}}\Vert_{\infty} \leq \max_{ij} \{\sum_{t=1}^r \max_{k} \Vert \mathbf{W}_{\mathbf{T}}^k \Vert \cdot \Vert \mathbf{u}_t^i \Vert_2 \cdot \Vert \mathbf{v}_t^j \Vert_2\} \\
                        = & \max_{k} \Vert \mathbf{W}_{\mathbf{T}}^k \Vert\cdot \max_{ij}\{ \sum_{t=1}^r \frac{1}{2} \Vert \mathbf{u}_t^i \Vert_2^2 + \frac{1}{2} \Vert \mathbf{v}_t^j \Vert_2^2 \}\\
                        = & \max_{k} \Vert\mathbf{W}_{\mathbf{T}}^k\Vert \cdot \max_{ij} \{\frac{1}{2} \Vert \mathcal{U}^{\text{H}} \star \xi_i \Vert_{\text{F}}^2 + \frac{1}{2} \Vert \mathcal{V}^{\text{H}} \star \xi_j \Vert_{\text{F}}^2\} 
                    \end{aligned}
                \end{equation}
            
                Now we analyze $\Vert \mathbf{W}_{\mathbf{T}}^k\Vert$. We write \eqref{eq:W_A_B} into the matrix form:
                \begin{align}
                    \mathbf{W}_{\mathbf{T}}^k &= \text{diag}(\mathbf{T}^{\text{H}}(k, :)),
                \end{align}
                which means that if we cast $\{\mathbf{W}_{\mathbf{T}}^k\}_{k=1}^{N_3}$ into a three-dimensional tensor $\mathcal{W}$, it is a diagonal tensor and its diagonal tubes are formed by matrix $\mathbf{T}^{\text{H}}$. Thus $\max_k \Vert \mathbf{W}_{\mathbf{T}}^k\Vert = \Vert \mathcal{W} \Vert = \Vert \mathbf{T} \Vert_{\infty}$ and
                \begin{equation}
                    \begin{aligned}
                        \Vert \mathit{T}^{\text{H}}(\mathcal{U} \star \mathcal{V}^{\text{H}}) \Vert_{\infty} \leq \Vert \mathbf{T} \Vert_{\infty} \max_{ij} \{\frac{1}{2} \Vert \mathcal{U}^{\text{H}} \star \xi_i \Vert_{\text{F}}^2 + \frac{1}{2} \Vert \mathcal{V}^{\text{H}} \star \xi_j \Vert_{\text{F}}^2\} \leq \Vert \mathbf{T} \Vert_{\infty} \frac{\mu r (n_1 + n_2) N_3}{2 n_1 n_2}.
                    \end{aligned}
                    \label{eq:T_H_U_V_H_infty}
                \end{equation}
                Another reason why we introduce Theorem \ref{theorem:generalized_convolution} here is that when $\mathbf{T}$ is a square matrix (not necessarily orthogonal), by letting $\mathbf{B} = \tilde{\mathbf{T}}$ the bound can be traced back to the singular tensors in the original domain and possibly some tighter bound can be obtained. 
        
                For $\Vert \mathit{T}^{\text{H}}(\mathcal{U} \star \mathcal{V}^{\text{H}}) \Vert_{\infty, 2}$,
                \begin{equation}
                    \begin{aligned}
                        \Vert \mathit{T}^{\text{H}}(\mathcal{U} \star \mathcal{V}^{\text{H}}) \Vert_{\infty, 2} & \leq \max_{ij} \{\Vert \mathit{T}^{\text{H}}(\mathcal{U} \star \mathcal{V}^{\text{H}}) \star \mathdutchcal{e}_j \Vert_{\text{F}}, \Vert \mathdutchcal{e}_i^{\text{H}} \star \mathit{T}^{\text{H}}(\mathcal{U} \star \mathcal{V}^{\text{H}}) \Vert_{\text{F}} \} \\
                        & = \max_{ij} \{\Vert \mathit{T}^{\text{H}}(\mathcal{U} \star \mathcal{V}^{\text{H}} \star \xi_j) \Vert_{\text{F}}, \Vert \mathit{T}^{\text{H}}(\xi_i^{\text{H}} \star \mathcal{U} \star \mathcal{V}^{\text{H}}) \Vert_{\text{F}} \} \\
                        & \leq \Vert\mathbf{T}^{\text{H}}\Vert\sqrt{\frac{\mu r (n_1 + n_2) N_3}{n_1 n_2}}
                    \end{aligned}
                    \label{eq:T_H_U_V_H_infty_2}
                \end{equation}
        
                Thus, plug $p_i \geq \frac{c_0}{10} (\kappa^2(\mathbf{T}) + \rho^2(\mathbf{T}))\frac{\lambda r (n_1 + n_2)}{n_1 n_2}\log((n_1 + n_2)N_3)$ and \eqref{eq:T_H_U_V_H_infty}, \eqref{eq:T_H_U_V_H_infty_2} into \eqref{eq:P_S_perp_Y} and we finally obtain w.h.p.
                \begin{equation}
                    \begin{aligned}
                        & \Vert \mathcal{P}_{\mathbb{S}^{\perp}}(\mathcal{Y}) \Vert \\
                        \leq & \frac{c_5}{p_i}\cdot\frac{\rho(\mathbf{T})\mu r (n_1 + n_2) \log((n_1 + n_2)N_3)}{2 n_1 n_2} + 8 c_5 \rho(\mathbf{T}) \sqrt{\frac{\mu r (n_1 + n_2) \log((n_1 + n_2)N_3)}{2n_1n_2p_i}} \\
                        & + 4 c_5 \sqrt{\frac{(\kappa^2(\mathbf{T}) + \rho^2(\mathbf{T}))\mu r (n_1 + n_2)}{n_1 n_2 p_i}} \leq \frac{10c_5}{c_0} + \frac{32c_5}{\sqrt{c_0}} + \frac{16c_5}{\sqrt{c_0}} \leq \frac{1}{2}
                    \end{aligned}
                \end{equation}
                with $c_0$ large enough.
            \end{proof}

            \section{Proofs of the Norm-bounding Lemmas}
                \label{section:proofs_of_lemmas}
                \subsection{Proof of Lemma~\ref{lemma:bounding_spectral_norm}}
                \label{proof:bounding_spectral_norm}
                \begin{proof}
                    First, we decompose the operator, 
                    \begin{align}
                        & \mathit{\tilde{T}}(\frac{1}{p} \mathcal{S}_{\Omega} - \mathcal{I})\mathcal{Z} 
                        = \sum_{ijk} (\frac{\delta_{ijk}}{p} - 1) \mathcal{Z}_{ijk} \mathit{\tilde{T}}(\mathdutchcal{e}_{ijk})
                        := \sum_{ijk} \mathcal{L}_{ijk}.
                    \end{align}
    
                    With $\mathbb{E}\left[\texttt{bdiag}(\mathcal{L}_{ijk})\right] = \mathbf{0}$, 
                    \begin{align}
                        \Vert \texttt{bdiag}(\mathcal{L}_{ijk}) \Vert & \leq \frac{1}{p} \Vert \mathcal{Z} \Vert_{\infty} \max_{ijk}\Vert\texttt{bdiag}(\mathit{\tilde{T}}(\mathdutchcal{e}_{ijk}))\Vert \leq \frac{1}{p} \Vert \mathcal{Z} \Vert_{\infty}\Vert\tilde{\mathbf{T}}\Vert_{\infty}
                    \end{align}
                    \begin{align}
                        & \Vert \sum_{ijk} \mathbb{E}[\texttt{bdiag}^{\text{H}}(\mathcal{L}_{ijk})\texttt{bdiag}(\mathcal{L}_{ijk})]\Vert \nonumber \\ = & \Vert \sum_{ijk} \mathbb{E}\left[ (\frac{\delta_{ijk}}{p} - 1)^2 \right] \cdot \vert \mathcal{Z}_{ijk} \vert^2 \cdot \texttt{bdiag}(\left[ \mathit{\tilde{T}}(\mathdutchcal{e}_{ijk}) \right]^{\text{H}} \star \mathit{\tilde{T}}(\mathdutchcal{e}_{ijk})) \Vert \nonumber \\
                        = & \frac{1-p}{p} \Vert \sum_{ijk} \vert \mathcal{Z}_{ijk} \vert^2 \cdot \left[ \mathit{\tilde{T}}(\mathdutchcal{e}_{ijk}) \right]^{\text{H}} \mathit{\tilde{T}}(\mathdutchcal{e}_{ijk}) \Vert \\
                        \leq & \frac{1-p}{p} \Vert \tilde{\mathbf{T}} \Vert_{\infty}^2 \max_{j} \sum_{ik}\vert \mathcal{Z}_{ijk} \vert^2 \nonumber \\
                        = & \frac{1-p}{p} \Vert \tilde{\mathbf{T}} \Vert_{\infty}^2 \Vert \mathcal{Z}_{ijk} \Vert_{1\mapsto 2}^2 \nonumber
                    \end{align}
                    \begin{align}
                        \Vert \sum_{ijk} \mathbb{E}[\texttt{bdiag}(\mathcal{L}_{ijk})\texttt{bdiag}^{\text{H}}(\mathcal{L}_{ijk})]\Vert & \leq \frac{1-p}{p} \Vert \tilde{\mathbf{T}} \Vert_{\infty}^2 \Vert \mathcal{Z}_{ijk} \Vert_{2\mapsto \infty}^2
                    \end{align}
                    
                    Use the extension of theorem \eqref{theorem:Bernstein}:
                    \begin{equation}
                        \begin{aligned}
                            & \Vert \mathit{\tilde{T}}(\frac{1}{p} \mathcal{S}_{\Omega} - \mathcal{I})\mathcal{Z} \Vert = \Vert \sum_{ijk} \mathcal{L}_{ijk} \Vert = \Vert \sum_{ijk} \texttt{bdiag}(\mathcal{L}_{ijk}) \Vert \\
                            & \leq 2\sqrt{c\frac{1}{p} \log((n_1 + n_2)N_3)} \Vert \tilde{\mathbf{T}} \Vert_{\infty} \Vert \mathcal{Z} \Vert_{\infty, 2} + c \frac{1}{p} \log((n_1 + n_2)N_3) \Vert \tilde{\mathbf{T}} \Vert_{\infty} \Vert \mathcal{Z} \Vert_{\infty}
                        \end{aligned}
                    \end{equation}
                    holds with probability at least $1 - (n_1 + n_2)^{1 - c}$, $\forall c > 0$.
                \end{proof}

                \subsection{Proof of Lemma~\ref{lemma:bounding_infinity_norm}}
                \label{proof:bounding_infinity_norm}
                \begin{proof}
                    First we decompose the operator:
                    \begin{equation}
                        \begin{aligned}
                            \mathit{T}^{\text{H}}\mathcal{P}_{\mathbb{S}}\mathit{\tilde{T}}(\mathcal{I} - \frac{1}{p}\mathcal{S}_{\Omega})\mathcal{Z} & = \mathit{T}^{\text{H}}\mathcal{P}_{\mathbb{S}}\mathit{\tilde{T}}\sum_{ijk}(1 - \frac{\delta_{ijk}}{p})\mathcal{Z}_{ijk}\mathdutchcal{e}_{ijk} \\
                            & = \sum_{ijk}(1 - \frac{\delta_{ijk}}{p})\mathcal{Z}_{ijk}\mathit{T}^{\text{H}}\mathcal{P}_{\mathbb{S}}\mathit{\tilde{T}}(\mathdutchcal{e}_{ijk}) := \sum_{ijk} \mathcal{A}_{ijk}.
                        \end{aligned}
                    \end{equation}
        
                    Each entry of $\mathcal{A}_{ijk}$ can be extracted using the inner-product with $\mathdutchcal{e}_{abc}$:
                    \begin{align}
                        \mathcal{A}_{ijk, abc} & = \langle (1 - \frac{\delta_{ijk}}{p})\mathcal{Z}_{ijk}\mathit{T}^{\text{H}}\mathcal{P}_{\mathbb{S}}\mathit{\tilde{T}}(\mathdutchcal{e}_{ijk}), \mathdutchcal{e}_{abc} \rangle \nonumber \\
                        & = (1 - \frac{\delta_{ijk}}{p})\mathcal{Z}_{ijk} \langle \mathcal{P}_{\mathbb{S}}\mathit{\tilde{T}}(\mathdutchcal{e}_{ijk}), \mathcal{P}_{\mathbb{S}}\mathit{T}(\mathdutchcal{e}_{abc}) \rangle.\nonumber
                    \end{align}
    
                    \begin{equation}
                        \begin{aligned}
                            \vert \mathcal{A}_{ijk, abc} \vert & = \vert (1 - \frac{\delta_{ijk}}{p})\mathcal{Z}_{ijk} \langle \mathcal{P}_{\mathbb{S}}\mathit{\tilde{T}}(\mathdutchcal{e}_{ijk}), \mathcal{P}_{\mathbb{S}}\mathit{T}(\mathdutchcal{e}_{abc}) \vert \\
                            & \leq \frac{1}{p} \Vert \mathcal{Z} \Vert_{\infty} \Vert \mathcal{P}_{\mathbb{S}}\mathit{\tilde{T}}(\mathdutchcal{e}_{ijk}) \Vert_{\text{F}} \Vert \mathcal{P}_{\mathbb{S}}\mathit{T}(\mathdutchcal{e}_{abc}) \Vert_{\text{F}}.
                        \end{aligned}
                    \end{equation}
    
                    \begin{equation}
                        \begin{aligned}
                            \vert \sum_{ijk} \mathbb{E}\left[ \vert\mathcal{A}_{ijk, abc}\vert^2 \right] & = \vert \sum_{ijk} \mathbb{E}\left[ (1 - \frac{\delta_{ijk}}{p})^2 \vert\mathcal{Z}_{ijk}\vert^2 \vert\langle \mathcal{P}_{\mathbb{S}}\mathit{\tilde{T}}(\mathdutchcal{e}_{ijk}), \mathcal{P}_{\mathbb{S}}\mathit{T}(\mathdutchcal{e}_{abc})\rangle\vert^2 \right] \\
                            & = \sum_{ijk} \frac{1-p}{p} \vert\mathcal{Z}_{ijk}\vert^2 \vert\langle \mathcal{P}_{\mathbb{S}}\mathit{\tilde{T}}(\mathdutchcal{e}_{ijk}), \mathcal{P}_{\mathbb{S}}\mathit{T}(\mathdutchcal{e}_{abc})\rangle\vert^2 \\
                            & \leq \frac{1-p}{p} \Vert \mathcal{Z} \Vert_{\infty}^2 \Vert\mathit{T}^{\dag}\mathcal{P}_{\mathbb{S}}\mathit{T}(\mathdutchcal{e}_{abc})\Vert_{\text{F}}^2 \\
                            & \leq \frac{1-p}{p} \Vert \mathcal{Z} \Vert_{\infty}^2\Vert \mathbf{T}^{\dag}\Vert^2\Vert \mathcal{P}_{\mathbb{S}}\mathit{T}(\mathdutchcal{e}_{abc})\Vert_{\text{F}}^2
                        \end{aligned}
                    \end{equation}
                    
                    Using Bernstein inequality:
                    \begin{equation}
                        \begin{aligned}
                            & \mathbb{P}\left[ \vert \left( \mathit{T}^{\text{H}}\mathcal{P}_{\mathbb{S}}\mathit{\tilde{T}}(\mathcal{I} - \frac{1}{p}\mathcal{S}_{\Omega})\mathcal{Z} \right)_{abc} \vert \geq \frac{1}{2} \Vert \mathcal{Z} \Vert_{\infty} \right] = \mathbb{P}\left[ \vert \sum_{ijk} \mathcal{A}_{ijk, abc} \vert \geq \frac{1}{2} \Vert \mathcal{Z} \Vert_{\infty} \right] \\
                            & \leq 2\exp\left(- \frac{\frac{1}{4}\Vert \mathcal{Z} \Vert_{\infty}^2}{\frac{2}{p} \Vert \mathcal{Z} \Vert_{\infty}^2 \Vert \mathit{T}^{\dag}\mathcal{P}_{\mathbb{S}}\mathit{T}(\mathdutchcal{e}_{abc})\Vert_{\text{F}}^2 + \frac{2}{3p}\Vert \mathcal{Z} \Vert_{\infty}^2 \Vert\mathcal{P}_{\mathbb{S}}\mathit{T}(\mathdutchcal{e}_{abc})\Vert_{\text{F}}\cdot\Vert\mathcal{P}_{\mathbb{S}}\mathit{\tilde{T}}(\mathdutchcal{e}_{abc})\Vert_{\text{F}}}\right) \\
                            & = 2\exp\left(- \frac{p}{8 \Vert \mathit{T}^{\dag}\mathcal{P}_{\mathbb{S}}\mathit{T}(\mathdutchcal{e}_{abc})\Vert_{\text{F}}^2 + \frac{8}{3}\Vert\mathcal{P}_{\mathbb{S}}\mathit{T}(\mathdutchcal{e}_{abc})\Vert_{\text{F}}\cdot\Vert\mathcal{P}_{\mathbb{S}}\mathit{\tilde{T}}(\mathdutchcal{e}_{abc})\Vert_{\text{F}}}\right) \\
                            & \leq 2\exp\left(- \frac{c_0 \kappa^2(\mathbf{T})\frac{\nu r (n_1 + n_2)}{n_1 n_2}\log((n_1 + n_2)N_3)}{8\kappa^2(\mathbf{T})\frac{\nu r (n_1 + n_2)}{n_1 n_2} + \frac{8}{3}\kappa^2(\mathbf{T})\frac{\nu r (n_1 + n_2)}{n_1 n_2}}\right) \\
                            & = 2((n_1 + n_2)N_3)^{-\frac{3c_0}{32}}.
                        \end{aligned}
                        \label{eq:prob_bounding_infinity_norm}
                    \end{equation}
                    By union bound, $\Vert \mathit{T}^{\text{H}}\mathcal{P}_{\mathbb{S}}\mathit{\tilde{T}}(\mathcal{I} - \frac{1}{p}\mathcal{S}_{\Omega})\mathcal{Z} \Vert_{\infty} \leq \Vert \mathcal{Z} \Vert_{\infty}$ holds with probability at least $1 - 2(n_1 + n_2)^{2-\frac{3c_0}{32}}N_3^{1-\frac{3c_0}{32}}$.
                \end{proof}
    
                \subsection{Proof of Lemma~\ref{lemma:bounding_infinity_2_norm}}
                \label{proof:bounding_infinity_2_norm}
                \begin{proof}
                    We decompose the operator as 
                    \begin{equation}
                        \begin{aligned}
                            &(\mathit{T}^{\text{H}} \mathcal{P}_{\mathbb{S}} \mathit{\tilde{T}}(\mathcal{I} - \frac{1}{p}\mathcal{S}_{\Omega}) \mathcal{Z}) \star \hat{\xi_b} \\
                            =& \sum_{ijk} (1 - \frac{\delta_{ijk}}{p}) \mathcal{Z}_{ijk} (\mathit{T}^{\text{H}} \mathcal{P}_{\mathbb{S}} \mathit{\tilde{T}}(\mathdutchcal{e}_{ijk})) \star \hat{\xi_b} := \sum_{ijk} \mathcal{C}_{ijk} \in \mathbb{C}^{n_1 \times 1 \times n_3}, 
                        \end{aligned}
                    \end{equation}
                    where $\hat{\xi_b} \in \mathbb{C}^{n_1 \times 1 \times n_3}$ is the tensor column basis in the original domain with $\hat{\xi_b}(1, b, :) = 1$ and equaling $0$ otherwise. Since $\Vert \mathcal{C}_{ijk} \Vert_{\text{F}} = \Vert \texttt{vec}(\mathcal{C}_{ijk}) \Vert_{2} = \Vert \texttt{vec}(\mathcal{C}_{ijk}) \Vert$, it is equipped with the form to apply Theorem \ref{theorem:Bernstein}.
                    \begin{equation}
                        \begin{aligned}
                            & \Vert \texttt{vec}(\mathcal{C}_{ijk}) \Vert_2 = \Vert \mathcal{C}_{ijk} \Vert_{\text{F}} \\
                            & \leq \frac{1}{p} \Vert \mathcal{Z} \Vert_{\infty} \Vert (\mathit{T}^{\text{H}} \mathcal{P}_{\mathbb{S}} \mathit{\tilde{T}}(\mathdutchcal{e}_{ijk})) \star \hat{\xi_b} \Vert_{\text{F}} \leq \frac{1}{p} \Vert \mathcal{Z} \Vert_{\infty} \Vert \mathit{T}^{\text{H}} \mathcal{P}_{\mathbb{S}} \mathit{\tilde{T}}(\mathdutchcal{e}_{ijk})\Vert_{\text{F}}.
                        \end{aligned}
                    \end{equation}
                
                    \begin{equation}
                        \begin{aligned}
                            & \vert \mathbb{E}[ \sum_{ijk} \texttt{vec}(\mathcal{C}_{ijk})^{\text{H}} \texttt{vec}(\mathcal{C}_{ijk})] \vert = \vert \mathbb{E}[ \sum_{ijk} \Vert \mathcal{C}_{ijk} \Vert_{\text{F}}^2] \vert \\
                            = & \frac{1 - p}{p} \sum_{ijk} \vert \mathcal{Z}_{ijk} \vert^2 \Vert (\mathit{T}^{\text{H}} \mathcal{P}_{\mathbb{S}} \mathit{\tilde{T}}(\mathdutchcal{e}_{ijk})) \star \hat{\xi_b} \Vert_{\text{F}}^2 \label{eq:using_Bernstein_bounding_infinity_2_norm}
                        \end{aligned}
                    \end{equation}
    
                    Looking into the term in \eqref{eq:using_Bernstein_bounding_infinity_2_norm}:
                    \begin{align}
                        \Vert (\mathit{T}^{\text{H}} \mathcal{P}_{\mathbb{S}} \mathit{\tilde{T}}(\mathdutchcal{e}_{ijk})) \star \hat{\xi_b}) \Vert_{\text{F}} \leq \Vert(\mathit{T}^{\text{H}} (\mathcal{P}_{\mathcal{U}}\star \mathit{\tilde{T}}(\mathdutchcal{e}_{ijk}))) \star \hat{\xi_b} \Vert_{\text{F}} + \Vert(\mathit{T}^{\text{H}} (\mathcal{P}_{\mathcal{U}^{\perp}}\star \mathit{\tilde{T}}(\mathdutchcal{e}_{ijk}) \star \mathcal{P}_{\mathcal{V}}) \star \hat{\xi_b} \Vert_{\text{F}}, 
                    \end{align}
    
                    yields
                    \begin{align}
                        \Vert(\mathit{T}^{\text{H}} (\mathcal{P}_{\mathcal{U}}\star \mathit{\tilde{T}}(\mathdutchcal{e}_{ijk}))) \star \hat{\xi_b} \Vert_{\text{F}} = \left\{
                            \begin{aligned}
                                & \Vert\mathit{T}^{\text{H}} (\mathcal{P}_{\mathcal{U}}\star \mathit{\tilde{T}}(e_{ibk})) \Vert_{\text{F}} &,~j = b \\
                                & 0 &,~j \neq b
                            \end{aligned}
                        \right.
                    \end{align}
        
                    \begin{align}
                        \Vert(\mathit{T}^{\text{H}} (\mathcal{P}_{\mathcal{U}^{\perp}}\star \mathit{\tilde{T}}(\mathdutchcal{e}_{ijk}) \star \mathcal{P}_{\mathcal{V}}) \star \hat{\xi_b} \Vert_{\text{F}} \leq \Vert \mathbf{T}^{\text{H}} \Vert \cdot \Vert \mathbf{\tilde{T}}\Vert_{\infty} \Vert \xi_j^{\text{H}} \star \mathcal{P}_{\mathcal{V}} \star \xi_b \Vert_{\text{F}}.
                    \end{align}
        
                    Therefore
                    \begin{equation}
                        \begin{aligned}
                            & \vert \mathbb{E}[ \sum_{ijk} \texttt{vec}(\mathcal{C}_{ijk})^{\text{H}} \texttt{vec}(\mathcal{C}_{ijk})] \vert = \vert \mathbb{E}[ \sum_{ijk} \Vert \mathcal{C}_{ijk} \Vert_{\text{F}}^2] \vert \\
                            \leq & \frac{2}{p} \sum_{ik} \vert \mathcal{Z}_{ibk} \vert^2 \Vert\mathit{T}^{\text{H}} (\mathcal{P}_{\mathcal{U}}\star \mathit{\tilde{T}}(e_{ibk})) \Vert_{\text{F}}^2 + \frac{2}{p} \Vert \mathbf{T}^{\text{H}} \Vert^2\cdot\Vert \tilde{\mathbf{T}}\Vert_{\infty}^2 \sum_{ijk} \vert \mathcal{Z}_{ijk} \vert^2 \cdot \Vert \xi_j^{\text{H}} \star \mathcal{P}_{\mathcal{V}} \star \xi_b \Vert_{\text{F}}^2 \\
                            \leq & \frac{2}{p} \max_{ijk} \Vert\mathit{T}^{\text{H}}\mathcal{P}_{\mathcal{U}}\mathit{\tilde{T}}(\mathdutchcal{e}_{ijk}) \Vert_{\text{F}}^2\sum_{ik} \vert \mathcal{Z}_{ibk} \vert^2 + \frac{2}{p} \Vert \mathbf{T}^{\text{H}} \Vert^2 \cdot\Vert \tilde{\mathbf{T}}\Vert_{\infty}^2 \sum_{ijk} \vert \mathcal{Z}_{ijk} \vert^2 \cdot \Vert \xi_j^{\text{H}} \star \mathcal{P}_{\mathcal{V}} \star \xi_b \Vert_{\text{F}}^2 \\
                            = & \frac{2}{p} \Vert \mathcal{Z} \Vert_{\infty, 2}^2 \max_{ijk} \Vert\mathit{T}^{\text{H}}\mathcal{P}_{\mathcal{U}}\mathit{\tilde{T}}(\mathdutchcal{e}_{ijk}) \Vert_{\text{F}}^2 + \frac{2}{p} \Vert \mathbf{T}^{\text{H}} \Vert^2\cdot\Vert \tilde{\mathbf{T}}\Vert_{\infty}^2 \sum_{j} \Vert \xi_j^{\text{H}} \star \mathcal{P}_{\mathcal{V}} \star \xi_b \Vert_{\text{F}}^2 \sum_{ik} \vert \mathcal{Z}_{ijk} \vert^2 \\
                            \leq & \frac{2}{p} (\max_{ijk} \Vert\mathit{T}^{\text{H}}\mathcal{P}_{\mathcal{U}}\mathit{\tilde{T}}(\mathdutchcal{e}_{ijk}) \Vert_{\text{F}}^2 + \Vert \mathbf{T}^{\text{H}} \Vert^2\cdot\Vert \tilde{\mathbf{T}}\Vert_{\infty}^2 \Vert \mathcal{P}_{\mathcal{V}} \star \xi_b \Vert_{\text{F}}^2)\cdot\Vert \mathcal{Z} \Vert_{\infty, 2}^2 \\
                            \leq & \frac{2}{p} (\kappa^2(\mathbf{T})\frac{\nu r}{n_1} + \Vert \mathbf{T} \Vert^2 \rho(\mathbf{T})\frac{\mu r}{n_2})\Vert \mathcal{Z} \Vert_{\infty, 2}^2.
                        \end{aligned}
                    \end{equation}
        
                    Since $\Vert \mathbb{E}[ \sum_{ijk} \texttt{vec}(\mathcal{C}_{ijk}) \texttt{vec}(\mathcal{C}_{ijk})^{\text{H}}] \Vert$ is the spectral norm of a matrix. When $\{\texttt{vec}(\mathcal{C}_{ijk})\}_{ijk}$ share the same basis the spectral norm takes the maximum, thus $\Vert \mathbb{E}[ \sum_{ijk} \texttt{vec}(\mathcal{C}_{ijk}) \texttt{vec}(\mathcal{C}_{ijk})^{\text{H}}] \Vert \leq \vert \mathbb{E}[ \sum_{ijk} \texttt{vec}(\mathcal{C}_{ijk})^{\text{H}} \texttt{vec}(\mathcal{C}_{ijk})]$.

                    Using the extension of Theorem \ref{theorem:Bernstein}:
                    \begin{equation}
                        \begin{aligned}
                            & \Vert \mathit{T}^{\text{H}} \mathcal{P}_{\mathbb{S}} \mathit{\tilde{T}}(\mathcal{I} - \frac{1}{p}\mathcal{S}_{\Omega}) \mathcal{Z} \Vert_{\infty, 2} \\ 
                            % \leq & 2\sqrt{c\sigma^2 \log(2\max(n_1, n_2)n_3)} + c R \log(2\max(n_1, n_2)n_3) \\
                            = & 2c\sqrt{\frac{2(\kappa^2(\mathbf{T})\frac{\nu r}{n_1} + \Vert \mathbf{T} \Vert^2 \rho(\mathbf{T})\frac{\mu r}{n_2})\cdot\log(2\max(n_1, n_2)n_3)}{p}}\Vert \mathcal{Z} \Vert_{\infty, 2} \\
                            + & c \frac{1}{p} \kappa(\mathbf{T})\sqrt{\frac{\nu r (n_1 + n_2)}{n_1 n_2}}\log(2\max(n_1, n_2)n_3)\Vert \mathcal{Z} \Vert_{\infty}
                        \end{aligned}
                    \end{equation}
                    holds with probability at least $1 - (2(n_1 + n_2)n_3)^{1-c}, \forall c > 0$.
        
                    The norm of tensor row can be bounded in the same fashion. Thus if $p \geq c_0 (\kappa^2(\mathbf{T}) + \rho^2(\mathbf{T}))\frac{\lambda r (n_1 + n_2)}{n_1 n_2}\log((n_1 + n_2)N_3)$ with $c$ large enough,
                    \begin{equation}
                        \begin{aligned}
                            \Vert \mathit{T}^{\text{H}} \mathcal{P}_{\mathbb{S}} \mathit{\tilde{T}}(\mathcal{I} - \frac{1}{p}\mathcal{S}_{\Omega}) \mathcal{Z} \Vert_{\infty, 2} & \leq 2c \sqrt{\frac{4}{c_0}}\Vert \mathcal{Z} \Vert_{\infty, 2} + \frac{c}{c_0}\sqrt{\frac{n_1 n_2}{\mu r (n_1 + n_2)}}\Vert \mathcal{Z} \Vert_{\infty} \\ & \leq \frac{1}{2} \Vert \mathcal{Z} \Vert_{\infty, 2} + \frac{1}{2} \sqrt{\frac{n_1 n_2}{\mu r (n_1 + n_2)}}\Vert \mathcal{Z} \Vert_{\infty}
                        \end{aligned}
                    \end{equation}
                    holds w.h.p..
                \end{proof}
        
            \section{Bound Frobenius Norm by Incoherence Condition}
                \label{section:bound_F_norm}
                Now let us bound $\max_{ijk} \Vert \mathcal{P}_\mathbb{S}\mathit{T}(\mathdutchcal{e}_{ijk}) \Vert_{\text{F}}^2$ and $\max_{ijk} \Vert \mathcal{P}_\mathbb{S}\mathit{\tilde{T}}(\mathdutchcal{e}_{ijk}) \Vert_{\text{F}}^2$
                \begin{align}
                    \mathcal{P}_\mathbb{S}\mathit{T}(\mathdutchcal{e}_{ijk}) = \mathcal{P}_{\mathcal{U}} \star \mathit{T}(\xi_i \star \zeta_k \star \xi_j^{\text{H}}) + \mathit{T}(\xi_i \star \zeta_k \star \xi_j^{\text{H}}) \star \mathcal{P}_{\mathcal{V}} - \mathcal{P}_{\mathcal{U}} \star \mathit{T}(\xi_i \star \zeta_k \star \xi_j^{\text{H}}) \star \mathcal{P}_{\mathcal{V}} 
                \end{align}
    
                \begin{equation}
                    \begin{aligned}
                        & \Vert \mathcal{P}_\mathbb{S}\mathit{T}(\mathdutchcal{e}_{ijk}) \Vert_{\text{F}}^2 \\
                        = & \langle \mathcal{P}_{\mathcal{U}} \star \mathit{T}(\xi_i \star \zeta_k \star \xi_j^{\text{H}}), \mathit{T}(\xi_i \star \zeta_k \star \xi_j^{\text{H}}) \rangle + \langle \mathit{T}(\xi_i \star \zeta_k \star \xi_j^{\text{H}}) \star \mathcal{P}_{\mathcal{V}}, \mathit{T}(\xi_i \star \zeta_k \star \xi_j^{\text{H}}) \rangle \\
                        - & \langle \mathcal{P}_{\mathcal{U}} \star \mathit{T}(\xi_i \star \zeta_k \star \xi_j^{\text{H}}) \star \mathcal{P}_{\mathcal{V}}, \mathit{T}(\xi_i \star \zeta_k \star \xi_j^{\text{H}}) \rangle
                    \end{aligned}
                \end{equation}
    
                \begin{equation}
                    \begin{aligned}
                        \langle \mathcal{P}_{\mathcal{U}} \star \mathit{T}(\xi_i \star \zeta_k \star \xi_j^{\text{H}}), \mathit{T}(\xi_i \star \zeta_k \star \xi_j^{\text{H}}) \rangle
                        = & \langle \mathcal{U}^{\text{H}} \star \mathit{T}(\xi_i \star \zeta_k \star \xi_j^{\text{H}}), \mathcal{U}^{\text{H}} \star \mathit{T}(\xi_i \star \zeta_k \star \xi_j^{\text{H}}) \rangle \\
                        % = & \Vert \mathcal{U}^{\text{H}} \star \mathit{T}(\xi_i \star \zeta_k \star \xi_j^{\text{H}}) \Vert_{\text{F}}^2 \\
                        = & \Vert \mathcal{U}^{\text{H}} \star \xi_i \star \mathit{T}(\zeta_k) \Vert_{\text{F}}^2
                    \end{aligned}
                \end{equation}
        
                In the same way, 
                \begin{equation}
                    \begin{aligned}
                        & \langle \mathit{T}(\xi_i \star \zeta_k \star \xi_j^{\text{H}}) \star \mathcal{P}_{\mathcal{V}}, \mathit{T}(\xi_i \star \zeta_k \star \xi_j^{\text{H}}) \rangle \\
                        % = & \langle \mathit{T}(\xi_i \star \zeta_k \star \xi_j^{\text{H}}) \star \mathcal{V}, \mathit{T}(\xi_i \star \zeta_k \star \xi_j^{\text{H}}) \star \mathcal{V}\rangle \\
                        % = & \Vert \mathit{T}(\xi_i \star \zeta_k \star \xi_j^{\text{H}}) \star \mathcal{V} \Vert_{\text{F}}^2 = \Vert \mathcal{V}^{\text{H}} \star \mathit{T}^{*}(\xi_j \star \xi_i^{\text{H}} \star \zeta_k) \Vert_{\text{F}}^2 \\
                        = & \Vert \mathcal{V}^{\text{H}} \star \xi_j \star \mathit{T}(\zeta_k) \Vert_{\text{F}}^2~(\text{here the conjugate does not change the value}).
                    \end{aligned}
                \end{equation}

                And we have
                \begin{align}
                    \Vert \mathcal{P}_\mathbb{S}\mathit{T}(\mathdutchcal{e}_{ijk}) \Vert_{\text{F}}^2 \leq \Vert \mathcal{U}^{\text{H}} \star \xi_i \star \mathit{T}(\zeta_k) \Vert_{\text{F}}^2 + \Vert \mathcal{V}^{\text{H}} \star \xi_j \star \mathit{T}(\zeta_k) \Vert_{\text{F}}^2.
                \end{align}
                When the energy of $\mathcal{U}(:, :, k)$ is uniformly distributed among its entries, $\max_{ijk} \Vert \mathcal{U}^{\text{H}} \star \xi_i \star \mathit{T}(\zeta_k) \Vert_{\text{F}}^2$ takes the minimum as $\frac{r}{n_1} \Vert \mathbf{T} \Vert_{1\mapsto2}^2$; When $\mathcal{U}(:, :, k)$ contains an identity tensor, i.e., the most non-uniformly distributed, $\max_{ijk} \Vert \mathcal{U}^{\text{H}} \star \xi_i \star \mathit{T}(\zeta_k) \Vert_{\text{F}}^2$ takes the maximum as $\Vert \mathbf{T} \Vert_{1\mapsto2}^2$. Applying tensor incoherence conditions: 
                \begin{align}
                    & \Vert \mathcal{U}^{\text{H}} \star \xi_i \star \mathit{T}(\zeta_k) \Vert_{\text{F}}^2 \leq \frac{\nu r}{n_1} \Vert \mathbf{T} \Vert_{1\mapsto2}^2 \nonumber \\
                    & \Vert \mathcal{V}^{\text{H}} \star \xi_j \star \mathit{T}(\zeta_k) \Vert_{\text{F}}^2 \leq \frac{\nu r}{n_2} \Vert \mathbf{T} \Vert_{1\mapsto2}^2 \nonumber
                \end{align}
    
                And we have:
                \begin{align}
                    \Vert \mathcal{P}_\mathbb{S}\mathit{T}(\mathdutchcal{e}_{ijk}) \Vert_{\text{F}}^2 & \leq \frac{\nu r}{n_1}\Vert \mathbf{T} \Vert_{1\mapsto2}^2 + \frac{\nu r}{n_2}\Vert \mathbf{T} \Vert_{1\mapsto2}^2 = \frac{\nu r(n_1 + n_2)}{n_1 n_2} \Vert \mathbf{T} \Vert_{1\mapsto2}^2\\
                    \Vert \mathcal{P}_\mathbb{S}\mathit{\tilde{T}}(\mathdutchcal{e}_{ijk}) \Vert_{\text{F}}^2 & \leq \frac{\nu r (n_1 + n_2)}{n_1 n_2} \Vert \tilde{\mathbf{T}} \Vert_{1\mapsto2}^2
                \end{align}

        \section{Additional Experiment Results on Video Data}
            \label{section:exp_vid}
            \begin{figure}[!h]
                \begin{subfigure}{0.45\textwidth}
                    \centering
                    \includegraphics[width=0.96\textwidth,height=0.6\textwidth]{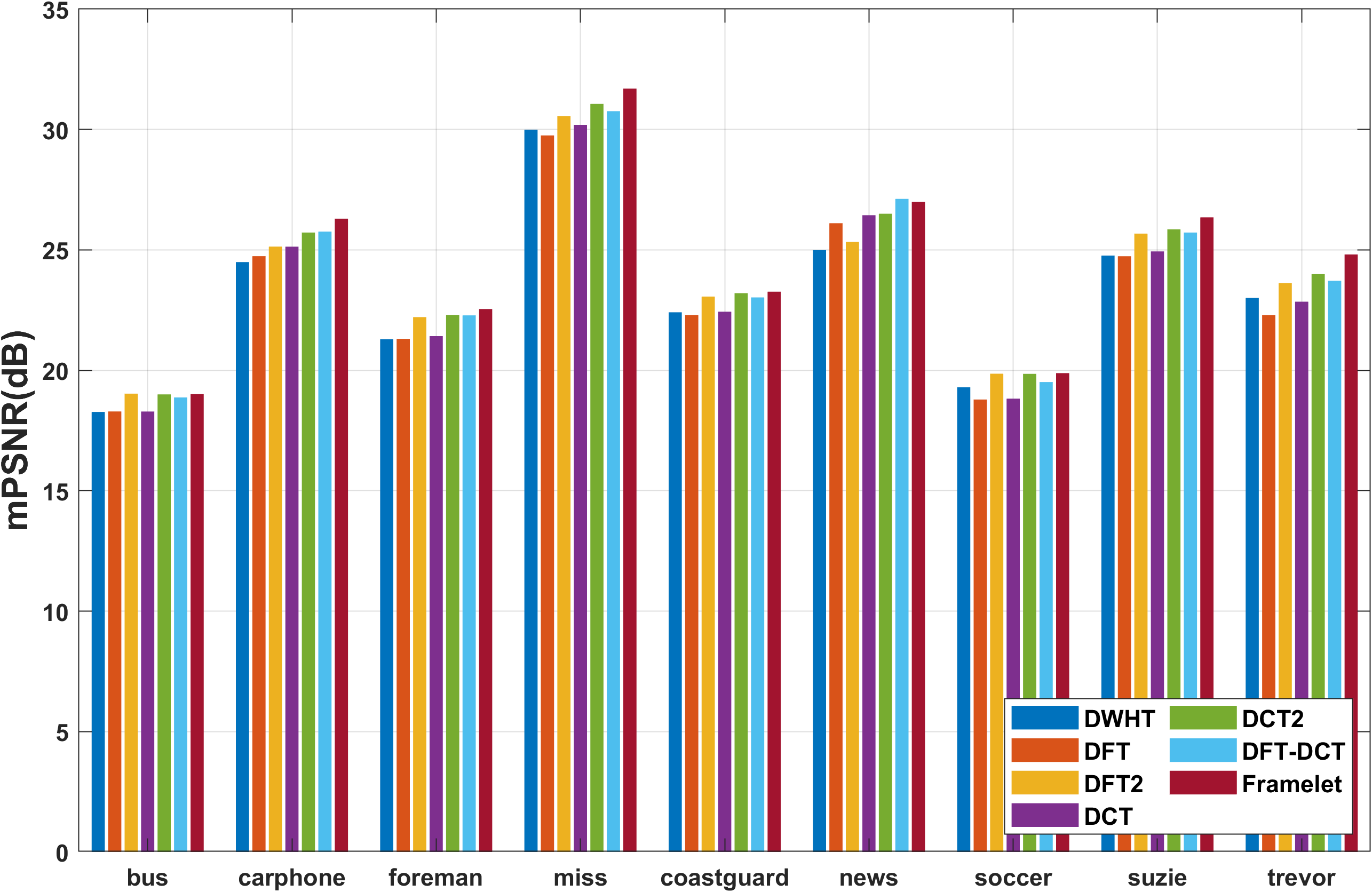}
                    %\caption{\scriptsize mPSNR}
                    \label{fig:yuv_PSNR}
                \end{subfigure}
                \hfill
                \begin{subfigure}{0.45\textwidth}
                    \centering
                    \includegraphics[width=0.96\textwidth,height=0.6\textwidth]{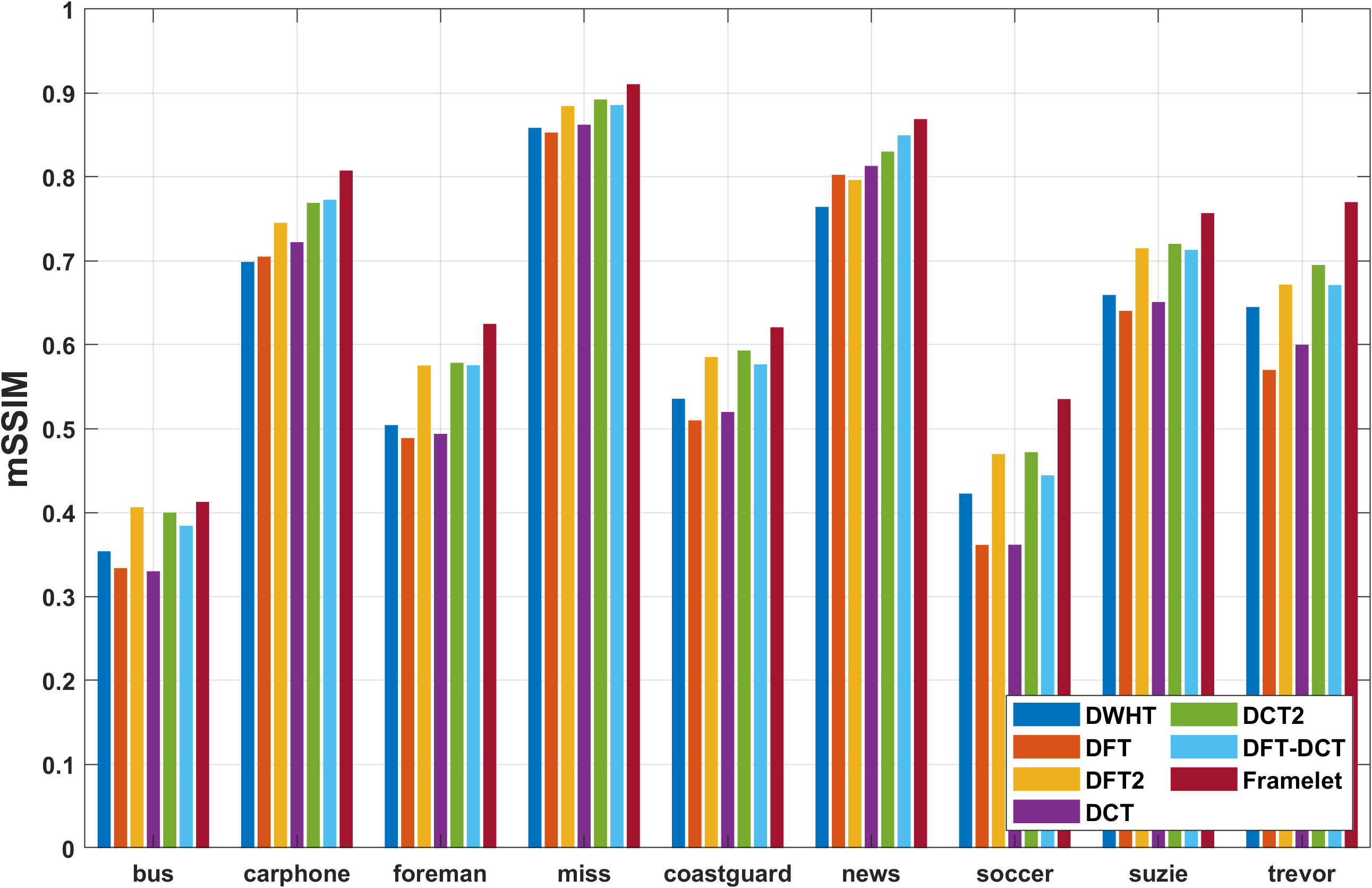}
                    %\caption{\scriptsize SSIM}
                    \label{fig:yuv_SSIM}
                \end{subfigure}
                \caption{\scriptsize mPSNR (left) and mSSIM (right) results on video data with SR=0.05.}
                \label{fig:yuv_ind_p=0.05}
            \end{figure}

        \section{Analyzing the Overall Sampling Rate}
            \label{section:ananlyzing_sampling_rate}
            To arrive at a relatively concise form of Theorem \ref{theorem:exact_completion}, several Scaling techniques have been used (e.g. the introduction of incoherence parameters), which partially hinders our insights to see which terms are affecting the sampling rate $p$. In this section, removing redundant (while formally useful) scalings, we directly study the functional terms through the construction of the dual certificate \ref{section:constructing_dual_certificate}.

            From the proof of Proposition \ref{proposition:constructing_dual_certificate}, one can observe that the satisfaction of the conditions of dual certificate \eqref{eq:cond:approximate_isometry}, \eqref{eq:cond:inner_product_condition}, \eqref{eq:cond:low_distortion} and \eqref{eq:cond:spectral_norm_cond} put requirements on the sampling rate $p$. Among them, \eqref{eq:cond:inner_product_condition} is satisfied by the structure of the dual certificate. \eqref{eq:cond:approximate_isometry} and \eqref{eq:cond:low_distortion} are essentially the same, only that \eqref{eq:cond:low_distortion} is harsher since it is upon the step sampling rate $p_i$ ($p = l\cdot p_i$). Thus, we only need to study the requirements posed by \eqref{eq:cond:low_distortion} and \eqref{eq:cond:spectral_norm_cond}.
            \begin{itemize}
                \item To satisfy \eqref{eq:cond:low_distortion}: by \eqref{eq:prob_approximate_isometry}, one can see that $p_i$ is propotional to the energy terms:
                \begin{align*}
                    & p_i \sim \kappa(\mathbf{T}) \cdot\max_{ijk}\Vert\mathcal{P}_\mathbb{S}\mathit{T}(\mathdutchcal{e}_{ijk})\Vert_{\text{F}}^2 \vee \Vert\mathcal{P}_\mathbb{S}\tilde{\mathit{T}}(\mathdutchcal{e}_{ijk})\Vert_{\text{F}}^2, \\
                    & p_i \sim \max_{ijk}\Vert\mathcal{P}_\mathbb{S}\mathit{\tilde{T}}(\mathdutchcal{e}_{ijk})\Vert_{\text{F}}\cdot\max_{ijk}\Vert\mathcal{P}_\mathbb{S}\mathit{T}(\mathdutchcal{e}_{ijk})\Vert_{\text{F}}.
                \end{align*}
                Since $\max_{ijk}\Vert\mathcal{P}_\mathbb{S}\mathit{\tilde{T}}(\mathdutchcal{e}_{ijk})\Vert_{\text{F}}\cdot\max_{ijk}\Vert\mathcal{P}_\mathbb{S}\mathit{T}(\mathdutchcal{e}_{ijk})\Vert_{\text{F}} \leq \frac{1}{2} (\max_{ijk}\Vert\mathcal{P}_\mathbb{S}\mathit{\tilde{T}}(\mathdutchcal{e}_{ijk})\Vert_{\text{F}}^2 + \max_{ijk}\Vert\mathcal{P}_\mathbb{S}\mathit{T}(\mathdutchcal{e}_{ijk})\Vert_{\text{F}}^2)$, thus the satisfaction of \eqref{eq:cond:low_distortion} only requires
                \begin{align}
                    p_i \sim \kappa(\mathbf{T}) \cdot\max_{ijk}\Vert\mathcal{P}_\mathbb{S}\mathit{T}(\mathdutchcal{e}_{ijk})\Vert_{\text{F}}^2 \vee \Vert\mathcal{P}_\mathbb{S}\tilde{\mathit{T}}(\mathdutchcal{e}_{ijk})\Vert_{\text{F}}^2.
                \end{align}

                \item To satisfy \eqref{eq:cond:spectral_norm_cond}: there are three lemmas used when proving \eqref{eq:cond:spectral_norm_cond}. We need to study the requirements posed by the assumptions of Lemma \eqref{lemma:bounding_infinity_norm} and \eqref{lemma:bounding_infinity_2_norm} (Lemma \eqref{lemma:bounding_spectral_norm} already holds).
                \begin{itemize}
                    \item To satisfy the assumption of Lemma \eqref{lemma:bounding_infinity_norm}: by \eqref{eq:prob_bounding_infinity_norm}, we can observe that:
                    \begin{align}
                        & p_i \sim \max_{ijk}\Vert\mathit{T}^{\dag}\mathcal{P}_\mathbb{S}\mathit{T}(\mathdutchcal{e}_{ijk})\Vert_{\text{F}}^2, \\
                        & p_i \sim \max_{ijk}\Vert\mathcal{P}_\mathbb{S}\mathit{\tilde{T}}(\mathdutchcal{e}_{ijk})\Vert_{\text{F}}\cdot\max_{ijk}\Vert\mathcal{P}_\mathbb{S}\mathit{T}(\mathdutchcal{e}_{ijk})\Vert_{\text{F}}.
                    \end{align}

                    \item To satisfy the assumption of Lemma \eqref{lemma:bounding_infinity_2_norm}: there are two terms affecting $p_i$ in this lemma: $\max_{ijk}\Vert\mathit{T}^{\text{H}}\mathcal{P}_\mathcal{U}\tilde{\mathit{T}}(\mathdutchcal{e}_{ijk})\Vert_{\text{F}}^2$ and $\Vert\mathbf{T}\Vert^2\Vert\tilde{\mathbf{T}}\Vert_{\infty}^2\Vert\mathcal{P}_{\mathcal{V}}\star\xi_b\Vert_{\text{F}}^2$, both coming from $\max_{ijk}\Vert\mathit{T}^{\text{H}}\mathcal{P}_\mathcal{U}\tilde{\mathit{T}}(\mathdutchcal{e}_{ijk})\Vert_{\text{F}}^2$. The reason why for slim transforms the first term is smaller is the same as illustrated in Section \ref{section:brief_analysis}. For the second term, which has been scaled to introduce incoherence parameters, since we require no structures on $\mathbf{T}$, it can not be simplified further to offer more insights. However, we can still directly analyze it: for a combinational transform $\mathbf{T}_{\text{c}} = \frac{1}{\sqrt{n}}[\mathbf{T}_1^{\text{T}}, \mathbf{T}_2^{\text{T}}, \hdots, \mathbf{T}_M^{\text{T}}]^{\text{T}}$, $\Vert\mathbf{T}_{\text{c}}\Vert^2\Vert\tilde{\mathbf{T}}_{\text{c}}\Vert_{\infty}^2 = \max\left\{\Vert\mathbf{T}_m\Vert^2\Vert\tilde{\mathbf{T}}_m\Vert_{\infty}^2\right\}$; also $\Vert\mathcal{P}_{\mathbb{S}_{\text{c}}}\star\xi_b\Vert_{\text{F}}^2 = \text{mean}\left\{\Vert\mathcal{P}_{\mathbb{S}_m}\star\xi_b\Vert_{\text{F}}^2\right\}$. Thus, when $\mathbf{T}_m$ contains not too less entries, $\Vert\mathbf{T}\Vert^2\Vert\tilde{\mathbf{T}}\Vert_{\infty}^2\Vert\mathcal{P}_{\mathbb{S}}\star\xi_b\Vert_{\text{F}}^2$ is approximately the same for slim and square transforms. Therefore, the differentiating requirements would be:
                    \begin{align}
                        & p_i \sim \max_{ijk}\Vert\mathit{T}^{\text{H}}\mathcal{P}_\mathcal{U}\tilde{\mathit{T}}(\mathdutchcal{e}_{ijk})\Vert_{\text{F}}^2 \\
                        & p_i \sim \max_{ijk}\Vert\mathit{T}^{\text{H}}\mathcal{P}_\mathcal{V}\tilde{\mathit{T}}(\mathdutchcal{e}_{ijk})\Vert_{\text{F}}^2
                    \end{align}

                    \item To satisfy the requirements posed by $\Vert \mathit{T}^{\text{H}}(\mathcal{U}\star\mathcal{V}^{\text{H}})\Vert_{\infty}$ and $\Vert \mathit{T}^{\text{H}}(\mathcal{U}\star\mathcal{V}^{\text{H}})\Vert_{\infty, 2}^2$: to finally satisfy \eqref{eq:cond:spectral_norm_cond}, the requirements are:  
                    \begin{align}
                        & p_i \sim \Vert\tilde{\mathbf{T}}\Vert_{\infty}\cdot\Vert \mathit{T}^{\text{H}}(\mathcal{U}\star\mathcal{V}^{\text{H}})\Vert_{\infty}, \\
                        & p_i \sim \Vert\tilde{\mathbf{T}}\Vert_{\infty}^2\cdot\Vert \mathit{T}^{\text{H}}(\mathcal{U}\star\mathcal{V}^{\text{H}})\Vert_{\infty, 2}^2.
                    \end{align}
                \end{itemize}
            \end{itemize}
            Thus, the sampling rate is proportional to the following terms:
            \begin{equation}
                \begin{aligned}
                    p \sim 
                    \begin{cases}
                        & \kappa(\mathbf{T}) \cdot \max_{ijk}(\Vert\mathcal{P}_\mathbb{S}\mathit{T}(\mathdutchcal{e}_{ijk})\Vert_{\text{F}}^2 \vee \Vert\mathcal{P}_\mathbb{S}\tilde{\mathit{T}}(\mathdutchcal{e}_{ijk})\Vert_{\text{F}}^2) \\
                        & \max_{ijk}\Vert\mathit{T}^{\dag}\mathcal{P}_\mathbb{S}\mathit{T}(\mathdutchcal{e}_{ijk})\Vert_{\text{F}}^2 \\
                        & \max_{ijk}\Vert\mathit{T}^{\text{H}}\mathcal{P}_\mathcal{U}\tilde{\mathit{T}}(\mathdutchcal{e}_{ijk})\Vert_{\text{F}}^2 \\
                        & \max_{ijk}\Vert\mathit{T}^{\text{H}}\mathcal{P}_\mathcal{V}\tilde{\mathit{T}}(\mathdutchcal{e}_{ijk})\Vert_{\text{F}}^2 \\
                        & \Vert\tilde{\mathbf{T}}\Vert_{\infty}\cdot\Vert \mathit{T}^{\text{H}}(\mathcal{U}\star\mathcal{V}^{\text{H}})\Vert_{\infty}, \\
                        & \Vert\tilde{\mathbf{T}}\Vert_{\infty}^2\cdot\Vert \mathit{T}^{\text{H}}(\mathcal{U}\star\mathcal{V}^{\text{H}})\Vert_{\infty, 2}^2.
                    \end{cases}
                \end{aligned}
                \label{eq:sampling_rate-energy}
            \end{equation}
            We use the same settings in Section \ref{section:brief_analysis} and conducted the Monte-Carlo experiment for 100 times to plot the empirical cumulative distribution of the terms in \eqref{eq:sampling_rate-energy} for two random unitary transforms RUT1 and RUT2, and the combined transform. 

            The distributions of these terms 
            have been shown in Figure \ref{fig:sampling_rate_energy}. One can observe that, except that in Figure \ref{fig:v1} and \ref{fig:v69} the factors of the slim transform are marginally larger than those of the square transforms, the rest terms of the slim transform are all distributed in a region of smaller values, which results in a smaller sampling rate that we need to achieve exact completion. 

            \begin{figure}[ht]
                \centering
                \begin{subfigure}{0.45\textwidth}
                    \centering
                    \includegraphics[width=\textwidth]{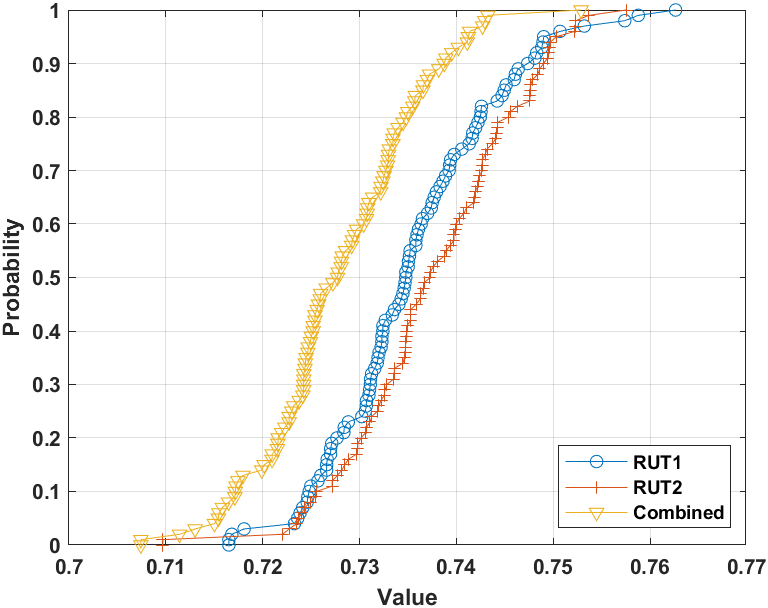}
                    \subcaption{\scriptsize \scriptsize$\kappa(\mathbf{T}) \cdot \max_{ijk}\{\Vert\mathcal{P}_\mathbb{S}\mathit{T}(\mathdutchcal{e}_{ijk})\Vert_{\text{F}}^2 \vee \Vert\mathcal{P}_\mathbb{S}\tilde{\mathit{T}}(\mathdutchcal{e}_{ijk})\Vert_{\text{F}}^2\}$}
                    \label{fig:arr_v_1_1}
                \end{subfigure}
                \hfill
                \begin{subfigure}{0.45\textwidth}
                    \centering
                    \includegraphics[width=\textwidth]{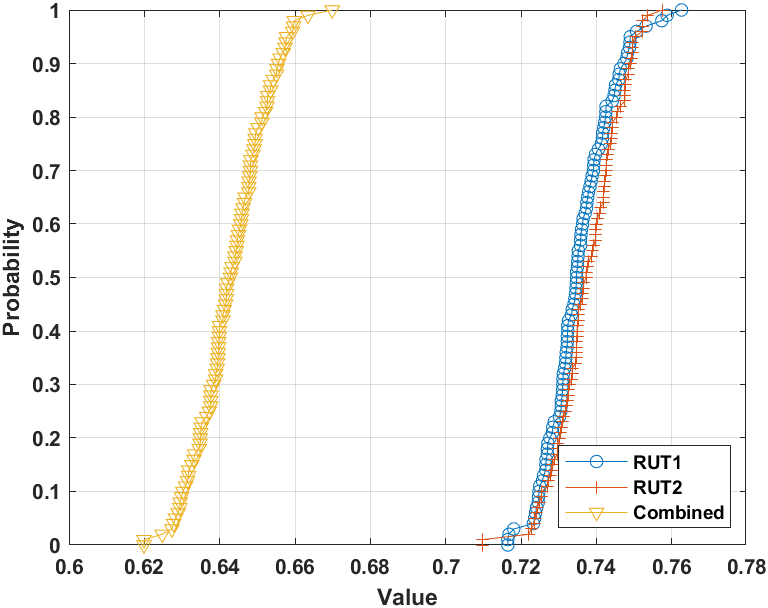}
                    \subcaption{\scriptsize \scriptsize$\max_{ijk}\Vert\mathit{T}^{\dag}\mathcal{P}_\mathbb{S}\mathit{T}(\mathdutchcal{e}_{ijk})\Vert_{\text{F}}^2$}
                    \label{fig:arr_v_1_2}
                \end{subfigure}
                \vfill
                \begin{subfigure}{0.45\textwidth}
                    \centering
                    \includegraphics[width=\textwidth]{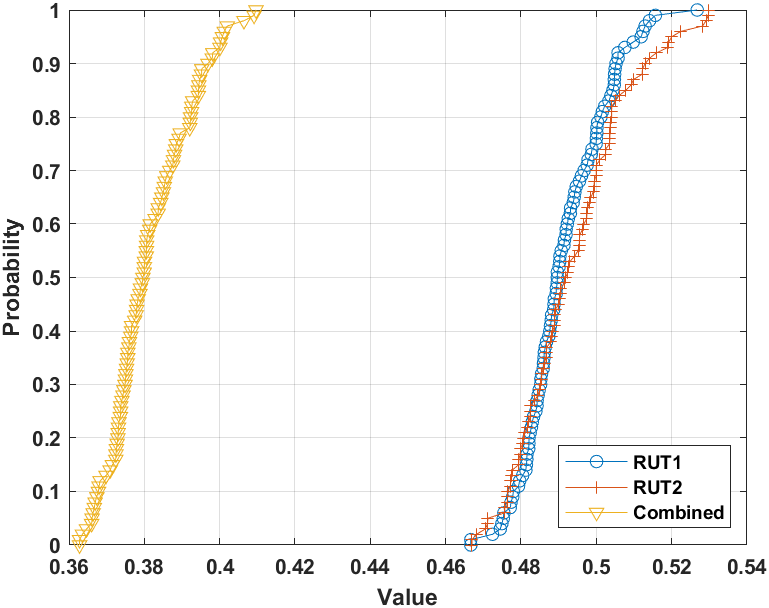}
                    \subcaption{\scriptsize \scriptsize$\max_{ijk}\Vert\mathit{T}^{\text{H}}\mathcal{P}_\mathcal{U}\tilde{\mathit{T}}(\mathdutchcal{e}_{ijk})\Vert_{\text{F}}^2$}
                    \label{fig:arr_v_1_4}
                \end{subfigure}
                \hfill
                \begin{subfigure}{0.45\textwidth}
                    \centering
                    \includegraphics[width=\textwidth]{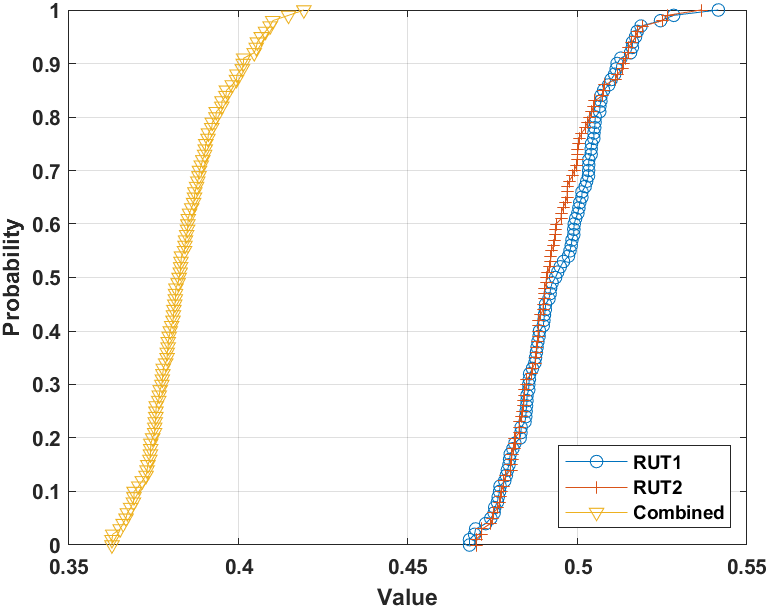}
                    \subcaption{\scriptsize \scriptsize$\max_{ijk}\Vert\mathit{T}^{\text{H}}\mathcal{P}_\mathcal{V}\tilde{\mathit{T}}(\mathdutchcal{e}_{ijk})\Vert_{\text{F}}^2$}
                    \label{fig:arr_v_1_4}
                \end{subfigure}
                \vfill
                \begin{subfigure}{0.45\textwidth}
                    \centering
                    \includegraphics[width=\textwidth]{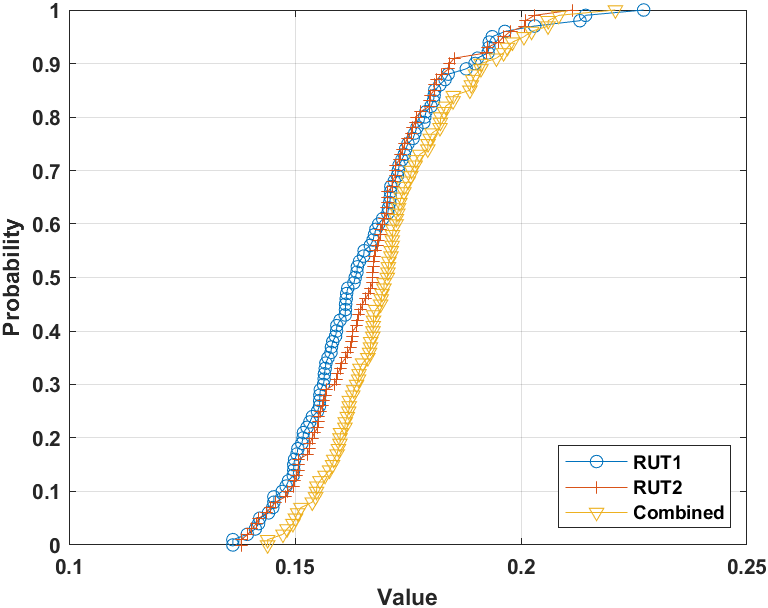}
                    \subcaption{\scriptsize $\Vert\tilde{\mathbf{T}}\Vert_{\infty}\cdot\Vert \mathit{T}^{\text{H}}(\mathcal{U}\star\mathcal{V}^{\text{H}})\Vert_{\infty}$}
                    \label{fig:v1}
                \end{subfigure}
                \hfill
                \begin{subfigure}{0.45\textwidth}
                    \centering
                    \includegraphics[width=\textwidth]{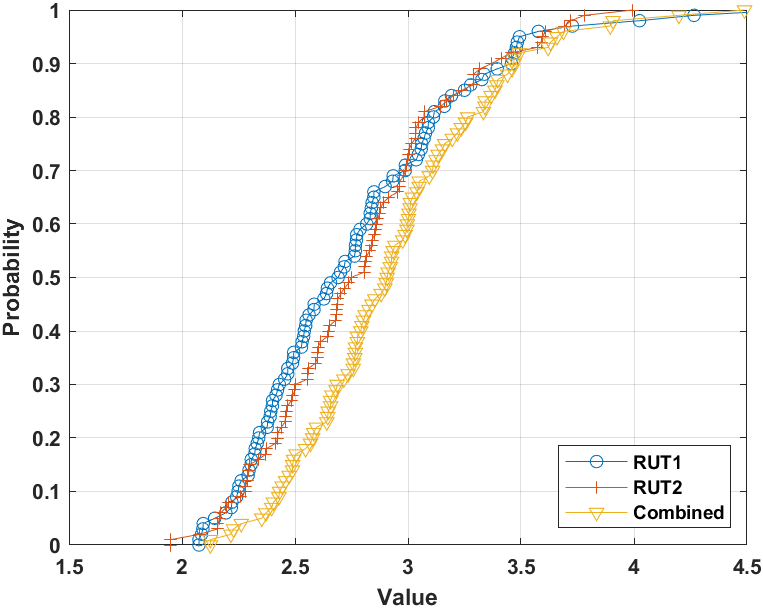}
                    \subcaption{\scriptsize $\Vert\tilde{\mathbf{T}}\Vert_{\infty}^2\cdot\Vert \mathit{T}^{\text{H}}(\mathcal{U}\star\mathcal{V}^{\text{H}})\Vert_{\infty, 2}^2$}
                    \label{fig:v69}
                \end{subfigure}
                \caption{\scriptsize Emperical cumulative distribution of the factors in \eqref{eq:sampling_rate-energy}.}
                \label{fig:sampling_rate_energy}
            \end{figure}
            
        \section{Algorithm to solve Program \ref{eq:program:min_transformed_TNN}}
            \label{section:solving_main_program}
            Program \ref{eq:program:min_transformed_TNN} can be solved via standard ADMM. By introducing the auxiliary variables $\mathcal{Y}, \mathcal{Z} \in \mathbb{C}^{n_1 \times n_2 \times N_3}$, the augmented Lagrangian can be written as:
            \begin{equation}
                \begin{aligned}
                    \mathcal{L}(\mathcal{X}, \mathcal{Y}, \mathcal{Z}) & = \min_{\mathcal{X}, \mathcal{Y}, \mathcal{Z}} \Vert \mathcal{Y} \Vert_{*} + \langle \mathit{T}(\mathcal{X}) - \mathcal{Y}, \mathcal{Z} \rangle + \frac{\alpha}{2} \Vert \mathit{T}(\mathcal{X}) - \mathcal{Y} \Vert_{\text{F}}^2 + \mathcal{I}_{\Omega}(\mathcal{X} - \mathcal{M}) 
                \end{aligned}
            \end{equation}
            where $\mathcal{I}_{\Omega}(\mathcal{A})$ equals to 0 if $\mathcal{P}_{\Omega}(\mathcal{A}) = \mathbf{0}$ and equals to $+\infty$ elsewhere.
            
            \begin{itemize}
                \item Solving $\mathcal{Y}$ sub-problem:
                \begin{align}
                    \mathcal{Y} = \arg\min_{\mathcal{Y}} \Vert \mathcal{Y} \Vert_* + \frac{\alpha}{2} \Vert \mathit{T}(\mathcal{X}) - \mathcal{Y} + \frac{\mathcal{Z}_1}{\alpha} \Vert_{\text{F}}^2 \nonumber
                \end{align}
    
                Similar to \cite{Lu_CVPR} while being more intuitive in form since the tensor nuclear norm is directly defined on the transform domain, define $\mathcal{D}_{\tau}(\cdot)$ as the singular value thresholding operator and it is the proximity operator of \eqref{eq:program_mixed_nuclear_Fro_norm}:
                \begin{align}
                    \mathcal{D}_{\tau}(\mathcal{A}) & = \arg\min_{\mathcal{X}} \tau\Vert\mathcal{X}\Vert_{*} + \frac{1}{2}\Vert\mathcal{X} - \mathcal{A}\Vert_{\text{F}}^2 \label{eq:program_mixed_nuclear_Fro_norm} = \mathcal{U} \star (\mathcal{S} - \tau)_{+} \star \mathcal{V}^{\text{H}}, 
                \end{align}
                where $\mathcal{A} = \mathcal{U} \star \mathcal{S} \star \mathcal{V}^{\text{H}}$.
    
                Using $\mathcal{D}_{\tau}(\cdot)$, $\mathcal{Y}$ can be computed as
                \begin{equation}
                    \mathcal{Y}^{t+1} = \mathcal{D}_{1/ \alpha^t}(\mathit{T}(\mathcal{X}^t) + \frac{\mathcal{Z}^t}{\alpha^t}) \label{eq:computing_Y}
                \end{equation}
    
                \item Solving $\mathcal{X}$ sub-problem:
                \begin{align}
                    \mathcal{X}^{t+1} & = \arg\min_{\mathcal{X}} \frac{\alpha}{2} \Vert \mathit{T}(\mathcal{X}) - \mathcal{Y} + \frac{\mathcal{Z}_1}{\alpha} \Vert_{\text{F}}^2 + \mathcal{I}_{\Omega}(\mathcal{X} - \mathcal{M})\\ & = \mathcal{S}_{\Omega^\text{C}}(\mathit{T}^{\dag}(\mathcal{Y}^{t} - \frac{\mathcal{Z}^t}{\alpha^t})) + \mathcal{S}_{\Omega}(\mathcal{M}) \label{eq:computing_X}
                \end{align}
    
                \item Computing Lagrangian multiplier
                \begin{align}
                    \mathcal{Z}^{t+1} = \mathcal{Z}^{t} + \alpha^t (\mathit{T}(\mathcal{X}^{t+1}) - \mathcal{Y}^{t+1}) \label{eq:computing_Z}
                \end{align}
            \end{itemize}
    
            The whole process is given in Algorithm \ref{alg:solving_by_ADMM}.
            \begin{algorithm}[ht]
                \caption{\scriptsize Solving \eqref{eq:program:min_transformed_TNN} by ADMM}
                \label{alg:solving_by_ADMM}
                \begin{algorithmic}
                \STATE {\bfseries Input:} $\mathcal{X}_0 \in \mathbb{C}^{n_1 \times n_2 \times n_3}, \mathbf{T} \in \mathbb{C}^{N_3 \times n_3}, \Omega, \rho, \alpha_0, \alpha_{\max}$
                \STATE {\bfseries Output:} $\mathcal{X} \in \mathbb{C}^{n_1 \times n_2 \times n_3}$
                \STATE {\bfseries Initialize: $\mathcal{X} = \mathcal{X}_0, \mathcal{Y} = \mathit{T}(\mathcal{X}), \mathcal{Z} = \mathbf{0}, \alpha = \alpha_0$}
                \WHILE{no convergence}
                    \STATE 1. Compute $\mathcal{Y}^{t+1}$ by eq. \eqref{eq:computing_Y};
                    \STATE 2. Compute $\mathcal{X}^{t+1}$ by eq. \eqref{eq:computing_X};
                    \STATE 3. Compute $\mathcal{Z}^{t+1}$ by eq. \eqref{eq:computing_Z};
                    \STATE 4. $\alpha^{t+1} = \min \{\rho \alpha^t, \alpha_{\max}\}$;
                \ENDWHILE
                \end{algorithmic}
            \end{algorithm}

        \section{Generate $\mathcal{M}$ by Alternating Projection}
            \label{section:generate_M}
            In this section, we introduce the alternating projection algorithm to generate $\mathcal{M}$ with a specified $\text{rank}(\mathit{T}(\mathcal{M}))$. First, provided with a randomly initiated tensor $\mathcal{M}_0$ and a preassigned transform $\mathit{T}(\cdot)$, we project $\mathit{T}(\mathcal{M})$ onto a low tubal rank manifold using truncated t-SVD. In matrix case, the truncated SVD gives the projection onto nearest low rank matrix manifold in the sense of the metric induced by the Frobenius norm. The truncated t-SVD also enjoys such property following the same reason. After the projection onto low tubal rank manifold, we project $\mathit{T}(\mathcal{M})$ onto the column space spanned by $\mathbf{T}$ to ensure the preimage $\mathcal{M}$ exists via least square, which also minimizes the difference of the energy before and after the projection (measured by the Frobenius norm). Also we keep the energy of $\mathit{T}(\mathcal{M})$ to avoid iterating to zero. By repeating these three steps we can generated random tensors $\mathcal{M}$ with a specified $\text{rank}(\mathit{T}(\mathcal{M}))$ and the procedure is given in Algorithm \ref{alg:gen_TX_with_rank_r}.
    
            \begin{algorithm}[h]
                \caption{\scriptsize Generate $\mathcal{M}$ with a specified $\text{rank}(\mathit{T}(\mathcal{M}))$}
                \label{alg:gen_TX_with_rank_r}
                \begin{algorithmic}
                \STATE {\bfseries Input:} $\mathcal{M}_0 \in \mathbb{C}^{n_1 \times n_2 \times n_3}, \mathit{T}: \mathbb{C}^{n_3}\mapsto\mathbb{C}^{N_3}, r$
                \STATE {\bfseries Output:} $\mathcal{M} \in \mathbb{C}^{n_1 \times n_2 \times n_3}$
                \STATE {\bfseries Initialize: $\mathcal{M} = \mathcal{M}_0$}
                \WHILE{no convergence}
                    \STATE 1. Project $\mathit{T}(\mathcal{M})$ onto a tubal rank-$r$ manifold using truncated t-SVD and obtains $\mathcal{N}$;
                    \STATE 2. Project $\mathcal{N}$ back to the column space spanned by $\mathit{T}$:
                    $\mathit{T}(\mathcal{M}) = \mathit{T}(\mathit{T}^{\dag}(\mathcal{N}))$;
                    \STATE 3. $\mathit{T}(\mathcal{M}) = \mathit{T}(\mathcal{M}) / \Vert \mathit{T}(\mathcal{M}) \Vert_{\text{F}}$;
                \ENDWHILE
                \end{algorithmic}
            \end{algorithm}
    
            In the experiments, for performance comparison, generating $\mathcal{M}$ with $\text{rank}(\mathit{T}_1(\mathcal{M})) = r_1$ and $\text{rank}(\mathit{T}_2(\mathcal{M})) = r_2$ is needed. Proceeding from Algorithm \ref{alg:gen_TX_with_rank_r}, an algorithm with two sets of projections is designed in Algorithm \ref{alg:gen_2_TX_with_rank_r1_r2}. Note that in Algorithm \ref{alg:gen_2_TX_with_rank_r1_r2}, since it is $\mathcal{M}$ being operated thus we need to project $\mathcal{N}_1$ and $\mathcal{N}_2$ onto the preimage space.
            
            \begin{algorithm}[h]
                \caption{\scriptsize Generate $\mathcal{M}$ with specified $\text{rank}(\mathit{T}_1(\mathcal{M}))$ and $\text{rank}(\mathit{T}_2(\mathcal{M}))$}
                \label{alg:gen_2_TX_with_rank_r1_r2}
                \begin{algorithmic}
                \STATE {\bfseries Input:} $\mathcal{M}_0 \in \mathbb{C}^{n_1 \times n_2 \times n_3}, \mathit{T}_1: \mathbb{C}^{n_3}\mapsto\mathbb{C}^{M_3}, r_1, \mathit{T}_2: \mathbb{C}^{n_3}\mapsto\mathbb{C}^{N_3}, r_2$
                \STATE {\bfseries Output:} $\mathcal{M} \in \mathbb{C}^{n_1 \times n_2 \times n_3}$
                \STATE {\bfseries Initialize: $\mathcal{M} = \mathcal{M}_0$}
                \WHILE{no convergence}
                    \STATE 1. Project $\mathit{T}_1(\mathcal{M})$ onto a tubal rank-$r$ manifold using truncated t-SVD and obtains $\mathcal{N}_1$;
                    \STATE 2. Project $\mathcal{N}_1$ back to the preimage space of the column space spanned by $\mathit{T}_1$:
                    $\mathcal{M} = \mathit{T}_1^{\dag}(\mathcal{N}_1)$;
                    \STATE 3. $\mathcal{M} = \mathcal{M} / \Vert \mathcal{M} \Vert_{\text{F}}$;
                    \STATE 4. Project $\mathit{T}_2(\mathcal{M})$ onto a tubal rank-$r$ manifold using truncated t-SVD and obtains $\mathcal{N}_2$;
                    \STATE 5. Project $\mathcal{N}_2$ back to the preimage space of the column space spanned by $\mathit{T}_2$:
                    $\mathcal{M} = \mathit{T}_2^{\dag}(\mathcal{N}_2)$;
                    \STATE 6. $\mathcal{M} = \mathcal{M} / \Vert \mathcal{M} \Vert_{\text{F}}$;
                \ENDWHILE
                \end{algorithmic}
            \end{algorithm}

    \section{Limitations}
        \label{section:limitations}
        The limitations of this work mainly consist of three parts:
        \begin{itemize}
            \item As mentioned in Remark \ref{remark:design_T}, the algorithm searching for an optimal transform $\mathbf{T}$ is not contained in this work, since the formulation of the optimization problem is coupled complicatedly and requires non-trivial design. Although it is out of the scope of this paper, we will consider it in future work.

            \item In this paper, the tensors that we consider are 3D and we have not considered tensors of higher dimension. In fact, our model and proof can be generalized to higher dimensionality. Such extension has been proposed as ''high order t-SVD'' \cite{high_order_t-SVD}. It generalizes $3$-order t-SVD to $d$-order while preserving the basic properties of t-SVD. In literature, $3$-order t-SVD is most frequently discussed since such extension to higher order is considered to be not too difficult.

            \item Though enjoying many satisfying properties, one frequently mentioned issue for the approach using nuclear norm minimization is the scaling problem since it usually requires time-consuming SVD operation. In this paper, the problem can be alleviated with the technique of parallelization.
            
            The computational complexity of the nuclear norm minimization algorithm in the paper is $\mathcal{O}(\max(n_1, n_2)\cdot\min^2(n_1, n_2)N_3 + n_1n_2n_3N_3)$ per iteration for all transforms. The first part corresponds to the truncated t-SVD operation, which involves computing SVD with complexity $\mathcal{O}(\max(n_1, n_2)\cdot\min^2(n_1, n_2))$ for all $N_3$ slices. It can be parallelized such that SVD of all $N_3$ slices can be calculated simultaneously, which mitigates the external running time burden brought by slim transforms. The second part refers to the domain transform operation.

        \end{itemize}
\clearpage

\end{document}